\documentclass[onecolumn, draftcls]{IEEEtran}
\usepackage{amsmath,amsfonts}
\usepackage{titlesec}
\titlespacing*{\subsubsection}{0pt}{*1}{*1}

\usepackage{graphicx}
\usepackage{subcaption}
\allowdisplaybreaks
\usepackage{algorithmic}
\usepackage{algorithm}
\usepackage{array}
\usepackage[caption=false,font=normalsize,labelfont=sf,textfont=sf]{subfig}
\usepackage{textcomp}
\usepackage{stfloats}
\usepackage{url}
\usepackage{verbatim}
\usepackage{graphicx}
\usepackage{cite}
\usepackage[utf8]{inputenc} 
\usepackage[T1]{fontenc}
\usepackage{url}
\usepackage{ifthen}
\usepackage{cite}
\usepackage{mathtools}
\usepackage{makecell}

\mathtoolsset{showonlyrefs}
\usepackage{amsfonts}
\usepackage{amssymb}
\usepackage{amsmath}
\usepackage{amsthm}
\usepackage{dsfont}
\usepackage{tikz}
\usetikzlibrary{decorations.pathreplacing, shapes.geometric, shapes.misc}

\newcommand{\argmax}{\operatornamewithlimits{argmax}}
\newcommand{\argmin}{\operatornamewithlimits{argmin}}
\newtheorem{theorem}{Theorem} % Define a new theorem 

\newtheorem{definition}{Definition}
\newtheorem{corollary}{Corollary}
\newtheorem{lemma}{Lemma}
\newtheorem{claim}{Claim}
\newtheorem{remark}{Remark}
\usepackage{algorithm}
\usepackage{algorithmic}
\usepackage{pgfplots}
\pgfplotsset{compat=1.18} % Set compatibility version (adjust as needed)
\begin{document}
% \input{coverletter}
% \newpage

\title{Online Clustering of Data Sequences with Bandit Information\thanks{Part of this work has been presented at the IEEE International Symposium on Information Theory, Ann Arbor, USA, 2025, and another part has been accepted for presentation at the IEEE International Conference on Acoustic, Speech, and Signal Processing, Barcelona, Spain, 2026. This work is supported in part by Qualcomm University Research Grant. Kota Srinivas Reddy was supported by the Department of Science and Technology (DST), Govt. of India, through the INSPIRE faculty fellowship.}
}
% \author{\IEEEauthorblockN{G Dhinesh Chandran, Kota Srinivas Reddy, Srikrishna Bhashyam,\\}
% \IEEEauthorblockA{Department of Electrical Engineering,\\}
% \IEEEauthorblockA{Indian Institute of Technology Madras, Chennai 600036, Tamil Nadu, India.\\}
% \IEEEauthorblockA{Email: ee22d200@smail.iitm.ac.in,\{ksreddy,skrishna\}@ee.iitm.ac.in} }

\author{
\IEEEauthorblockN{
G.~Dhinesh~Chandran\IEEEauthorrefmark{1},
Kota~Srinivas~Reddy\IEEEauthorrefmark{2},
Srikrishna~Bhashyam\IEEEauthorrefmark{1}\\
}
\IEEEauthorblockA{\IEEEauthorrefmark{1}
Department of Electrical Engineering, IIT Madras, Chennai, Tamil Nadu, India\\
% Institute A\\
% Email: dhinesh@xyz.edu
}
\IEEEauthorblockA{\IEEEauthorrefmark{2}
Department of Artificial Intelligence, IIT Kharagpur, Kharagpur, West Bengal, India\\
% Institute B\\
% Email: srinivas@xyz.edu
}
\IEEEauthorblockA{Email: ee22d200@smail.iitm.ac.in, ksreddy@ai.iitkgp.ac.in, skrishna@ee.iitm.ac.in}
}
% \author{IEEE Publication Technology,~\IEEEmembership{Staff,~IEEE,}
%         % <-this % stops a space
% \thanks{This paper was produced by the IEEE Publication Technology Group. They are in Piscataway, NJ.}% <-this % stops a space
% \thanks{Manuscript received October 26, 2023; revised December 8, 2023.}}

% % The paper headers
% \markboth{Journal of \LaTeX\ Class Files,~Vol.~1, No.~2, December~2023}%
% {Shell \MakeLowercase{\textit{et al.}}: A Sample Article Using IEEEtran.cls for IEEE Journals}

% \IEEEpubid{0000--0000~\copyright~2023 IEEE}
% Remember, if you use this you must call \IEEEpubidadjcol in the second
% column for its text to clear the IEEEpubid mark.

\maketitle

\begin{abstract}
    We study the problem of online clustering of data sequences in the multi-armed bandit (MAB) framework under the fixed-confidence setting. There are $M$ arms, each providing i.i.d. samples from a parametric distribution whose parameters are unknown. The $M$ arms form $K$ clusters based on the distance between the true parameters. In the MAB setting, one arm can be sampled at each time. The objective is to estimate the clusters of the arms using as few samples as possible from the arms, subject to an upper bound on the error probability. Our setting allows for: arms within a cluster to have non-identical distributions, vector parameter arms, vector observations, and $K \le M$ clusters. 
    %Unlike existing works in the literature that assume that the arms within a cluster have identical distributions, we allow the arms in a cluster to have non-identical distributions.
    We propose and analyze the Average Tracking Bandit Online Clustering (ATBOC) algorithm. For multivariate Gaussian distributed arms, we show that ATBOC is asymptotically order-optimal, i.e., the expected sample complexity grows at most twice as fast as that of lower bound in asymptotic regime ($\delta\rightarrow 0$). This upper bound on expected sample complexity is also valid for multivariate subgaussian arms. For single-parameter exponential family distributed arms, we show that ATBOC is asymptotically optimal, i.e., the expected sample complexity growth matches the lower bound in the asymptotic regime. We also propose a computationally more efficient alternatives Lower and Upper Confidence Bound based Bandit Online Clustering Algorithm (LUCBBOC), and Bandit Online Clustering-Elimination (BOC-ELIM). We derive the computational complexity of all the proposed algorithms and also compare the per-sample run time through simulations. The LUCBBOC and BOC-ELIM algorithms require lesser per-sample run time compared to ATBOC and achieve comparable performance. All the proposed algorithms are $\delta$-Probably correct, i.e., the error probability of cluster estimate at the stopping time is upper bounded by $\delta$. We validate the asymptotic optimality guarantees through simulations, and 
    present the comparison of our proposed algorithms with other related work through simulations on both synthetic and real-world datasets. 
\end{abstract}

\begin{IEEEkeywords}
Multi-armed bandit, Online Clustering, Average Tracking, Sample Complexity.
\end{IEEEkeywords}

\section{Introduction}

Clustering involves partitioning a collection of items into different groups, where the items within each group share similar properties. Clustering has many applications including drug discovery \cite{maccuish2010clustering}, market segmentation \cite{chaturvedi1997feature}, 
% \cite{grua2019clustream}
pattern recognition \cite{theodoridis2006pattern}, traffic monitoring \cite{liu2021spatio, abolhelm2021large}, and clustering of biological sequences \cite{chiu2022clustering}.
% Most of the practical applications require to cluster the items in an online fashion \cite{yang2024optimal}.
For example, in market segmentation, the items to be clustered are the customers.
We collect feedback from different customers on multiple products and group the customers into clusters based on their feedback, so that we can recommend appropriate products to customers in the future according to the cluster to which they belong. It is desirable to obtain the clusters with as little feedback from customers as possible. 

% Such applications require the design of an online clustering algorithm that observes sequential feedback from customers and adaptively decides from which customers to collect feedback, in order to reduce the overall amount of feedback.

Clustering algorithms can be linkage-based such as the Single Linkage clustering algorithm (SLINK) \cite{rohlf198212}, or based on cost minimization like K-Means \cite{ahmed2020k} and K-Medoids \cite{kaur2014k}, or spectral clustering (SPEC) \cite{ng2001spectral}. 
The problem of clustering data points, where the collection of data points is grouped into clusters, has been well-studied in the literature; 
See \cite{ezugwu2022comprehensive} for a comprehensive survey.
We focus on the class of data sequence clustering problems in which we divide a collection of i.i.d. data sequences (arms) generated from unknown distributions into clusters.
% Another class of clustering problems, also referred to as Online Clustering \cite{liberty2016algorithm,silva2013data, fahy2019finding}, involves observing samples from a single data stream, where the objective is to cluster the observed data points in this data stream in a streaming manner without storing the entire stream. This problem differs from our setting.
Another class of clustering problems, also referred to as online clustering \cite{barbakh2008online, liberty2016algorithm, silva2013data, fahy2019finding}, involves observing samples from a single data stream, with the objective of assigning each data point to one of the clusters without storing the entire stream. This problem differs from our setting.

Data sequence clustering can be studied in a fixed-sample size setting (FSS) or a sequential setting (SEQ). In the FSS setting, a finite sequence of data points (samples) from each arm is available a priori, and the cluster estimate is based on the available samples. Clustering algorithms such as K-Means \cite{wang2018exponentially}, K-Medoids \cite{wang2019k}, and SLINK \cite{wang2020exponentially} have been studied in the FSS setting. In the SEQ setting, samples from the arms are available sequentially, and the algorithm is equipped with a stopping rule to decide when to stop sampling and output the cluster estimate. Sequential clustering can achieve the same clustering performance as FSS clustering with fewer samples on average \cite{zhu2025exponentially, diao2025sequential}. Sequential methods are also more suitable for applications where the data are naturally sequentially available. Sequential (SEQ) clustering can be studied in a full information setting or a Bandit setting. In the full information setting, at each round, a sample is available from each arm. Sequential versions of the clustering algorithms K-Medoids \cite{sreenivasan2023nonparametric}, SLINK \cite{singh2025exponentially}, and SPEC \cite{chandran2026sequentialspectralclusteringdata} have been studied in the literature in a full information setting. In the bandit setting, a sample can be observed only from the selected arm in each round. 

Clustering of data sequences in the Bandit setting is called Bandit Online Clustering (BOC) \cite{yang2024optimal, chandran2025online}, and an algorithm designed to solve a BOC problem is called a BOC algorithm. BOC algorithms, in addition to a stopping rule, are equipped with a sampling rule that adaptively selects an arm to observe based on past observations. 
Typically, for a given error probability, BOC algorithms, due to their adaptive sampling rule, use fewer observations on average to estimate clusters than their full-information and FSS counterparts.
% The clustering algorithm to be designed for the application of market segmentation, discussed in the beginning of this section, falls with in this category of clustering algorithm called BOC algorithm. 
%The application of market segmentation, discussed in the beginning of this section, requires the design of a BOC algorithm that observes sequential feedback from customers by adaptively deciding from which customer to collect feedback, in order to reduce the overall amount of feedback. 
BOC problems have been studied recently under the fixed confidence setting in \cite{yang2024optimal,thuot2024active,katariya2019maxgap,yavas2025general}, where the probability of error is fixed and the BOC algorithm is expected to estimate the cluster by observing less number of samples as possible. 
In \cite{yang2024optimal, yavas2025general}, an asymptotically optimal BOC algorithm is designed for clustering the arms, assuming that all arms in any given cluster have the same mean. For the same setting, \cite{thuot2024active} studied the non-asymptotic regime. However, in practical applications, the arms in a cluster do not need to have the same mean \cite{katariya2019maxgap}. The MaxGap algorithms in \cite{katariya2019maxgap} handle such cases, but are limited to two clusters with one-dimensional samples. In this paper, we address the general BOC problem, allowing different means within clusters, more than two clusters, and multi-dimensional samples.

\subsection{Related Work}
The multi-armed bandit (MAB) problem was first introduced in \cite{thompson1933likelihood}. MAB problems can be broadly categorized as regret-minimization problems, where the objective is to maximize cumulative reward by minimizing regret with respect to the optimal arm \cite{auer2002finite}, or pure exploration problems, where the goal is to identify the optimal arm or the structural properties of the arms. Pure exploration problems can be studied in the fixed-budget setting, where the number of arm pulls is limited and the goal is to minimize the probability of error for the fixed budget \cite{yang2022minimax}, or the fixed-confidence setting, where there is an upper bound constraint on the error probability and the goal is to minimize the total number of arm pulls for the fixed error probability \cite{yang2024optimal}. Pure exploration problems include best arm identification (BAI) \cite{garivier2016optimal,jedra2020optimal,kalyanakrishnan2012pac}, where the objective is to identify the arm with the highest mean; threshold bandits \cite{locatelli2016optimal}, where the objective is to classify arms as above or below a given threshold; and clustering, typically referred to as Bandit Online Clustering (BOC), where the objective is to group arms into clusters based on similarity in their underlying distributions. 

We focus on the BOC problem in the fixed-confidence setting. 
Any bandit algorithm designed to solve a pure exploration problem under fixed confidence setting consists of a sampling rule to sequentially select arms, a stopping rule, and the declaration rule to estimate the optimal arm or the structural properties of the arms.
The BOC algorithms proposed for fixed confidence setting, based on their sampling rules, can be either gap-based or tracking-based. In gap-based BOC algorithms, at each round, arm selection is based on the estimated gap between arms \cite{katariya2019maxgap}. In the tracking-based BOC algorithms, at each round, the arm selection is made such that the proportion of samples from each arm tracks an optimal proportion (which is a problem-instance-dependent unknown quantity but can be estimated at each round from the samples) \cite{yang2024optimal}. We design both gap-based and tracking-based BOC algorithms. 

Sequential multi-hypothesis testing has also been studied in the MAB setting in \cite{deshmukh2021sequential,prabhu2022sequential}. The BOC problem can be viewed as a special case of the multi-hypothesis problem, with each possible partition of the arms treated as a hypothesis. Nevertheless, the number of hypotheses grows exponentially with the number of arms, resulting in significant computational complexity. The algorithms designed in the work \cite{deshmukh2021sequential, prabhu2022sequential}, for sequential multi-hypothesis testing, are tracking-based.

Existing tracking-based BOC algorithms are designed based on direct tracking (D-Tracking), where the proportion of samples from each arm tracks the estimated optimal proportion of samples at each round \cite{yang2024optimal, deshmukh2021sequential, prabhu2022sequential}. In \cite{yang2024optimal}, the unknown true optimal proportion of arms is proven to be unique and in \cite{deshmukh2021sequential, prabhu2022sequential}, the unknown true optimal proportion of arms is assumed to be unique. In general, the true optimal proportion of arms need not be unique. Motivated from the algorithm designed to solve BAI problem in the linear bandit setting \cite{jedra2020optimal}, for such scenarios, we design average tracking BOC algorithms that track the average of the estimated optimal proportion of samples. However, the gap-based algorithms are typically computationally efficient than tracking-based algorithms, as the tracking-based algorithm involves solving an optimization algorithm to estimate the optimal proportion of samples from each arm.
We also design a couple of gap-based algorithms, one based on \cite{kalyanakrishnan2012pac}, which solves the BAI problem by computing the lower and upper confidence bounds of the arms, and another one based on \cite{katariya2019maxgap}, which solves the MaxGap identification problem based on arm elimination.

\subsection{Our Contributions}
%, i.e, non-zero intra-cluster distance. %The number of clusters can also be arbitrary, as long as $d_{INTRA} < d_{INTER}$, where $d_{INTRA}$ is the maximum intra-cluster neighbour distance between the means and $d_{INTER}$ is the minimum inter-cluster distance. 
% In this paper, we make the following contributions:

% \subsection{}
The main contributions in this work are as follows. 
\begin{itemize}
    \item We propose a tracking-based BOC algorithm called Average Tracking Bandit Online Clustering (ATBOC).
We derive an upper bound on the expected sample complexity of ATBOC when the arms follow multivariate sub-Gaussian distributions.
By comparing this bound with the asymptotic lower bound, we show that for Gaussian-distributed arms, ATBOC is asymptotically order-optimal, and we refer to this variant of ATBOC as ATBOC-Gauss.
The same algorithm, when applied to multivariate sub-Gaussian arms other than Gaussian, is referred to as ATBOC-subGauss.
For the case where arms follow a single-parameter exponential family distribution, we use a different stopping threshold and prove it to be asymptotically optimal. We refer to this variant of ATBOC as ATBOC-1pExp.
    \item More computationally efficient gap-based BOC algorithms, namely, Lower and Upper Confidence Bound--based Bandit Online Clustering (LUCBBOC) and Bandit Online Clustering Elimination (BOC-ELIM), are proposed.
Upper bounds on the expected sample complexity are derived for both LUCBBOC and BOC-ELIM.
The three variants of ATBOC -- ATBOC-Gauss, ATBOC-subGauss, and ATBOC-1pExp -- as well as LUCBBOC and BOC-ELIM have been proven to be $\delta$-PC, that is, the probability of a clustering error at stopping time is upper bounded by $\delta$, for some $\delta \in (0,1)$.

    \item The implementation of the sampling and stopping rules of ATBOC-Gauss/ATBOC-subGauss and LUCBBOC involves solving a finite number of Quadratic Constrained Quadratic Programs (QCQPs). We propose an Alternating Direction Method of Multipliers (ADMM) based algorithm for solving the QCQP. For ATBOC-1pExp, we further simplify the optimization problem involved in the implementation into a finite minimization of closed-form expressions. 
    % We propose Alternating Direction Method of Multiplier (ADMM) based simplifications for the Quadratic Constrained Quadratic problems (QCQP) for solving the optimization problem involved in the implementation of sampling ans stopping rules in ATBOC-Gauss, ATBOC-subGauss and LUCBBOC. For ATBOC-1pExp, we further simiplify the optimization problem involved in the implementation into a finite minimization of closed form expressions.
    \item We derive the computational complexity order of all the proposed algorithms and also compare the per sample run time through simulations. We infer that the computationally efficient gap-based BOC algorithms, LUCBBOC and BOC-ELIM, require a lesser per-sample run time compared to ATBOC and achieve comparable performance.
    \item We validate the asymptotic theoretical guarantees of the proposed algorithms through simulations on synthetic datasets. In addition, we compare the performance of our algorithms with existing methods in the literature through simulations on synthetic and real-world datasets. Our proposed algorithms provide a significant improvement in the expected sample complexity compared to FSS and round-robin sampling settings for the same probability of error. For the specific setups considered in existing literature, i.e., such as when arms within the same cluster follow the same distribution or when the arms are partitioned into two clusters,  we observe that our proposed generalized algorithms perform comparably to existing methods. Moreover, when existing algorithms designed for identical distribution of arms in the same cluster are applied to our more general setup—where arms within the same cluster need not share the same distribution, our proposed algorithms outperform it by a large margin.
\end{itemize}

Parts of this work have been accepted for presentation at ISIT 2025 \cite{chandran2025online} and ICASSP 2026 \cite{chandran2025optimal}. In \cite{chandran2025online}, ATBOC-Gauss and LUCBBOC algorithms have been proposed. In \cite{chandran2025optimal}, ATBOC-1pExp algorithm has been proposed.

\subsection{Paper Organization and Notation}
The paper is organized as follows. Section~\ref{sec:ProblemSetup} describes the problem setup. Section~\ref{sec:Lowerbound} presents an algorithm-independent, problem-dependent lower bound on the expected sample complexity. Section~\ref{sec:AvgTrackingOC} introduces the proposed algorithms—ATBOC-Gauss, ATBOC-subGauss, and ATBOC-1pExp—and presents their theoretical guarantees. The computationally efficient BOC algorithms, LUCBBOC and BOC-ELIM, along with their theoretical guarantees, are discussed in Sections~\ref{sec:LUCB} and~\ref{sec: BOC-Elim}, respectively. Section~\ref{sec: compcomplx} discusses the implementation aspects and provides a computational complexity analysis of the proposed algorithms. Simulation results and conclusions are presented in Sections~\ref{sec:Simulations} and~\ref{sec:Conclusion}, respectively. Proofs are deferred to the Technical Appendix.

For an integer $n \geq 1$, $[n] \coloneqq \{1,2,\ldots,n\}$. The probability simplex of dimension $M$ is defined as \(\mathcal{P}_M \coloneqq \left\{ \boldsymbol{w} \in \mathbb{R}_{+}^M \mid w_1 + \ldots + w_M = 1 \right\}\). For any set $A$, we use $2^A$ to denote the power set of $A$, i.e., the set of all subsets of $A$. We use $\|\boldsymbol{v}\|_\infty$ to represent the $\ell_{\infty}$ norm of a vector $\boldsymbol{v}=[v_1, \ldots, v_M]$, defined as $\|\boldsymbol{v}\|_{\infty}\coloneqq \max_{m\in[M]}|v_m|$.

% Proof details are provided in the Appendix.
\section{System Model and Preliminaries} \label{sec:ProblemSetup}

%For an integer $n \geq 1$, let $[n] \coloneqq \{1,2,\ldots,n\}$. We consider a BOC problem, which contains a MAB with $M$ arms, where each arm $m \in [M]$ generates $d-$ dimensional i.i.d. samples sequentially according to a multivariate-Gaussian distribution\footnote{The multivariate-Gaussian rewards assumption is mainly for simplicity in the presentation. Our analyses can be extended to any single parameter exponential family with appropriate modifications.} with $d-$ dimensional unknown mean vector $\boldsymbol{\mu}_m$ and the identity matrix of dimension $d\times d$ as the  covariance matrix. We refer to $\boldsymbol{\mu} = [\boldsymbol{\mu}_1, \ldots, \boldsymbol{\mu}_M] \in \mathbb{R}^{d \times M}$ as the collection of mean vectors.
%The arms are grouped into $K$ clusters based on the distance between the means of the arms. Let $[K]$ represent the set of clusters. For any arm $m \in [M]$, which belongs to cluster $k \in [K]$, the cluster index of arm $m$, denoted by $c_m$, is $k$. We denote the cluster index vector by $\mathbf{c} = [c_1, \ldots, c_M]$. The tuple $(\boldsymbol{\mu}, \mathbf{c})$ contains information regarding the probability distribution of arms and how the arms are grouped into $K$ clusters. Hence, each tuple $(\boldsymbol{\mu}, \mathbf{c})$, with $\mathbf{c} \in [K]^M$ and $\boldsymbol{\mu} \in \mathbb{R}^{d \times M}$, defines a clustering problem.
\subsection{Clustering problem setup}
% \textcolor{blue}{
 We consider a BOC problem involving an MAB with $M$ arms, where each arm $m \in [M]$ generates $d$-dimensional i.i.d. samples from an unknown probability distribution. In this work, we consider two classes of distributions: multivariate $\sigma$-sub-Gaussian distributions, parameterized by their mean vectors,  and the single-parameter exponential family of distributions. 
% It can be understood that any discussion for the case of a multivariate sub-Gaussian distribution, dimension of samples $d$ can take any arbitrary integer value, that is, $d\in \mathbb{N}$, whereas $d=1$ for the discussion on a single-parameter exponential family of distribution.  
We use the same notation $\boldsymbol{\mu}_m$ of dimension $d$ to indicate both the mean vector of the multivariate sub-Gaussian distribution and the parameter of the single-parameter exponential family of distribution corresponding to the arm $m$. We call $\boldsymbol{\mu}_m\in\Theta$ as the parameter of the $m^{th}$ arm, where $\Theta\subseteq\mathbb{R}^d$ is the parameter space. Let $\boldsymbol{\mu} = [\boldsymbol{\mu}_1, \ldots, \boldsymbol{\mu}_M] \in \Theta^{M}$ represent the collection of parameters of the $M$ arms, which we refer to as the problem instance. The $M$ arms are grouped into $K$ clusters based on the distances between their parameters using the SLINK clustering algorithm, as illustrated in Fig.~\ref{fig: SLINKdescription}. We assume that the number of clusters $K$ is fixed and known to the learner. Let $[K]$ denote the set of clusters. For any arm $m \in [M]$ in the cluster $k \in [K]$, the cluster index of arm $m$, denoted $c_m$, is $k$. The cluster index vector corresponding to all $M$ arms is $\mathbf{c} = [c_1, \ldots, c_M]$. The tuple $(\boldsymbol{\mu}, \mathbf{c})$, with $\mathbf{c} \in [K]^M$ and $\boldsymbol{\mu} \in \Theta^{M}$, contains information on the probability distribution of arms and how the arms are grouped into $K$ clusters, and  hence defines the clustering problem. The clustering algorithm needs to determine the cluster index vector.
\subsection{Cluster distances}
Let \(d_{i,j}(\boldsymbol{\mu}) \coloneqq \left\| \boldsymbol{\mu}_i - \boldsymbol{\mu}_j \right\|\) denote the distance between arms \(i\) and \(j\). The inter-cluster distance between clusters \(p\) and \(q\) is defined as  
\[
    d(D_p, D_q) \coloneqq \min_{n \in D_p, m \in D_q} d_{m,n}(\boldsymbol{\mu}),
\]
where \(D_k \coloneqq \{ m \mid c_m = k \}\) represents the set of arms in the \(k^{\text{th}}\) cluster. The minimum inter-cluster distance is \(d_{\text{INTER}} \coloneqq \min_{D_p \neq D_q} d(p, q)\).  

% We use $2^{D_k}$ to denote the set of all subsets of $D_k$. 
The intra-cluster neighbor distance for cluster \(k\) is  
\[
    d(D_k) \coloneqq \max_{\substack{P_1 \in 2^{D_k} \setminus \{\emptyset, D_k\} \\ P_2 = D_k \setminus P_1}} \min_{i \in P_1, j \in P_2} d_{i,j}(\boldsymbol{\mu}),
\]
i.e., for any two sets $P_1$ and $P_2$ that partition the $k^{th}$ cluster, we find the minimum distance between those two partitions, and then the maximum over all such possible partitions. Here, $2^{D_k}$ denotes the set of all subsets of $D_k$. The maximum intra-cluster neighbor distance is \(d_{\text{INTRA}} \coloneqq \max_{k \in [K]} d(D_k)\).  

\subsection{Separation assumption for clusters}
Typically, we need some separation between clusters for any algorithm to cluster the arms.
We consider the family of clustering problems defined by the tuple \((\boldsymbol{\mu}, \boldsymbol{c})\) satisfying the condition\footnote{ The maximum intra-cluster distance defined in \cite{wang2019k, wang2020exponentially,sreenivasan2023nonparametric} is greater than or equal to the maximum intra-cluster neighbor distance defined in our paper. Consequently, our clustering condition is less restrictive than the condition assumed in  \cite{wang2019k, wang2020exponentially, sreenivasan2023nonparametric}, which requires the maximum intra-cluster distance to be smaller than the minimum inter-cluster distance. As a result, our framework accommodates a larger class of clustering problems.} \(d_{\text{INTRA}} < d_{\text{INTER}}\). This separation condition of $d_{\text{INTRA}} < d_{\text{INTER}}$ has been assumed in the work on non-parametric linkage-based clustering in \cite{singh2025exponentially}. This condition ensures that the \(M\) arms with parameters \(\boldsymbol{\mu}\) are unambiguously grouped into \(K\) clusters by single-linkage (SLINK) clustering algorithm as discussed in \cite{song2011unique, singh2025exponentially}.

%A brief description of the SLINK clustering algorithm is as follows. Each arm is initially treated as an individual cluster, resulting in $M$ clusters. In the first iteration, the algorithm identifies the pair of clusters with the smallest pairwise minimum distance and merges them, reducing the number of clusters from $M$ to $M-1$. This process continues until the desired number of clusters, $K$, is achieved. Since SLINK iteratively merges the clusters that have the minimum pairwise minimum distance, the minimum inter-cluster distance of the resulting $K$ clusters will always be greater than the maximum distance between any two partitions $P_1, P_2$ in any cluster $k$. Hence, given any collection of mean vectors $\boldsymbol{\mu} \in \mathbb{R}^{d\times M}$, we can get a cluster index vector that satisfies the condition that $d_{INTRA} < d_{INTER}$ through the SLINK algorithm. 
%Let $\mathcal{C}:\mathbb{R}^{d\times M} \rightarrow [K]^M$ be the relation that takes the collection of mean vectors $\boldsymbol{\mu}$ as input, partitions the arms through the SLINK algorithm, and outputs the cluster index vector $\mathcal{C}\left(\boldsymbol{\mu}\right)$. Figure \ref{fig:relation} provides an illustrative example to understand \(\mathcal{C}(\boldsymbol{\mu})\), given a collection of mean vectors \(\boldsymbol{\mu}\).

A brief description of the SLINK clustering algorithm is as follows. Each arm is initially treated as an individual cluster, resulting in $M$ clusters. At each iteration, the pair of clusters with the smallest inter-cluster distance is merged, reducing the number of clusters by one. This process continues until $K$ clusters remain. The working of SLINK clustering algorithm is illustrated with an example in Fig.~\ref{fig: SLINKdescription}.
Since SLINK merges clusters with the smallest inter-cluster distance, the resulting $K$ clusters ensure \(d_{\text{INTRA}} < d_{\text{INTER}}\).

% Figure \ref{fig: SLINKdescription} provides the illustrative example to explain the working of the Single Linkage Clustering Algorithm and illustrates \(\mathcal{C}(\boldsymbol{\mu})\) for a given \(\boldsymbol{\mu}\).
\def\xsize{0.25}
\begin{figure}[htbp]
    \centering

    % === First Grid ===
    \begin{subfigure}[b]{0.24\textwidth}
        \centering
        \begin{tikzpicture}[scale=0.32]
            % Yellow ellipses
            \foreach \x/\y in {-1/-2, -1/-1, 1/1, 2/2, 3/-3, 3.5/-3} {
                \filldraw[fill=yellow, fill opacity=0.5, draw=black, thick] (\x,\y) ellipse (0.35 and 0.35);
            }
            % Axes
            \draw[gray!70] (-4,-4) grid (4,4);
            \draw[thick,->] (-4,0)--(4.2,0);
            \draw[thick,->] (0,-4)--(0,4.2);

            %\labels
            \node[font=\small, right] at (-3,-2) {1};
            \node[font=\small, right] at (-3,-1) {2};
            \node[font=\small, right] at (-0.6,0.7) {3};
            \node[font=\small, right] at (2.1,2.4) {4};
            \node[font=\small, right] at (1, -3) {5};
            \node[font=\small, right] at (2.8,-1.8) {6};
            
            % Red crosses
            \foreach \x/\y in {-1/-2, -1/-1, 1/1, 2/2, 3/-3, 3.5/-3} {
                \draw[red, line width=1pt] (\x-\xsize,\y-\xsize)--(\x+\xsize,\y+\xsize);
                \draw[red, line width=1pt] (\x-\xsize,\y+\xsize)--(\x+\xsize,\y-\xsize);
            }
        \end{tikzpicture}
        \caption{}
    \end{subfigure}
    \hfill
    % === Second Grid ===
    \begin{subfigure}[b]{0.24\textwidth}
        \centering
        \begin{tikzpicture}[scale=0.32]
            % Separate ellipses
            \foreach \x/\y in {-1/-2, -1/-1, 1/1, 2/2} {
                \filldraw[fill=yellow, fill opacity=0.5, draw=black, thick] (\x,\y) ellipse (0.35 and 0.35);
            }
            % Merged ellipse for (4,3) and (4,4)
            \filldraw[fill=yellow, fill opacity=0.5, draw=black, thick] (3.25,-3) ellipse (0.8 and 0.6);
            % Axes
            \draw[gray!70] (-4,-4) grid (4,4);
            \draw[thick,->] (-4,0)--(4.2,0);
            \draw[thick,->] (0,-4)--(0,4.2);
            %\labels
            \node[font=\small, right] at (-3,-2) {1};
            \node[font=\small, right] at (-3,-1) {2};
            \node[font=\small, right] at (-0.6,0.7) {3};
            \node[font=\small, right] at (2.1,2.4) {4};
            \node[font=\small, right] at (1, -3) {5};
            \node[font=\small, right] at (2.8,-1.8) {6};
            % Red crosses
            \foreach \x/\y in {-1/-2, -1/-1, 1/1, 2/2, 3/-3, 3.5/-3} {
                \draw[red, line width=1pt] (\x-\xsize,\y-\xsize)--(\x+\xsize,\y+\xsize);
                \draw[red, line width=1pt] (\x-\xsize,\y+\xsize)--(\x+\xsize,\y-\xsize);
            }
        \end{tikzpicture}
        \caption{}
    \end{subfigure}
    \hfill
    % === Third Grid ===
    \begin{subfigure}[b]{0.24\textwidth}
        \centering
        \begin{tikzpicture}[scale=0.32]
            % Separate ellipses
            \foreach \x/\y in {1/1, 2/2} {
                \filldraw[fill=yellow, fill opacity=0.5, draw=black, thick] (\x,\y) ellipse (0.35 and 0.35);
            }
            \filldraw[fill=yellow, fill opacity=0.5, draw=black, thick] (-1,-1.5) ellipse (0.6 and 1);
            \filldraw[fill=yellow, fill opacity=0.5, draw=black, thick] (3.25,-3) ellipse (0.8 and 0.6);
            % Axes
            \draw[gray!70] (-4,-4) grid (4,4);
            \draw[thick,->] (-4,0)--(4.2,0);
            \draw[thick,->] (0,-4)--(0,4.2);
            %\labels
            \node[font=\small, right] at (-3,-2) {1};
            \node[font=\small, right] at (-3,-1) {2};
            \node[font=\small, right] at (-0.6,0.7) {3};
            \node[font=\small, right] at (2.1,2.4) {4};
            \node[font=\small, right] at (1, -3) {5};
            \node[font=\small, right] at (2.8,-1.8) {6};
            % Red crosses
            \foreach \x/\y in {-1/-2, -1/-1, 1/1, 2/2, 3/-3, 3.5/-3} {
                \draw[red, line width=1pt] (\x-\xsize,\y-\xsize)--(\x+\xsize,\y+\xsize);
                \draw[red, line width=1pt] (\x-\xsize,\y+\xsize)--(\x+\xsize,\y-\xsize);
            }
        \end{tikzpicture}
        \caption{}
    \end{subfigure}
    \hfill
    % === Fourth Grid ===
    \begin{subfigure}[b]{0.24\textwidth}
        \centering
        \begin{tikzpicture}[scale=0.32]
            \filldraw[fill=yellow, fill opacity=0.5, draw=black, thick] (1.5,1.5) ellipse (1 and 1);
            \filldraw[fill=yellow, fill opacity=0.5, draw=black, thick] (-1,-1.5) ellipse (0.6 and 1);
            \filldraw[fill=yellow, fill opacity=0.5, draw=black, thick] (3.25,-3) ellipse (0.8 and 0.6);
            % Axes
            \draw[gray!70] (-4,-4) grid (4,4);
            \draw[thick,->] (-4,0)--(4.2,0);
            \draw[thick,->] (0,-4)--(0,4.2);
            %\labels
            \node[font=\small, right] at (-3,-2) {1};
            \node[font=\small, right] at (-3,-1) {2};
            \node[font=\small, right] at (-0.6,0.7) {3};
            \node[font=\small, right] at (2.1,2.4) {4};
            \node[font=\small, right] at (1, -3) {5};
            \node[font=\small, right] at (2.8,-1.8) {6};
            
            % Red crosses
            \foreach \x/\y in {-1/-2, -1/-1, 1/1, 2/2, 3/-3, 3.5/-3} {
                \draw[red, line width=1pt] (\x-\xsize,\y-\xsize)--(\x+\xsize,\y+\xsize);
                \draw[red, line width=1pt] (\x-\xsize,\y+\xsize)--(\x+\xsize,\y-\xsize);
            }
        \end{tikzpicture}
        \caption{}
    \end{subfigure}
    \caption{Consider a BOC problem with $d=2$, $K=3$, $M=6$, and the parameter vectors $\boldsymbol{\mu} = \begin{bmatrix}
        -1 & -1 & 1 & 2 & 3 & 3.5 \\
        -1 & -2 & 1 & 2 & -3 & -3
    \end{bmatrix}$. SLINK initially assumes each point as a cluster, as shown in (a). Then, the two closest clusters are subsequently merged until the number of clusters reaches $K = 3$, as shown in (b), (c), and (d).}
    % Illustrative example to understand Single Linkage Clustering Algorithm (SLINK) and $\mathcal{C}(\boldsymbol{\mu})$. We have $d=2$, $K=3$, $M=6$, and mean vector $\boldsymbol{\mu} = \begin{bmatrix}
    %     -4 & -2.5 & 3 & 4 & 4 & 4 \\
    %     4 & 2.5 & -3 & -4 & 3 & 4
    % \end{bmatrix}$. Let $i^{th}$ of $\boldsymbol{\mu}$ be the mean vector of the $i^{th}$ arm, for $i \in [6]$. SLINK initially assumes each of the points as a cluster, as in (1). Then the two closest clusters are subsequently merged until the number of clusters $K = 3$ is reached, as in (b), (c), and (d).
    % On using SLINK, arms $1, 2$ will be assigned to cluster $1$; arms $3, 4$ will be assigned to cluster $2$, and arms $5,6$ will be assigned to cluster $3$. Hence, it outputs the cluster index vector, $\mathcal{C}(\boldsymbol{\mu}) = [1, 1, 2, 2, 3, 3]$.
    % }
    \label{fig: SLINKdescription}
\end{figure}

Let \(\mathcal{C}:\Theta^M \to [K]^M\) denote the mapping that takes a collection of parameters \(\boldsymbol{\mu} \in \Theta^M\), partitions the arms using SLINK, and outputs the cluster index vector \(\mathcal{C}(\boldsymbol{\mu})\). 
Given a collection of mean vectors \(\boldsymbol{\mu}\), the cluster index vector does not have to be unique. For the example considered in Fig.~\ref{fig: SLINKdescription}, both \(\boldsymbol{c^{(1)}} = [1, 1, 2, 2, 3, 3]\) and \(\boldsymbol{c^{(2)}} = [2, 2, 1, 1, 3, 3]\) are acceptable cluster index vectors. More precisely, for any two cluster index vectors \(\boldsymbol{c^{(1)}}\) and \(\boldsymbol{c^{(2)}}\), if there exists a permutation \(\sigma\) over \([K]\) such that \(\boldsymbol{c^{(1)}} = \sigma\left( \boldsymbol{c^{(2)}} \right)\), then \(\boldsymbol{c^{(1)}}\) is considered equivalent to \(\boldsymbol{c^{(2)}}\), and we denote this as \(\boldsymbol{c^{(1)}} \sim \boldsymbol{c^{(2)}}\). Here, \(\sigma(\boldsymbol{c}) \coloneqq \left[ \sigma(c_1), \ldots, \sigma(c_n) \right]\).   Note that given the collection of parameters \(\boldsymbol{\mu}\), the cluster index vector \(\boldsymbol{c}\) cannot be arbitrary; instead, it entirely depends on \(\boldsymbol{\mu}\). Hence, for the given number of clusters $K$, the problem instance is specified solely by the \(d\times M\)-dimensional matrix \(\boldsymbol{\mu} \in \Theta^M\), which contains the parameters of the \(M\) arms, with \(\boldsymbol{c} \sim \mathcal{C}(\boldsymbol{\mu})\).

\subsection{Clustering algorithm and performance metric}
Given the collection of parameters \(\boldsymbol{\mu}\), the cluster index vector can be determined using the relation \(\mathcal{C}\). However, in our problem,  \(\boldsymbol{\mu}\) is unknown. Therefore, in order to identify the cluster index vector \(\mathcal{C}(\boldsymbol{\mu})\), we observe samples generated from the arms. We consider a bandit information setup in which one arm is adaptively selected at each time step \(t\) and a sample is observed from the selected arm. Given \(\delta \in (0,1)\), the goal is to group the arms with the least expected number of arm selections (expected stopping time), while ensuring that the probability of error\footnote{Since we study a pure exploration problem under fixed confidence setting, we have the error probability metric and not a regret-based metric.} remains below \(\delta\).

%Given the collection of mean vectors \(\boldsymbol{\mu}\), we can determine the cluster index vector using the relation \(\mathcal{C}\). But, in our problem, the actual values of the means $\boldsymbol{\mu}$ are unknown. Therefore, in order to identify the cluster index vector \(\mathcal{C}(\boldsymbol{\mu})\), we have to observe the samples generated from the arms. In our work, we consider a bandit information setup, in which we select an arm at each time step $t$ and obtain a sample from that selected arm. Given $\delta \in (0,1)$, the goal is to cluster the arms with the least expected number of arm selections (a.k.a. expected stopping time), while keeping the error probability below $\delta$. %We call an algorithm designed to cluster the arms from active samples with bandit information as bandit online clustering (BOC) algorithm.%  and it prescribes the following:

%\begin{itemize}
    %\item A \textit{sampling rule} that specifies which arm to sample among the \(M\) arms at time \(t\).
    %\item A \textit{stopping rule} that determines when to cease further sampling of arms.
    %\item A \textit{declaration rule} that specifies the estimate of the cluster index vector \(\hat{\boldsymbol{c}} \in [K]^M\) to be output.
%\end{itemize}

We denote an algorithm designed to cluster arms by \(\pi\) and its sample complexity by \(\tau_\delta(\pi)\), which represents the total number of arm pulls under the algorithm \(\pi\) for a fixed confidence level (error probability) \(\delta\). We use $\hat{\boldsymbol{c}}(\pi)$ to denote the cluster index vector estimated by the algorithm $\pi$ at the stopping time. 

\begin{definition} \label{def:deltaPC}
    For \(\delta \in (0, 1)\), an algorithm \(\pi\) is said to be  \(\delta\)-PC (Probably Correct) if  for all \(\boldsymbol{\mu} \in \Theta^M\), we have
    \begin{equation}
        \mathbb{P}_{\boldsymbol{\mu}}^\pi \left( \hat{\boldsymbol{c}}(\pi) \sim \mathcal{C}(\boldsymbol{\mu}) \right) \geq 1 - \delta.
    \end{equation}
    Here, \(\mathbb{P}_{\boldsymbol{\mu}}^\pi(\cdot)\) denotes the probability measure under algorithm \(\pi\) and the problem instance \(\boldsymbol{\mu}\).
\end{definition}

Let \(\mathbb{E}_{\boldsymbol{\mu}}^{\pi}\left[ \tau_\delta(\pi) \right]\) denote the expected sample complexity of algorithm \(\pi\) for the problem instance \(\boldsymbol{\mu}\). Our goal is to design a \(\delta\)-PC algorithm \(\pi\) that minimizes the expected sample complexity \(\mathbb{E}_{\boldsymbol{\mu}}^{\pi}\left[ \tau_\delta(\pi) \right]\).

%We denote an algorithm designed to cluster the arms by \(\pi\) and its sample complexity by \(\tau_\delta(\pi)\), which represents the total number of arm pulls across all arms under the algorithm \(\pi\) for a fixed confidence level (error probability) of \(\delta\).

%\begin{definition} \label{def:deltaPC}
   % For \(\delta \in (0, 1)\), an algorithm \(\pi\) is said to be  \(\delta\)-PC (Probably Correct) if  for all \(\boldsymbol{\mu} \in \mathbb{R}^{d \times M}\), we have
   % \begin{equation}
       % \mathbb{P}_{\boldsymbol{\mu}}^\pi \left( \hat{\boldsymbol{c}} \sim \mathcal{C}(\boldsymbol{\mu}) \right) \geq 1 - \delta.
    %\end{equation}
    %Here, \(\mathbb{P}_{\boldsymbol{\mu}}^\pi(\cdot)\) denotes the probability measure under algorithm \(\pi\) and the problem instance \(\boldsymbol{\mu}\).
%\end{definition}

%Let \(\mathbb{E}_{\boldsymbol{\mu}}^{\pi}\left[ \tau_\delta(\pi) \right]\) denotes the expected sample complexity under algorithm \(\pi\) and the problem instance \(\boldsymbol{\mu}\). Our objective is to design a \(\delta\)-PC  algorithm \(\pi\) for finding the cluster index vector while minimizing the expected sample complexity \(\mathbb{E}_{\boldsymbol{\mu}}^{\pi}\left[ \tau_\delta(\pi) \right]\). %A formal description of the \(\delta\)-PC algorithm is presented below.
%Note that any \(\delta\)-PC algorithm \(\pi\) must declare the correct output with probability at least \(1 - \delta\) for all problem instances \(\boldsymbol{\mu}\).

% as \(\pi\) is designed to be oblivious to the specific details of the underlying problem instance \(\boldsymbol{\mu}\).
\begin{remark}
 %    In this work, we group the $M$ arms, each with $d$-dimensional samples, into $K$ clusters using the SLINK clustering algorithm. In the special case where $d = 1$, the SLINK clustering algorithm reduces to identifying the largest $K-1$ gaps between the means, which is known as the MaxGap identification problem in \cite{katariya2019maxgap}.
 % The MaxGap algorithms in \cite{katariya2019maxgap} are designed for the specific case where the number of clusters $K=2$.

In this work, we group the $M$ arms, each associated with $d$-dimensional samples, into $K$ clusters using the SLINK clustering algorithm. In the special case where $d = 1$, the SLINK algorithm reduces to identifying the largest $K-1$ gaps between the means. In the further special case where $d = 1$ and $K = 2$, our BOC problem reduces to the MaxGap identification problem studied in \cite{katariya2019maxgap}, where the objective is to identify the largest gap between the means.
\end{remark}

\section{Lower Bound}\label{sec:Lowerbound}
%In this section, we present an instance-dependent lower bound on the expected sample complexity. %We also discuss some of the continuity properties of the expressions involved in the lower bound.  
%For any problem instance $\boldsymbol{\mu}$, we define the alternative space of $\boldsymbol{\mu}$ as $\text{Alt}(\boldsymbol{\mu}) \coloneqq \left\{ \boldsymbol{\lambda} \in \mathbb{R}^{d\times M} \mid \mathcal{C}(\boldsymbol{\lambda}) \nsim \mathcal{C}(\boldsymbol{\mu}) \right\}$. We call any problem instance $\boldsymbol{\lambda}$ in $\text{Alt}(\boldsymbol{\mu})$ an alternative instance. We introduce the probability simplex $\mathcal{P}_M$, defined as $\mathcal{P}_M \coloneqq \left\{ \boldsymbol{w} \in \mathbb{R}_{+}^M \mid w_1 + \ldots + w_M = 1 \right\}$. Any point $\boldsymbol{w} = [w_1, \ldots, w_M]$ in the probability simplex can be understood as a probability distribution on the arm set $[M]$. Let $d_{kl}(a, b)$ denote the KL-divergence between Bernoulli distributions with means $a$ and $b$. 
%The following theorem gives a lower bound on the expected sample complexity of any $\delta-$PC algorithm.

In this section, we present an algorithm-independent, problem instance-dependent lower bound on the expected sample complexity. For any problem instance \(\boldsymbol{\mu}\), the alternative space is defined as \(\text{Alt}(\boldsymbol{\mu}) \coloneqq \left\{ \boldsymbol{\lambda} \in \Theta^M \mid \mathcal{C}(\boldsymbol{\lambda}) \nsim \mathcal{C}(\boldsymbol{\mu}) \right\}\), and any \(\boldsymbol{\lambda} \in \text{Alt}(\boldsymbol{\mu})\) is called an alternative problem instance. The following theorem provides a lower bound on the expected sample complexity of any \(\delta\)-PC algorithm.

\begin{theorem} \label{Theorem:lowerbound}
    Let $\delta \in (0, 1)$. For any $\delta-$PC algorithm $\pi$ and any problem instance $\boldsymbol{\mu} \in \Theta^{M}$, the expected sample complexity is lower bounded as
    \begin{equation}
        \mathbb{E}_{\boldsymbol{\mu}}^{\pi}\left[ \tau_\delta(\pi) \right] \geq (1-2\delta)\log\left(\frac{1-\delta}{\delta}\right) T^*(\boldsymbol{\mu}),
    \end{equation}
    where
    \begin{equation}\label{eq:T_star_mu}
        T^*(\boldsymbol{\mu}) \coloneqq \left[\sup_{\boldsymbol{w} \in \mathcal{P}_M} \inf_{\boldsymbol{\lambda} \in \mathrm{Alt}(\boldsymbol{\mu})}  \sum_{m=1}^M w_m  d_{\text{KL}}\left(\boldsymbol{\mu}_m, \boldsymbol{\lambda}_m\right)\right]^{-1},
    \end{equation}
    with $\boldsymbol{w} = [w_1, \ldots, w_M]$ representing a probability distribution over the arm set $[M]$ and belonging to the $M$-dimensional probability simplex $\mathcal{P}_M$, and $d_{\mathrm{KL}}(\boldsymbol{\mu}_1, \boldsymbol{\mu}_2)$ denoting the KL divergence between the probability distributions parameterized by $\boldsymbol{\mu}_1 \in \Theta$ and $\boldsymbol{\mu}_2 \in \Theta$. 
    Furthermore, 
    \begin{equation}
        \liminf_{\delta \rightarrow 0} \frac{\mathbb{E}_{\boldsymbol{\mu}}^{\pi}\left[ \tau_\delta(\pi) \right]}{\log\left( \frac{1}{\delta} \right)} \geq T^*(\boldsymbol{\mu}).
    \end{equation}
\end{theorem}
\begin{proof}
    See Appendix \ref{appsec: lowerbound}
\end{proof}

The proof of this theorem uses the transportation inequality presented in Lemma 1 of \cite{kaufmann2016complexity} and proceeds analogously to the proof of Theorem 1 in \cite{yang2024optimal}. 
In the lower bound expression, $\boldsymbol{w}$ can be understood as arm pull proportions, i.e., the fraction of times each arm is sampled. The solution to the $\sup-\inf$ problem in the lower bound identifies the optimal arm pull proportions to distinguish the true instance $\boldsymbol{\mu}$ from the most confusing alternative instance.
 The sampling rule design of our algorithm uses the solution to the $\sup-\inf$ problem in the lower bound expression. Let $\psi(\boldsymbol{w}, \boldsymbol{\mu})$ denote the inner infimum in \eqref{eq:T_star_mu}, i.e.,
\begin{equation} \label{eq:inmin}
    \psi(\boldsymbol{w}, \boldsymbol{\mu}) \coloneqq \inf_{\boldsymbol{\lambda} \in \text{Alt}(\boldsymbol{\mu})}  \sum_{m=1}^M w_m d_{\text{KL}}\left(\boldsymbol{\mu}_m, \boldsymbol{\lambda}_m\right).
\end{equation}
% Note that $\psi(\boldsymbol{w}, \boldsymbol{\mu})$ is continuous in its domain, and the following lemma formally claims this.

\begin{lemma} \label{proposition:innermincont}
    $\psi(\boldsymbol{w}, \boldsymbol{\mu})$ is continuous in its domain $\mathcal{P}_M \times \Theta^{M}$.
\end{lemma} 
\begin{proof}
    See Appendix \ref{appsec: innermincont}
\end{proof}
The proof of this Lemma proceeds analogously to the proof of Proposition 3 of \cite{prabhu2022sequential}.
Since $\psi(\boldsymbol{w}, \boldsymbol{\mu})$ is continuous and the probability simplex $\mathcal{P}_M$ is compact, we can replace $\sup$ with $\max$ in the expression of $T^*(\boldsymbol{\mu})$ in \eqref{eq:T_star_mu}. That is,
\begin{equation}
    T^*(\boldsymbol{\mu})^{-1} = \max_{\boldsymbol{w} \in \mathcal{P}_M} \psi(\boldsymbol{w}, \boldsymbol{\mu}).
\end{equation}

Note that the solution $\boldsymbol{w}$ at which the above optimization problem is maximized need not be unique. Hence, we call the set of optimal solutions $\mathcal{S}^*(\boldsymbol{\mu})$ and define it as 
\begin{equation} \label{eq:cstarmu}
    \mathcal{S}^*(\boldsymbol{\mu}) = \argmax_{w \in \mathcal{P}_M} \psi(\boldsymbol{w}, \boldsymbol{\mu}).
\end{equation}
From Berge's maximum Theorem \cite{berge1877topological, sundaram1996first}, we have $T^*(\boldsymbol{\mu})^{-1}$ is continuous on $\boldsymbol{\mu}$ and the correspondence $\mathcal{S}^*(\boldsymbol{\mu})$ is  upper hemicontinuous on $\boldsymbol{\mu}$. In addition, for all $\boldsymbol{\mu}$, the set $\mathcal{S}^*(\boldsymbol{\mu})$ is convex, compact, and non-empty.

\begin{remark} \label{rem:dtboc}
    %{\color{red}Also, write a remark on the non-uniqueness of the optimal solution. The next remark can be written as normal text.}

    Existing literature on BOC problems and hypothesis testing problems often consider scenarios where the optimal solution in \eqref{eq:cstarmu} is either unique or the correspondence $\mathcal{S}^*(\boldsymbol{\mu})$ admits a continuous selection. Uniqueness, along with the upper hemicontinuity property from \cite{berge1877topological, sundaram1996first}, ensures the existence of a continuous selection. In such cases, D-tracking or C-tracking rules from \cite{garivier2016optimal} can be used. For example, in \cite{yang2024optimal}, the optimal solution to \eqref{eq:cstarmu} is unique and the D-tracking rule was adopted. In the hypothesis testing problems considered in \cite{deshmukh2021sequential} uniqueness was assumed, and in \cite{prabhu2022sequential}, the existence of a continuous selection was assumed. 
    % \textcolor{red}{This allows them to design an optimal algorithm matching the lower bound as $\delta \rightarrow 0$.} 
    In our work, we neither have any conclusive result to show that a continuous selection exists nor are we assuming it. To address this, we adopt the average tracking rule for sampling. 
    % \textcolor{red}{However, this relaxation results in an additional factor of 2 in the sample complexity of the proposed algorithm (Theorems \ref{Theorem3:AlmostSureOptimal}, \ref{Theorem4:AsymptoticOptimal})}. 
\end{remark}

We use the expressions in \eqref{eq:inmin} and \eqref{eq:cstarmu} to formulate our stopping and sampling rules, respectively. Therefore, it is beneficial to simplify the inner $\inf$ problem for implementation purposes. The details of the simplification of the inner $\inf$ problem along with its complexity analysis have been discussed in section \ref{sec: compcomplx}.
% First, we write the inner $\inf$ problem \eqref{eq:inmin} as  the finite minimization of Quadratic Constrained Quadratic Program (QCQP) (Section \ref{appsec:innerminsimp} in Appendix). \textcolor{red}{Then, we solve each of QCQP using QCQP Algorithm proposed in \cite{lu2011kkt}. 
% For the special case of $d=1$, we simplified the inner $\inf$ problem as the minimum of a finite number of closed-form expressions. We skip the details for the special case of $d=1$.}
\section{Average Tracking Bandit Online Clustering (ATBOC) Algorithm and its performance } \label{sec:AvgTrackingOC}
To solve the BOC problem described in Section~\ref{sec:ProblemSetup}, we propose the Average Tracking Bandit Online Clustering (ATBOC) algorithm. Depending on the distribution of the arms, we consider three variants of ATBOC: ATBOC-Gauss, for multivariate Gaussian-distributed arms; ATBOC-subGauss, for multivariate sub-Gaussian arms other than Gaussian; and ATBOC-1pExp, for arms following a single-parameter exponential family of distributions.
From an implementation perspective, ATBOC-Gauss and ATBOC-subGauss are identical. However, when the arms are Gaussian distributed, the algorithm is {\em asymptotically order-optimal}. Hence, for clarity of exposition, we refer to the algorithm as ATBOC-Gauss when the arms are Gaussian distributed. In contrast, ATBOC-1pExp differs from ATBOC-Gauss/ATBOC-subGauss in its stopping threshold and is, as a consequence, {\em asymptotically optimal}.
Nevertheless, all three variants of ATBOC share the same algorithmic framework, and their pseudocode is presented in Algorithm~\ref{Algo:AvgTrackingOC}.

The ATBOC algorithm operates in time steps $t \in \mathbb{N}$. Let $N_m(t)$ denote the number of times that arm $m$ is sampled up to time $t$, and let $\boldsymbol{N}(t) \coloneqq [N_1(t), \ldots, N_M(t)]$. 
Let $\boldsymbol{X}_m(t)$ be the sample generated from the $m^{\text{th}}$ arm at time $t$.
Denote the arm sampled at time $t$ by $A_t$.
The estimated parameter of arms at time $t$ is $\boldsymbol{\hat{\mu}}(t) \coloneqq [\hat{\boldsymbol{\mu}}_1(t), \ldots, \hat{\boldsymbol{\mu}}_M(t)]$, where $\hat{\boldsymbol{\mu}}_m(t)$ represents the estimated parameter of arm $m$ at time $t$, and is given as follows. 
\begin{equation} \label{eq: paraestimate}
    \hat{\boldsymbol{\mu}}_m(t) = \begin{cases}
        \frac{1}{N_m(t)}\sum_{s\in[t]:A_s=m} \boldsymbol{X}_m(s) &\text{ for ATBOC-Gauss and ATBOC-subGauss} \\
        \dot{A}^{-1}\left(\frac{1}{N_m(t)}\sum_{s\in[t]:A_s=m}T(\boldsymbol{X}_m(s))\right) &\text{ for ATBOC-1pExp}.
    \end{cases}
\end{equation}
Here, $\dot{A}^{-1}(\cdot)$ represents the inverse of the first derivative of the log partition function, and $T(\cdot)$ is the sufficient statistic in the canonical form of the single-parameter exponential family of distributions (Section 1.5 in \cite{lehmann1998theory}).
To estimate parameters, we use the sample mean for sub-Gaussian arms (ATBOC-Gauss/ATBOC-subGauss) and the maximum likelihood estimate for a single-parameter exponential family distribution (ATBOC-1pExp).
We use $Z(t)$ and $\beta(\delta, t)$ to denote the test statistics and the stopping rule threshold at time $t$, respectively, where $\delta$ is the probability of error.
% Define $Z(t) \coloneqq t\psi\left(\frac{\boldsymbol{N}(t)}{t}, \boldsymbol{\hat{\mu}}(t)\right)$ as the log of the Generalized Likelihood Ratio (GLR) statistic and $\beta(\delta, t) \coloneqq 2 \log\left(\frac{\sqrt{\prod_{m=1}^M (N_m(t)+1)^d}}{\delta}\right)$ as the threshold of the stopping rule at time $t$, where $\delta$ is the error probability. 
Let $\boldsymbol{w}^{*}(t) \coloneqq [w_1^{*}(t), \ldots, w_M^{*}(t)]$ denote the optimal arm pull proportions in \eqref{eq:cstarmu} obtained when the estimated  $\boldsymbol{\hat{\mu}}(t)$ is plugged-in. For $\boldsymbol{w} \in \mathbb{R}^M$, the support is $\text{supp}(\boldsymbol{w}) \coloneqq \{m \in [M] : w_m \neq 0\}$.

\begin{algorithm} 
%\caption{Average Tracking Bandit Online Clustering (\textit{ATBOC}) Algorithm }\label{Algo:AvgTrackingOC}
\caption{\textit{ATBOC}}\label{Algo:AvgTrackingOC}
\begin{algorithmic}[1]
\STATE \textbf{Input:} $\delta, M, K$
\STATE \textbf{Initialize:} $t=0$, $N_m(0)=0 \text{ for all } m \in [M]$

\REPEAT
    \IF{$\displaystyle \min_{m \in [M]} N_m(t) < \sqrt{\frac{t}{M}}$}
        \STATE $\displaystyle A_{t+1} = \argmin_{m \in [M]} N_m(t)$
    \ELSE
        \STATE $\displaystyle A_{t+1} = \argmin_{m \in \mathrm{supp}\left( \sum_{s=1}^t w^{*}(s) \right)}  \left[ \frac{N_m(t)}{t} - \frac{1}{t}\sum_{s=1}^t w^{*}_m(s) \right] $
    \ENDIF
    \STATE sample arm $A_{t+1}$
    \STATE $t \leftarrow t+1$, update $\boldsymbol{\hat{\mu}}_m(t)$ and $N_m(t)$ for all $m \in [M]$
    \STATE Compute the test statistics $Z(t)$ in \eqref{eq: z} and the threshold $\beta\left( \delta, t \right)$ in \eqref{eq:thresh}.
    \STATE Compute the optimal arm pull proportion estimate $\boldsymbol{w}^{*}(t)$ in \eqref{eq: armpullestimate} corresponding to $\hat{\boldsymbol{\mu}}(t)$.

\UNTIL{$Z(t) \geq \beta(\delta, t)$} %\text{   (using \eqref{eq:thresh}) }

\STATE $\boldsymbol{\hat{c}} = \mathcal{C}\left( \boldsymbol{\hat{\mu}}(t) \right)$
\STATE \textbf{Output:} $\boldsymbol{\hat{c}}$
\end{algorithmic}
\end{algorithm}

\subsection{Algorithm Description}
%{\color{red}The inputs to Algorithm \ref{Algo:AvgTrackingOC} are: $\delta$ - maximum probability of error the algorithm can incur, $K$ - number of clusters, and $[M]$ - arm set.
%The output of the algorithm is
%%    \item 
    %$\boldsymbol{\hat{c}}$ - estimated cluster index vector.}
    The inputs to ATBOC are the error probability $\delta$, the number of clusters $K$, and the number of arms $M$, and the output of the algorithm is the estimated cluster index vector $\boldsymbol{\hat{c}}$ (Lines 1 and 15 in Algorithm~\ref{Algo:AvgTrackingOC}).
During each time step $t$, the algorithm does the following:
\begin{itemize}
    \item \textit{Sampling rule:} select an arm to sample, $A_{t+1}$.
    \item \textit{Stopping rule:} stop the algorithm if $Z(t) \geq \beta(\delta, t)$.
    \item \textit{Decision rule:} declare the output $\boldsymbol{\hat{c}}$ upon stopping.
\end{itemize}
% \noindent
Now we discuss each of these steps in detail, along with the theoretical consequences.
\subsubsection{Sampling Rule}
% \textbf{Sampling Rule: }
%The sampling rule used is the average tracking strategy proposed in \cite{jedra2020optimal}. The main idea of the tracking strategy is to make the arm pull proportions $\frac{\boldsymbol{N}(t)}{t}$ track the optimal arm pull proportions $\boldsymbol{w}^*$, for some $\boldsymbol{w}^* \in C^*(\boldsymbol{\mu})$  \eqref{eq:cstarmu}.
%However, finding $C^*(\boldsymbol{\mu})$ requires knowledge of the true mean vectors $\boldsymbol{\mu}$. Hence, plug-in estimates $\boldsymbol{w}(t)$, which are computed using the empirical mean vectors $\boldsymbol{\hat{\mu}}(t)$, are used in place of $\boldsymbol{w}^*$. In the average tracking strategy, at each time $t$, the arm $m$ for which {\color{red}$\frac{\boldsymbol{N}_m(t)}{t}$} is lagging the {\em average} of the plug-in estimates $w_m(t)$ upto $t$ the most is selected (See Line 7 of Algorithm~\ref{Algo:AvgTrackingOC}). The plug-in estimate is computed in Line 12.

 The sampling rule for the ATBOC is motivated by the average-tracking strategy proposed in \cite{jedra2020optimal} for the BAI problem in linear bandits. The main idea of the tracking-based sampling rule is to make the empirical arm pull proportions $\frac{\boldsymbol{N}(t)}{t}$ track the optimal arm pull proportions $\boldsymbol{w}^*$, for some $\boldsymbol{w}^* \in \mathcal{S}^*(\boldsymbol{\mu})$ in \eqref{eq:cstarmu}. Since determining $\mathcal{S}^*(\boldsymbol{\mu})$ requires knowledge of the true parameters $\boldsymbol{\mu}$, plug-in estimates $\boldsymbol{w}^{*}(t)$, computed from the estimated parameters $\boldsymbol{\hat{\mu}}(t)$ are used instead, as follows.
 % \textcolor{blue}{
 \begin{equation} \label{eq: armpullestimate}
    \boldsymbol{w}^{*}(t) \in \begin{cases}
    
        \argmax_{w \in \mathcal{P}_M} \psi_{\text{mod}}\left(\boldsymbol{w}, \hat{\boldsymbol{\mu}}_m(t)\right) &\text{for ATBOC-Gauss/ATBOC-SubGauss} \\
        \argmax_{w \in \mathcal{P}_M}\psi\left(\boldsymbol{w}, \hat{\boldsymbol{\mu}}_m(t)\right) &\text{for ATBOC-1pExp}
    \end{cases},
 \end{equation}
 % }
 where 
 \begin{equation} \label{eq: modinnerinf}
    \psi_{mod}(\boldsymbol{w}, \boldsymbol{\mu}) \coloneqq \frac{1}{2\sigma^2}\inf_{\boldsymbol{\lambda} \in \text{Alt}(\boldsymbol{\mu})}  \sum_{m=1}^M w_m \left\|\boldsymbol{\mu}_m- \boldsymbol{\lambda}_m\right\|^2.
\end{equation}
Note that for the Gaussian distributed arms (ATBOC-Gauss), $\psi_{\text{mod}(\cdot, \cdot)}$ is the same as $\psi(\cdot, \cdot)$. For ATBOC-subGauss, $\psi_{\text{mod}}(\cdot,\cdot)$ is obtained by replacing the KL divergence in $\psi(\cdot,\cdot)$ by the squared norm distance.
% ATBOC-subGauss, to compute the arm pull proportion estimate, maximizes the modified infimum problem defined in \eqref{eq: modinnerinf}. Consequently, for ATBOC-subGauss, the empirical arm pull proportions $\frac{\boldsymbol{N}(t)}{t}$ converges to a set 
% Such modification for ATBOC-subGauss, in computing the arm pull proportion estimate is required as the test statistics in the stopping rule of the ATBOC-subGauss is of the form $\psi_{\text{mod}}(\cdot, \cdot)$, and $\mathcal{S}^{*}_{\text{mod}}(\cdot)$ is the set of arm pull proportion that maximizes $\psi_{\text{mod}}(\cdot, \cdot)$.
 At each time $t$, the arm $m$ whose $\frac{N_m(t)}{t}$ lags the average of the plug-in estimates $w^{*}_m(t)$ up to $t$ the most is chosen (see Line 7 of Algorithm~\ref{Algo:AvgTrackingOC}).

To ensure that the estimated parameters $\boldsymbol{\hat{\mu}}(t)$ converge to the true parameters $\boldsymbol{\mu}$, we perform forced exploration as follows. At each time $t$, if there exists arms whose number of samples until time $t$ is less than $\sqrt{\frac{t}{M}}$, then the sampling rule will select the arm with the fewest number of samples among these arms (Line 5 of Algorithm~\ref{Algo:AvgTrackingOC}).
 
Lemma \ref{prop:armpullpropconverge} discusses the convergence properties of the empirical arm pull proportions $\frac{\boldsymbol{N}(t)}{t}$ for the sampling rule with forced exploration and average tracking described above. 
For any $\boldsymbol{w} \in \mathbb{R}^M$
and for any compact set $C \subseteq \mathbb{R}^M$, let
$
d_\infty(\boldsymbol{w}, C) \coloneqq \min_{\boldsymbol{w}' \in C} \|\boldsymbol{w}- \boldsymbol{w}'\|_\infty.
$ 

% For any $\boldsymbol{w}, \boldsymbol{w}' \in \mathbb{R}^M$, let $
% d_\infty(\boldsymbol{w}, \boldsymbol{w}') \coloneqq \max_{m \in [M]} \left| w_m - w_m' \right|
% $
% and for any compact set $C \subseteq \mathbb{R}^M$, let
% $
% d_\infty(\boldsymbol{w}, C) \coloneqq \min_{\boldsymbol{w}' \in C} \|\boldsymbol{w}, \boldsymbol{w}'\|_\infty.
% $ 
\begin{lemma} \label{prop:armpullpropconverge}
    For any problem instance $\boldsymbol{\mu}\in\Theta^M$, 
    \begin{enumerate}
        \item ATBOC-Gauss and ATBOC-1pExp satisfy 
        \begin{equation}
            \lim_{t \rightarrow \infty} d_\infty\left( \frac{\boldsymbol{N}(t)}{t}, \mathcal{S}^*(\boldsymbol{\mu}) \right) = 0 \quad \text{a.s.}
        \end{equation}
        \item ATBOC-subGauss satisfies
        \begin{equation}
            \lim_{t \rightarrow \infty} d_\infty\left( \frac{\boldsymbol{N}(t)}{t}, \mathcal{S}_{\text{mod}}^*(\boldsymbol{\mu}) \right) = 0 \quad \text{a.s.},
        \end{equation}
    \end{enumerate}
    where
\begin{equation} \label{eq:modcstarmu}
    \mathcal{S}_{\text{mod}}^*(\boldsymbol{\mu}) = \argmax_{w \in \mathcal{P}_M} \psi_{\text{mod}}(\boldsymbol{w}, \boldsymbol{\mu}).
\end{equation}
    
%    For the ATBOC algorithm, the arm pull proportions $\frac{\boldsymbol{N}(t)}{t}$ converge to the set $\mathcal{S}^*(\boldsymbol{\mu})$ defined in \eqref{eq:cstarmu} almost surely (a.s.), i.e.,
% \[
% \lim_{t \rightarrow \infty} d_\infty\left( \frac{\boldsymbol{N}(t)}{t}, \mathcal{S}^*(\boldsymbol{\mu}) \right) = 0 \quad \text{a.s.}
% \]
\end{lemma}
\begin{proof}
    See Appendix \ref{appsec: armpullpropconverge}
\end{proof}
In ATBOC-Gauss and ATBOC-1pExp, the tracking sequence $\boldsymbol{w}^{*}(t)$ is chosen by maximizing the inner infimum function $\psi(\cdot,\cdot)$ associated with the lower-bound expression in \eqref{eq:T_star_mu}. Consequently, the empirical arm-pull proportions $\frac{\boldsymbol{N}(t)}{t}$ converge to the optimal arm-pull proportions given in \eqref{eq:cstarmu}.
In contrast, in ATBOC-subGauss, a modified inner infimum problem $\psi_{\text{mod}}(\cdot,\cdot)$ is used, in which the KL divergence in $\psi(\cdot,\cdot)$ is replaced by the squared norm distance. This modification leads to the convergence of the empirical arm-pull proportions $\frac{\boldsymbol{N}(t)}{t}$ to the modified set in \eqref{eq:modcstarmu}.
 % \noindent
% \textbf{Stopping Rule: }

\subsubsection{Stopping Rule}
The test statistic $Z(t)$ of ATBOC is based on the $\log$ of the Generalized Likelihood Ratio (GLR).
% The algorithm stops when the log of the GLR statistic is greater than or equal to the threshold. 
The $\log$-GLR statistic computes the log likelihood ratio for the estimated cluster index vector against its closest alternative. 
% We present the details of $Z(t)$ and the threshold $\beta(\delta, t)$ now.
Let $\mathbb{P}_{\lambda_m}(\cdot)$ be the probability distribution with parameter $\lambda_m$.
% We write $\underline{\boldsymbol{X}}_m^i$ to denote the $i^{\text{th}}$ sample from the $m^{\text{th}}$ arm. The estimated cluster index vector at time $t$ is $\mathcal{C}\left( \boldsymbol{\hat{\mu}}(t) \right)$. Now, 
The $\log$-GLR is given by
\begin{align}
   & \log\left[\frac{ \displaystyle \max_{\substack{\boldsymbol{\lambda}: \mathcal{C}(\boldsymbol{\lambda})\sim  \mathcal{C}\left( \boldsymbol{\hat{\mu}}(t) \right)}} \left(\prod_{m\in[M]}\prod_{s\in[t]:A_s=m}\mathbb{P}_{\lambda_m}(\boldsymbol{X}_m(t)) \right)}{\displaystyle \max_{\boldsymbol{\lambda}\in\text{Alt}(\boldsymbol{\hat{\mu}}(t))} \left(\prod_{m\in[M]}\prod_{s\in[t]:A_s=m}\mathbb{P}_{\lambda_m}(\boldsymbol{X}_m(t))\right)}\right]. \nonumber
\end{align}
The above expression after simplification yields the following.
\begin{align}   
    % Z(t) = 
    -t\min_{\substack{\boldsymbol{\lambda}: \mathcal{C}(\boldsymbol{\lambda})\sim  \mathcal{C}\left( \boldsymbol{\hat{\mu}}(t) \right)}} \left(\sum_{m=1}^M \frac{N_m(t)}{t} d_{\text{KL}}\left(\hat{\boldsymbol{\mu}}_m(t), \boldsymbol{\lambda}_m\right)\right)   
    +t\min_{\boldsymbol{\lambda}\in\text{Alt}(\boldsymbol{\hat{\mu}}(t))} \left(\sum_{m=1}^M \frac{N_m(t)}{t} d_{\text{KL}}\left(\hat{\boldsymbol{\mu}}_m(t), \boldsymbol{\lambda}_m\right)\right). \nonumber
\end{align} 
% \textcolor{blue}{
The first term in the above-derived expression becomes 0 on setting $\boldsymbol{\lambda_m} = \boldsymbol{\hat{\mu}}_m(t)$ for all $m \in [M]$. 
The second term takes the form of the inner minimization problem \eqref{eq:inmin} in the lower bound expression. Hence, the $\log$-GLR takes the form $t\psi\left(\frac{\boldsymbol{N}(t)}{t}, \boldsymbol{\hat{\mu}}(t)\right)$.
% \fi
We use the log-GLR obtained above as the test statistic $Z(t)$ for the single-parameter exponential family of distributions (ATBOC-1pExp). For sub-Gaussian distributions (ATBOC-Gauss/ATBOC-subGauss), we use $t\psi_{\text{mod}}\left(\frac{\boldsymbol{N}(t)}{t}, \boldsymbol{\hat{\mu}}(t)\right)$ as the test statistic $Z(t)$. Note that for Gaussian distributed arms $t\psi_{\text{mod}}\left(\frac{\boldsymbol{N}(t)}{t}, \boldsymbol{\hat{\mu}}(t)\right)$ is equal to log-GLR.
% same as the form one obtained for the mulivariate Gaussian distribution. 
The test statistic $Z(t)$ is given by:
\begin{equation} \label{eq: z}
   Z(t) = \begin{cases}
       t\psi_{mod}\left(\frac{\boldsymbol{N}(t)}{t}, \hat{\boldsymbol{\mu}}(t)\right) &\text{ for ATBOC-Gauss/ATBOC-SubGauss} \\
       t\psi\left(\frac{\boldsymbol{N}(t)}{t}, \hat{\boldsymbol{\mu}}(t)\right) &\text{ for ATBOC-1pExp}
   \end{cases}. 
\end{equation}
A time uniform KL-based deviation inequality is required to prove the $\delta$-PC guarantee. However, such a deviation inequality is not available for the multivariate sub-Gaussian distribution. Instead, a squared distance-based deviation inequality is available. Hence, ATBOC-SubGauss uses the modified function $\psi_{\text{mod}}(\cdot, \cdot)$, which is similar to the inner infimum function $\psi(\cdot, \cdot)$, but measures the weighted squared distance between the problem instance and its closest alternative.
% }
%This statistic is computed in Line 11 of Algorithm \ref{Algo:AvgTrackingOC}.
The threshold used in the stopping rule is
\begin{equation} \label{eq:thresh}
    \beta\left( \delta, t \right) = \begin{cases}
    2 \log\left( \frac{\sqrt{\prod_{m=1}^M\left(N_m(t)+1\right)^d}}{\delta} \right) &\text{ for ATBOC-Gauss/ATBOC-SubGauss}\\
        3 \sum_{m=1}^M \log\left(1 + \log\left(N_m(t)\right)\right) 
    + (1 + \zeta)\log\left(\frac{1}{\delta}\right) \\
    \hspace{2cm} + (1 + \zeta)M \log\left(\frac{\pi^2/3}{\left(\log(1 + \zeta)\right)^2}\right) &\text{ for ATBOC-1pExp}
    \end{cases}.
\end{equation}
The choice of this threshold is based on the time uniform deviation inequalities for the respective distributions to ensure $\delta$-PC guarantee.
The algorithm stops if the condition $Z(t) \geq \beta(\delta, t)$ is satisfied (Line 13 of Algorithm \ref{Algo:AvgTrackingOC}).
Using the test statistic in \eqref{eq: z} and the threshold \eqref{eq:thresh}, the formal expression for the stopping time of the algorithm is given by, 
$    \tau_\delta(\text{ATBOC}) = \tau_\delta = \inf\left\{ t \in \mathbb{N}: Z(t) \geq \beta(\delta, t) \right\}.$\\
% \noindent
\subsubsection{Declaration Rule }
% \textbf{Declaration Rule: }
When the algorithm stops, it declares the estimated cluster index vector $\boldsymbol{\hat{c}}$, based on the estimated parameters of the arms $\boldsymbol{\hat{\mu}}(\tau_{\delta})$ using the relation $\mathcal{C}$ discussed in Section \ref{sec:ProblemSetup}, i.e., $\boldsymbol{\hat{c}} = \mathcal{C}(\boldsymbol{\hat{\mu}}(\tau_{\delta}))$. Hence, we can consider the relation $\mathcal{C}$ as our \textit{declaration rule}. 

\subsection{ATBOC Algorithm Performance} \label{sec:AlgorithmPerformance}
The performance guarantees provided in this section include the correctness of the cluster index vector estimate $\boldsymbol{\hat{c}}$ and upper bounds on the sample complexity $\tau_\delta$. Theorem \ref{Theorem2:deltaPAC} confirms that the algorithm stops in finite time and that its error probability is upper bounded by $\delta$.
\begin{theorem} \label{Theorem2:deltaPAC}
    All three variants of the proposed ATBOC algorithm: ATBOC-Gauss, ATBOC-subGauss, and ATBOC-1pExp, stop in finite time and are $\delta-$ PC algorithms, i.e.
    \begin{equation}
    \begin{aligned}
        \mathbb{P}_{\boldsymbol{\mu}}\left[ \tau_\delta < \infty \right] = 1 \text{ and } \mathbb{P}_{\boldsymbol{\mu}}\left[ \mathcal{C}\left( \boldsymbol{\hat{\mu}}\left( \tau_\delta \right) \right) \nsim \mathcal{C}\left( \boldsymbol{\mu} \right) \right] \leq \delta.
    \end{aligned}
    \end{equation}
\end{theorem}
\begin{proof}
    See Appendix \ref{appsec: deltapac}.
\end{proof}
% \begin{remark}
%     Any BOC algorithm with a stopping rule of the form $Z(t) \geq \beta(\delta, t)$, where $Z(t)$ and $\beta(\delta, t)$ are given by equations \eqref{eq:zt} and \eqref{eq:thresh}, respectively, is a $\delta-$PC algorithm. In Section \ref{sec:LUCB}, we define a more computationally efficient algorithm called the LUCBBOC algorithm, using the same stopping rule; hence, LUCBBOC is also a $\delta-$PC algorithm.
% \end{remark}
Now, we discuss the upper bound on the stopping time $\tau_\delta$ of the three variants of the ATBOC algorithm.  
Let us define $\displaystyle T^*_{mod}(\boldsymbol{\mu}) = \left[\max_{\boldsymbol{w} \in \mathcal{P}_M} \psi_{mod}(\boldsymbol{w}, \boldsymbol{\mu})\right]^{-1}.$
\begin{theorem} \label{Theorem3:AlmostSureOptimal}
    (Almost sure sample complexity upper bound) For any problem instance $\boldsymbol{\mu}$, 
    \begin{enumerate}
        \item ATBOC-SubGauss satisfies
        \begin{equation}
        \mathbb{P}_{\boldsymbol{\mu}}\left[ \limsup_{\delta \rightarrow 0} \frac{\tau_\delta}{\log\left( 
        \frac{1}{\delta} \right)} \leq 2T_{mod}^*(\boldsymbol{\mu})\right] = 1.
    \end{equation}
    \item ATBOC-Gauss and ATBOC-1pExp satisfy
    \begin{equation}
        \mathbb{P}_{\boldsymbol{\mu}}\left[ \limsup_{\delta \rightarrow 0} \frac{\tau_\delta}{\log\left( 
        \frac{1}{\delta} \right)} \leq aT^*(\boldsymbol{\mu})\right] = 1,
    \end{equation}
    where
    \begin{equation}
        a = \begin{cases}
            2 &\text{ for ATBOC-Gauss} \\
            1+\zeta &\text{ for ATBOC-1pExp}
        \end{cases}, \text{ for } \zeta \in \left(0, \frac{1}{2}\right).
    \end{equation}
    \end{enumerate}
    
    % \begin{equation}
    %     \mathbb{P}_{\boldsymbol{\mu}}^{\text{ATBOC}}\left( \limsup_{\delta \rightarrow 0} \frac{\tau_\delta}{\log\left( 
    %     \frac{1}{\delta} \right)} \leq 2T^*(\boldsymbol{\mu})\right) = 1.
    % \end{equation}
\end{theorem}
\begin{proof}
    See Appendix \ref{appsec: almostsureoptimal}.
\end{proof}

The above result provides an almost sure upper bound on the stopping time $\tau_\delta$ of the ATBOC algorithm. We also derive an upper bound on the expected stopping time $\mathbb{E}_{\boldsymbol{\mu}}[\tau_\delta]$ of the ATBOC algorithm in Theorem \ref{Theorem4:AsymptoticOptimal}.

\begin{theorem} \label{Theorem4:AsymptoticOptimal}
    (Expected sample complexity upper bound) For any problem instance $\boldsymbol{\mu}$, 
    \begin{enumerate}
        \item ATBOC-SubGauss satisfies
        \begin{equation}
        \limsup_{\delta \rightarrow 0} \frac{\mathbb{E}_{\boldsymbol{\mu}}[\tau_\delta]}{\log\left( \frac{1}{\delta} \right)} \leq 2T_{mod}^*(\boldsymbol{\mu}).
    \end{equation}
    \item ATBOC-Gauss and ATBOC-1pExp satisfy
    \begin{equation}
        \limsup_{\delta \rightarrow 0} \frac{\mathbb{E}_{\boldsymbol{\mu}}[\tau_\delta]}{\log\left( \frac{1}{\delta} \right)} \leq aT^*(\boldsymbol{\mu}).
    \end{equation}
    \end{enumerate}
\end{theorem}
\begin{proof}
    See Appendix \ref{appsec: asymptoticoptimal}.
\end{proof}

The factor of $2$ for ATBOC-Gauss/ATBOC-subGauss arises from the time-uniform concentration inequality derived for multivariate sub-Gaussian distribution (Theorem 1 of \cite{abbasi2011improved}) that we use to choose the stopping threshold while proving the algorithm to be $\delta-$PC. We use a tighter time uniform KL-based concentration bound for the case of a single parameter exponential family of distributions using \cite[Lemma~13]{kaufmann2021mixture}, which relates the KL-based deviation in the estimated parameter $\hat{\boldsymbol{\theta}}(t)$ to a suitably chosen mixture martingale and Ville's inequality \cite{huang2025sequential}. Using this tighter bound in ATBOC-1pExp for clustering the single-parameter exponential family distributed arms, we get a factor of $(1 + \zeta)$, where $\zeta$ can be arbitrarily small, instead of the factor of $2$.

    By comparing the asymptotic slope of the expected sample complexity upper bound of the ATBOC algorithms given in Theorem \ref{Theorem4:AsymptoticOptimal} with the asymptotic lower bound slope given in Theorem \ref{Theorem:lowerbound}, we infer the following. 1) ATBOC-Gauss is asymptotically order-optimal with a multiplicative gap of $2$. 2) ATBOC-subGauss is asymptotically suboptimal, with its asymptotic slope differing from the asymptotically order-optimal slope by $2\left|T^{*}_{\text{mod}}(\boldsymbol{\mu})-T^{*}(\boldsymbol{\mu})\right|$. 3) ATBOC-1pExp has tunable multiplicative gap $1+\zeta$, for some $\zeta\in\left(0, \frac{1}{2}\right)$, and on setting $\zeta$ arbitrarily close to $0$, it becomes asymptotic optimal.
    
    % see that the multiplicative gap between the upper and lower bounds is at most 2, which is independent of the problem parameters. Hence, ATBOC is order optimal.
 \begin{remark}
    The ATBOC algorithms are $\delta$-PC and asymptotically order-optimal for multivariate Gaussian and single-parameter exponential family distributions. Moreover, it is tailored to solve a more general $K$-cluster problem than other works in the BOC literature. However, the requirement to solve an optimization problem at each time step $t$ for the sampling rule (Lines 7 and 12 in Algorithm \ref{Algo:AvgTrackingOC}) results in high computational complexity. Hence, in the following sections, we propose the LUCBBOC and BOC-ELIM algorithms, which are more computationally efficient than ATBOC. 
    % Additionally, simulation results show that LUCBBOC is comparable in performance to ATBOC (see Section \ref{sec:Simulations} for details).
\end{remark}

\section{Lower and Upper Confidence Bound-based Bandit Online Clustering (LUCBBOC) Algorithm } \label{sec:LUCB}
% The pseudo-code for the LUCBBOC algorithm is given in Algorithm \ref{Algo:LUCBOC}. 
In this section, we propose the LUCBBOC algorithm, for clustering the multivariate sub-Gaussian distributed arms. 
The declaration and stopping rules in the LUCBBOC algorithm are the same as those of the ATBOC algorithm. Additionally, we have retained the forced exploration step from ATBOC. However, we replace the computationally heavy average tracking step (Lines 7 and 12 in Algorithm \ref{Algo:AvgTrackingOC}) with a procedure inspired by the LUCB algorithm for best arm identification problem from \cite{kalyanakrishnan2012pac}. Note that in the ATBOC algorithm, at any time $t$, when  forced exploration is not required, Lines 7 and 12 in Algorithm \ref{Algo:AvgTrackingOC} use $\boldsymbol{\hat{\mu}}(t)$ and $\boldsymbol{N}(t)$ to suggest the next arm to select $A_{t+1}$. Here, we propose an alternative block of procedures, which we call \textit{LUCBBOC-sampling}, that takes $\boldsymbol{\hat{\mu}}(t)$ and $\boldsymbol{N}(t)$ as inputs and outputs the arm to select next $A_{t+1}$. The pseudo-code for \textit{LUCBBOC-sampling} is given in Algorithm \ref{Algo:LUCBOC}.
In general, the proposed LUCBBOC algorithm is the same as that of ATBOC (Algorithm \ref{Algo:AvgTrackingOC}) except that Lines 7, 12 in Algorithm \ref{Algo:AvgTrackingOC} are replaced with \textit{the} LUCBBOC sampling block presented in Algorithm \ref{Algo:LUCBOC}. %The description of \textit{LUCBBOC-sampling} block is as follows.
Let $\alpha_m(t)$ denote the confidence radius\footnote{The confidence radius presented here is valid for arms following multivariate sub-Gaussian distribution. By using an appropriate confidence radius for single-parameter exponential family of distributions, LUCBBOC algorithm can be applied for single parameter exponential family distributed arms.} of arm $m$ at time $t$, defined as $
\alpha_m(t) \coloneqq \sqrt{\frac{2}{N_m(t)} \log\left( \frac{2^{d+1}M N_m^2(t)}{\delta} \right)}.$ The expression for the confidence radius $\alpha_m(t)$ is chosen such that the true mean $\boldsymbol{\mu}_m$ lies in the ball of radius $\alpha_m(t)$ with the center being $\hat{\boldsymbol{\mu}}_m(t)$ with probability at least $1 - \delta$. Specifically, for $\delta \in (0, 1)$,
$$    {\mathbb{P}\left[ \bigcap_{t \in \mathbb{N}}\bigcap_{m \in [M]} \left\{ \left\| \hat{\boldsymbol{\mu}}_m(t) - \boldsymbol{\mu}_m \right\| \leq \alpha_m(t) \right\} \right] \geq 1 - \delta.}$$ 
We write $E_{ij}(t)$ to denote the empirical gap between the arms $i$ and $j$ at time step $t$. Let $U_{ij}(t)$ and $L_{ij}(t)$ denote the UCB and LCB of the gap between the arms $i$ and $j$ at time $t$, respectively. Define $E_{ij}(t)$, $U_{ij}(t)$ and $L_{ij}(t)$ as:
\begin{align}
    E_{ij}(t) \coloneqq \left\| \hat{\boldsymbol{\mu}}_i(t) - \hat{\boldsymbol{\mu}}_j(t) \right\|, \ \ i, j \in [M]. \label{eq:Eit} \\
    U_{ij}(t)\coloneqq E_{ij}(t) +\alpha_i(t) +\alpha_j(t), \ \ i, j \in [M]. \label{eq:Uit} \\
    L_{ij}(t) \coloneqq E_{ij}(t) - \alpha_i(t) - \alpha_j(t), \ \ i, j \in [M]. \label{eq:Lit}
\end{align}
%At each time $t$, we write $n_p$ and $m_q$ to denote the pair of arms with smallest empirical gap, from the clusters $p$ and $q$ respectively. We find the pair of cluster indices $(p^{*}, q^{*})$, whose pair of arms with the smallest empirical gap ($n_{p^{*}}, m_{q^{*}}$) has the smallest LCB gap (line 3 of Algorithm \ref{Algo:LUCBOC}). Using the SLINK algorithm on the empirical mean vectors corresponding to the $k^{th}$ cluster, we divide it into two partitions, $D_{k_1}$ and $D_{k_2}$. (Line 4 of Algorithm \ref{Algo:LUCBOC}). We write $a_k$ and $b_k$ to denote the pair of arms with the smallest empirical gap, from the partitions $D_{k_1}$ and $D_{k_2}$, respectively, and we call this gap the maximum intra-cluster gap of the $k^{th}$ cluster.  We find the cluster index $k^{*}$ that corresponds to the maximum UCB of the maximum intra-cluster gap (Line 5 of Algorithm \ref{Algo:LUCBOC}). We want the smallest inter-cluster gap to be greater than the largest intra-cluster gap. Hence, the potential set of arms to sample $\mathcal{A}$ are the arms corresponding to the smallest inter-cluster LCB gap and the highest intra-cluster UCB gap (Line 6 of Algorithm \ref{Algo:LUCBOC}). At the next time $t+1$, the sampling rule selects an arm $m \in \mathcal{A}$ that has been sampled the fewest times (Line 7 of Algorithm \ref{Algo:LUCBOC}).%, i.e., $A_{t+1} = \arg\min_{a \in \mathcal{A}} N_a(t)$.

The basic idea behind the LUCBBOC sampling rule is to identify two types of arm pairs: pair of arms that belong to different clusters yet are close to each other and pair of arms that belong to the same cluster yet are far apart. We refer to the arms belonging to these pairs as potential arms, since the former type has the potential to merge two clusters, while the latter type has the potential to split a cluster. We then select a potential arm with the smaller number of arm pulls to obtain a new sample.

At any time $t$, we use the SLINK clustering algorithm to find the estimated cluster index vector $\hat{\boldsymbol{c}}(t) = \mathcal{C}\left(\hat{\boldsymbol{\mu}}(t)\right)$. For any $k\in [K]$, let $\hat{D}_k(t)$ denote the set of arms that belong to the estimated cluster $k$ at time $t$, i.e., $\hat{D}_k(t)=\left\{ m\in[M]: \hat{c}_m(t)=k \right\}$.
Let $n_p$ and $m_q$ denote the pair of arms with the smallest empirical gap from clusters $p$ and $q$, respectively, at time $t$. Cluster indices $(p^{*}, q^{*})$ are chosen so that their corresponding pair of arms with the smallest empirical gap $(n_{p^{*}}, m_{q^{*}})$ has the smallest LCB gap (Line 3, Algorithm~\ref{Algo:LUCBOC}).
We refer to the arms $(n_{p^{*}}, m_{q^{*}})$ as the potential arms that belong to different clusters yet are close.

For all $k\in [K]$, using the SLINK algorithm on the empirical mean vectors of cluster $k$, $\left\{\hat{\boldsymbol{\mu}}_m(t), m\in \hat{D}_k(t)\right\}$, we split it into partitions $D_{k,1}$ and $D_{k,2}$ (Line 5). Let $a_k$ and $b_k$ denote the pair of arms with the smallest empirical gap from the partitions $D_{k,1}$ and $D_{k,2}$, respectively, and we refer to this gap as the intra-cluster neighbor gap of cluster $k$ (Line 6). The cluster $k^{*}$ is selected based on the maximum UCB of the intra-cluster neighbor gaps (Line 8). We refer to the arms $(a_{k^{*}}, b_{k^{*}})$ as the potential arms that belong to the same cluster yet are far apart.
At time $t$, the set of potential arms $\mathcal{A}$ consists of pairs of arms corresponding to the smallest inter-cluster LCB gap $(n_{p^{*}}, m_{q^{*}})$ and the largest intra-cluster UCB gap $(a_{k^{*}}, b_{k^{*}})$ (Line 9). At $t+1$, the arm $m \in \mathcal{A}$ with the fewest arm pulls is selected (Line 10).

\begin{algorithm}
% \caption{Lower and Upper Confidence Bound-based Bandit Online Clustering (\textit{LUCBBOC}) Algorithm}\label{Algo:LUCBOC}
\caption{\textit{LUCBBOC-sampling}} \label{Algo:LUCBOC}
\begin{algorithmic}[1]
\STATE \textbf{Input:} $\boldsymbol{\hat{\mu}}(t), \boldsymbol{N}(t).$
\STATE Compute $E_{ij}(t)$, $U_{ij}(t)$ and $L_{ij}(t)$ using \eqref{eq:Eit}, \eqref{eq:Uit} and \eqref{eq:Lit}, respectively.
\STATE $(p^{*}, q^{*}) = \argmin_{p, q \in [K], p\neq q} L_{n_pm_q}(t)$ \text{ where, } 
$(n_p, m_q) = \argmin_{n \in \hat{D}_p(t), m \in \hat{D}_q(t)} \left\| \hat{\boldsymbol{\mu}}_m(t) - \hat{\boldsymbol{\mu}}_n(t) \right\|$
\FOR{$k=1$ to $K$}
\STATE $D_{k,1}, D_{k,2} \leftarrow$ split the $k^{th}$ cluster $\hat{D}_k(t)$ into $2$ clusters through SLINK.
\STATE $(a_k, b_k) = \argmin_{a \in D_{k,1}, b \in D_{k,2}} \left\| \hat{\boldsymbol{\mu}}_a(t) - \hat{\boldsymbol{\mu}}_b(t) \right\|$
\ENDFOR
% \STATE $D_{k_1}, D_{k_2} \leftarrow$ split the $k^{th}$ cluster $D_k$ into $2$ clusters through SLINK, for all $k \in [K]$.
\STATE $k^{*} = \argmax_{k \in [K]} U_{a_kb_k}(t)$ 
% \text{ where, } 
% $(a_k, b_k) = \argmin_{a \in D_{k_1}, m \in D_{k_2}} \left\| \hat{\boldsymbol{\mu}}_a(t) - \hat{\boldsymbol{\mu}}_b(t) \right\|$
\STATE Set of potential arms to sample: $\mathcal{A} = \{n_{p^{*}}, m_{q^{*}}, a_{k^{*}}, b_{k^{*}}\}$
\STATE $\displaystyle A_{t+1} = \argmin_{m \in \mathcal{A}} N_m(t)$
\STATE \textbf{Output:} $A_{t+1}$.
\end{algorithmic}
% \vspace{-5mm}
\end{algorithm}
\vspace{-3mm}

\subsection{LUCBBOC algorithm performance}
\begin{theorem} \label{Theorem2:deltaPACLUCBBOC}
    The LUCBBOC algorithm stops in finite time almost surely and is a $\delta-$ PC, i.e.
    \begin{equation}
    \begin{aligned}
        \mathbb{P}_{\boldsymbol{\mu}}\left[ \tau_\delta < \infty \right] = 1 \text{ and } \mathbb{P}_{\boldsymbol{\mu}}\left[ \mathcal{C}\left( \boldsymbol{\hat{\mu}}\left( \tau_\delta \right) \right) \nsim \mathcal{C}\left( \boldsymbol{\mu} \right) \right] \leq \delta.
    \end{aligned}
    \end{equation}
\end{theorem}
\begin{proof}
    The stopping rule for LUCBBOC is the same as that for ATBOC. Hence, this theorem follows from Theorem~\ref{Theorem2:deltaPAC}, where ATBOC is shown to be a $\delta$-PC algorithm.
\end{proof}

% \begin{theorem} \label{Theorem4:AsymptoticOptimalLUCBBOC}
\noindent{\em Expected sample complexity upper bound:} 
In our asymptotic analysis, the sample complexity depends on the convergence point of the empirical arm pull proportions. 
But unlike in ATBOC, where the convergence point of the empirical arm pull proportions is characterized in \eqref{eq:cstarmu} and \eqref{eq:modcstarmu}, we do not have any theoretical characterization for the convergence point $\boldsymbol{w}^*_{\text{LUCBBOC}}$.
Suppose we assume that such a convergence point exists for LUCBBOC, and we let $\boldsymbol{w}^{*}_{\text{LUCBBOC}}$ denote the convergence point of the empirical arm pull proportions, i.e., $\left\|\boldsymbol{w}^*(t) - \boldsymbol{w}^*_{\text{LUCBBOC}}\right\|_{\infty}=0$. Then, for any problem instance $\boldsymbol{\mu}$, LUCBBOC satisfies
    \begin{equation} \label{eq:slopeLUCBBOC}
    \limsup_{\delta \rightarrow 0} \frac{\mathbb{E}_{\boldsymbol{\mu}}^{\text{LUCBBOC}}[\tau_\delta]}{\log\left( \frac{1}{\delta} \right)} \leq \frac{2}{\psi_{\text{mod}}\left(\boldsymbol{w}^*_{\text{LUCBBOC}}, \boldsymbol{\mu}\right)}.
\end{equation}
% \end{theorem}
% \begin{proof}
    Under this convergence assumption, we can prove this result by following similar steps as in the sample complexity analysis of ATBOC-subGauss in Theorem \ref{Theorem4:AsymptoticOptimal}.
% \end{proof}
\begin{remark} \label{rem:rr}
     To verify our asymptotic result for LUCBBOC in \eqref{eq:slopeLUCBBOC}, we numerically compute the convergence point and show that the performance of LUCBBOC {satisfies the asymptotic upper bound} (See Fig.~\ref{fig:optimalityd2} in Section~\ref{sec:Simulations}). Additionally, we know that for the round robin strategy (RR), the empirical arm pull proportion for each of the arms converges to $1/M$, i.e., $w^*_{\text{RR}} = 1/M$. We also show that the performance of RR satisfies the upper bound computed using the arm pull proportions $w^*_{\text{RR}}$ (See Fig.~\ref{fig:optimalityd2} in Section~\ref{sec:Simulations}).
\end{remark}

\section{BOC-ELIM Algorithm and its performance} \label{sec: BOC-Elim}
In this section, we propose the Bandit Online Clustering-Elimination (BOC-ELIM) algorithm, for any arbitrary number of clusters $K$, as an extension of the MaxGapElim algorithm introduced in \cite{katariya2019maxgap} for the case of $K = 2$ clusters. Its pseudocode is presented in Algorithm \ref{Algo:BOC-Elim}. This algorithm is specifically designed for the special case of 
 $d=1$, where the SLINK clustering problem reduces to top $K-1$ gap identification problem. The BOC-ELIM algorithm finds the arms corresponding to each of the top $K-1$ gaps.\footnote{MaxGapElim Algorithm discussed in \cite{katariya2019maxgap} finds the two arms corresponding to the maximum gap. However, if we are only interested in clustering the arms according to the maximum gap and not interested in the identities of the arms corresponding to the  maximum gap, then a different stopping rule can be used. Equation (20) in \cite{katariya2019maxgap} presents such an early stopping rule for the clustering task. An extension of this rule to the $K > 2$ case is possible. However, we skip the details of such an early stopping rule for the BOC-ELIM algorithm. This is required only if multiple arm pairs have the same maximum gap.} %BOC-Elim is a more computationally efficient algorithm {\textcolor{red}{than?}}. 

The intuition behind the \textsc{BOC-ELIM} algorithm is as follows. Initially, all arms are sampled. Note that the only information required about an arm is whether its right or left gaps are any of the top $K-1$ gaps. Once this information about that arm has been determined, the arm is no longer considered active; that is, the arm is eliminated and will not be sampled again.
We determine that the right (left) gap of an arm is among the top $K-1$ maximum gaps if the arm’s minimum right (left) gap is greater than the UCB of the $K^{\text{th}}$ highest gap. Similarly, we determine that an arm’s right (left) gap is not among the top $K-1$ gaps if the arm’s maximum right (left) gap is less than the LCB of the $(K-1)^{\text{th}}$ highest gap.
Once a sufficient number of arms have been eliminated to identify the top $K-1$ gaps, the algorithm terminates, and the clusters are declared. First, we introduce the notation and then provide a formal description of the BOC-ELIM algorithm.

 The means of the arms are one-dimensional, and without loss of generality, let us assume that the means are sorted, i.e., $\mu_1\geq \mu_2 \ldots \geq \mu_M$.
Let $c_{N_m(t)}$ denote the confidence parameter\footnote{For the special case of $d=1$, this is the same as the confidence radius for sub-Gaussian distributions in Section \ref{sec:LUCB}. As mentioned in Footnote 3,  by using an appropriate confidence parameter, BOC-ELIM can also be applied to other classes of probability distributions.} of arm $m$ at time $t$, defined as $
c_{N_m(t)} \coloneqq \sqrt{\frac{2}{N_m(t)} \log\left( \frac{4M N_m^2(t)}{\delta} \right)}.$
Let $l_m(t)$  and $r_m(t)$ denote the lower confidence bound (LCB) and upper confidence bound (UCB) of arm $m$ at time $t$, defined as $
l_m(t) \coloneqq \hat{\mu}_m(t) - c_{N_m(t)}$ and  $r_m(t) \coloneqq \hat{\mu}_m(t) + c_{N_m(t)}$, respectively.
The expression for the confidence parameter $c_{N_m(t)}$ is chosen such that the true mean $\mu_m$ lies within the confidence interval $[l_m(t), r_m(t)]$ with high probability. Specifically, for $\delta \in (0, 1)$,
\begin{equation} \label{eq:confinterval}
 {\mathbb{P}\left[ \bigcap_{t \in \mathbb{N}}\bigcap_{m \in [M]} \left\{ \mu_m \in [l_m(t), r_m(t)] \right\} \right] \geq 1 - \delta.}
 \end{equation}

We define the maximum right gap of arm $m$ in the same manner as defined in \cite{katariya2019maxgap} and this is given by 
\begin{equation} \label{eq:rightgap}
    U\Delta_m^r(t) \coloneqq \max_{j \in P_m^r} G\left( l_j(t), t \right),
\end{equation}
where $P_m^r \coloneqq \left\{ j: l_j(t) \in \left[ l_m(t), r_m(t) \right] \right\}$ and 
\begin{equation*}
    G(x, t) \coloneqq \begin{cases}
        \min_{j: l_j(t)>x} r_j(t)-x &\text{ if } \left\{ j: l_j(t)>x \right\} \neq \emptyset\\
        \max_{j\neq m} r_j(t)-x &\text{ otherwise. }
    \end{cases}
\end{equation*}
The expression in \eqref{eq:rightgap} computes the maximum gap formed by arm $m$ on the right side satisfying the confidence bounds in \eqref{eq:confinterval}.
A similar definition is used for the maximum left gap of arm $m$, denoted by $U\Delta_m^l(t)$. A more detailed description of the maximum right and left gap can be found in \cite{katariya2019maxgap}.
Now, we define the minimum right gap formed by the arm $m$ as follows. If the confidence interval of the arm $m$ is disjoint from the confidence interval of all other arms, and if there exists an arm $j$ whose confidence interval is to the right side of the confidence interval of arm $m$, then we define the minimum right gap of the arm $m$ as $\min_{j:l_j(t)>r_m(t)} l_j(t)-r_m(t)$. Otherwise, the minimum right gap is $0$. Mathematically, we have
\begin{equation}
    L\Delta_m^r(t) \coloneqq \begin{cases}
        &\min_{j:l_j(t)>r_m(t)} l_j(t)-r_m(t)\\ 
        &\hspace{0.5cm}\text{ if } \left[ l_m(t), r_m(t)\right]\cap \left\{\cup_{j \neq m}\left[ l_j(t), r_j(t) \right]\right\}=\emptyset \text{ and } \left\{ j: l_j(t) >r_m(t) \right\} \neq \emptyset \\
        &0 \hspace{0.3cm}\text{ otherwise. }
    \end{cases}
\end{equation}
A similar definition is used for the minimum left gap of the arm $m$, denoted by $L\Delta_m^l(t)$. 
We use the notation $(k)_t$ to denote the arm with the $k^{th}$ highest empirical mean at time $t$.
We define the LCB of the $k^{th}$ gap as follows.
\begin{equation}
    L^k(t) \coloneqq \min_{j \in \left\{ (1)_t, \ldots, (k)_t \right\}} l_j(t) - \max_{j \in \left\{ (k+1)_t, \ldots, (M)_t \right\}} r_j(t) \quad \forall k \in [M-1].
\end{equation}
We denote the $k^{th}$ highest gap's LCB by $L^{(k)}(t)$, i.e., $L^{(M-1)}(t)\leq\ldots\leq L^{(1)}(t)$. Define the $k^{th}$ gap's UCB as 
\begin{equation}
    U^k(t) \coloneqq \min_{j \in \left\{ (1)_t, \ldots, (k)_t \right\}} r_j(t) - \max_{j \in \left\{ (k+1)_t, \ldots, (M)_t \right\}} l_j(t).
\end{equation}
We denote the $k^{th}$ highest gap's UCB by $U^{(k)}(t)$, i.e., $U^{(M-1)}(t)\leq\ldots\leq U^{(1)}(t)$.

At time $t$, each arm can be in one of three states with respect to its right-side gap: 
(i) it has been determined that the right-side gap is not among the top $K-1$ gaps, 
(ii) it has been determined that the right-side gap is among the top $K-1$ gaps, or 
(iii) this determination has not yet been made.
Accordingly, let $F_r$ denote the set of arms whose right-side gap has been ruled out of being among the top $K-1$ gaps, and let $S_r$ denote the set of arms whose right-side gap has been identified as one of the top $K-1$ gaps. Let $A_r$ denote the set of arms whose right-side gap has not yet been classified, i.e., arms that are neither in $F_r$ nor in $S_r$; we refer to these as \emph{right-sided active arms}.
Analogous definitions apply to the left-side gap, yielding the sets $F_l$, $S_l$, and $A_l$. Finally, we define the set of active arms at time $t$ as $A \coloneqq A_r \cup A_l$,
that is, an arm is considered active if it is active on at least one side.

\begin{algorithm}
% \caption{Lower and Upper Confidence Bound-based Bandit Online Clustering (\textit{LUCBBOC}) Algorithm}\label{Algo:LUCBOC}
\caption{\textit{BOC-ELIM}} \label{Algo:BOC-Elim}
\begin{algorithmic}[1]
% \STATE \textbf{Input:} $\boldsymbol{\hat{\mu}}(t), \boldsymbol{N}(t).$
\STATE \textbf{Input:} $\delta, M, K$
\STATE \textbf{Initialize:} $t=0$, $N_m(0)=0 \text{ for all } m \in [M]$, $A = A_r = A_l = \{ 1, 2, \ldots, M \}$, $t=0$, $F_r=F_l=S_r=S_l=\emptyset$.
\REPEAT
\STATE $t \leftarrow t+1$.
\STATE $\forall m \in A$, sample arm $m$ and update $l_m(t)$ and $r_m(t)$.
\STATE Compute $U\Delta_m^r(t), U\Delta_m^l(t), L^{(K-1)}(t)$, $ L\Delta_m^r(t), L\Delta_m^l(t), U^{(K)}(t)$.
\STATE $\forall m \in A$ 
\begin{itemize}
    \item if $U\Delta_m^r(t) < L^{(K-1)}(t)$, 
    then $A_r=A_r\setminus\{m\}$, $F_r=F_r\cup\{m\}$
    \item if $U\Delta_m^l(t) < L^{(K-1)}(t)$, 
    then $A_l=A_l\setminus\{m\}$, $F_l=F_l\cup\{m\}$
    \item if $L\Delta_m^r(t) > U^{(K)}(t)$, 
    then $A_r=A_r\setminus\{m\}$, $S_r=S_r\cup\{m\}$
    \item if $L\Delta_m^l(t) > U^{(K)}(t)$, 
    then $A_l=A_l\setminus\{m\}$, $S_l=S_l\cup\{m\}$
\end{itemize}
$A = A_r \cup A_l$.
\UNTIL $|F_r\cup F_l|=2(M-K)$ or $|S_r\cup S_l|=2(K-1)$.
% \STATE Compute $E_{ij}(t)$, $U_{ij}(t)$ and $L_{ij}(t)$ using \eqref{eq:Eit}, \eqref{eq:Uit} and \eqref{eq:Lit}, respectively.
% \STATE $(p^{*}, q^{*}) = \argmin_{p, q \in [K], p\neq q} L_{n_pm_q}(t)$ \text{ where, } \\
% $(n_p, m_q) = \argmin_{n \in D_p, m \in D_q} \left\| \hat{\boldsymbol{\mu}}_m(t) - \hat{\boldsymbol{\mu}}_n(t) \right\|$
% \STATE $D_{k_1}, D_{k_2} \leftarrow$ split the $k^{th}$ cluster $D_k$ into $2$ clusters through SLINK, for all $k \in [K]$.
% \STATE $k^{*} = \argmax_{k \in [K]} U_{a_kb_k}(t)$ \text{ where, } \\
% $(a_k, b_k) = \argmin_{a \in D_{k_1}, m \in D_{k_2}} \left\| \hat{\boldsymbol{\mu}}_a(t) - \hat{\boldsymbol{\mu}}_b(t) \right\|$
% \STATE Set of potential arms to sample: \\$\mathcal{A} = \{n_{p^{*}}, m_{q^{*}}, a_{k^{*}}, b_{k^{*}}\}$
% \STATE $\displaystyle A_{t+1} = \argmin_{m \in \mathcal{A}} N_m(t)$
% \STATE \textbf{Output:} $A_{t+1}$.
\end{algorithmic}
\end{algorithm}

At the start of the algorithm, all arms are both right-sided active $A_r=[M]$ and left-sided active $A_l=[M]$, and no arms are grouped under $F_r, F_l, S_r, S_l$. In time step $t$, we sample all active arms, update $l_m(t)$, $r_m(t)$ and compute all gaps (Line 6 of Algorithm \ref{Algo:BOC-Elim}). If the maximum right gap of any arm $m$ is less than the $(K-1)^{th}$ highest gap's LCB ($U\Delta_m^r(t) < L^{(K-1)}(t)$), then the right side of arm $m$ cannot form one of the top $K-1$ gaps ($F_r=F_r\cup\{m\}$) and is no longer a right-sided active arm ($A_r=A_r\setminus\{m\}$). Similarly, if the minimum right gap of any arm $m$ is greater than the $K^{th}$ highest gap's UCB ($L\Delta_m^r(t) > U^{(K)}(t)$), then the right side of arm $m$ forms one of the top $K-1$ gaps ($S_r=S_r\cup\{m\}$) and no more a right sided active arm ($A_r=A_r\setminus\{m\}$). A similar procedure will be followed for the left side of the arms (Line 7 of Algorithm \ref{Algo:BOC-Elim}). This process of arm elimination from the active set will be repeated until it identifies either the arms corresponding to the top $K-1$ gaps or the arms corresponding to the $M-K$ gaps that are not among the top gaps (Line 8 of Algorithm \ref{Algo:BOC-Elim}).
% \textcolor{red}{Note that, unlike in the previous sections, at each time $t$, all active arms are sampled. Hence, at time $t$, the total number of arm pulls is given by $\sum_{m\in[M]}N_m(t)$ and not $t$.}

% \section{BOC-Elim Performance} \label{sec: BOC-Elim Perf}
% This section include
Now we discuss the accuracy and sample complexity results for the proposed BOC-ELIM algorithm. Theorem \ref{BOC-Elim delta PC} presents the accuracy guarantee of the algorithm. 
\begin{theorem} \label{BOC-Elim delta PC}
    With probability $1-\delta$, BOC-ELIM outputs the correct clustering of arms.
\end{theorem}
\begin{proof}
    See Appendix \ref{appsubsec:elimdelpc}.% \ref{appsec: bocelim}.
\end{proof}
We define the $k^{th}$ gap as $\Delta_k \coloneqq {\mu}_{k} - {\mu}_{k+1}$, where $k \in [M-1]$. 
%Note that, for $M$ arms we have $M-1$ gaps $\left\{\Delta_i\right\}_{i=1}^{M-1}$. 
We denote the $k^{th}$ highest gap by $\Delta_{(k)}$. 
Let the pair of arms $m_k$ and $m_k+1$ correspond to the $k^{th}$ highest gap.
We define the gap between arms $i$ and $j$ as $\Delta_{i, j} \coloneqq {\mu}_{j} - {\mu}_{i}$.
We denote the hardness parameter of the arm $m$ by $\rho_m$. The number of times the arm $m$ is pulled by the algorithm depends on $\rho_m$ and is defined as $\rho_m \coloneqq \min\left\{ \rho_m^r, \rho_m^l \right\}$, where $\rho_m^r$ and $\rho_m^l$ are defined as follows. For any arm $m$, whose right side gap does not form one of the top $K-1$ gaps, i.e. $\forall m \notin \left\{ m_1+1, \ldots, m_{K-1}+1 \right\}$, $\rho_m^r$ is defined as
\begin{equation} \label{eq: rho r not top}
\begin{split}
    \rho_m^r \coloneqq \max \left\{ \max_{j: j<m}\left[ 
    \min\left\{ \frac{\Delta_{m, j}}{4}, \frac{\Delta_{(K-1)}-\Delta_{m, j}}{8} \right\} \right],\frac{\Delta_{(K-1)}-\Delta_{m, 1}}{8} \right\}.
\end{split}
\end{equation}
The first term in the outer $\max$ is assumed to be $\infty$ if $m=1$, i.e., there does not exist a $j$ which satisfies the constraint $j<m$.
For any arm $m$, whose right side gap does form any one of the top $K-1$ gaps, i.e., $\forall m \in \left\{ m_1+1, \ldots, m_{K-1}+1 \right\}$, $\rho_m^r$ is defined as 
\begin{equation} \label{eq: rho r top}
    \rho_m^r \coloneqq \min\left\{ \frac{{}\Delta_{m+1, m}}{4}, \frac{\Delta_{m, m-1} - \Delta_{(K)}}{8} \right\}.
\end{equation}
The first term in the above $\min$ is assumed to be $\infty$ if $m=K$.
We use a similar definition for $\rho_m^l$. Now we present the sample complexity of the proposed BOC-ELIM algorithm in Theorem \ref{boc-elim sc}.
\begin{theorem} \label{boc-elim sc}
    With probability $1-\delta$, the sample complexity of the proposed BOC-ELIM algorithm is bounded by 
    \begin{equation}
        \sum_{m \in [M]} \frac{23\log\left(\frac{M}{\delta \rho_m}\right)}{\rho_m^2} .
    \end{equation}
\end{theorem}
\begin{proof}
    See Appendix \ref{appsubsec:elimsc}.%\ref{appsec: bocelim}.
\end{proof}

\section{Implementation details and Computational Complexity Analysis} \label{sec: compcomplx}

In this section, we provide some implementation details and the computational complexity of the algorithms ATBOC, LUCBBOC, and BOC-ELIM proposed in Sections \ref{sec:AvgTrackingOC}, \ref{sec:LUCB}, and \ref{sec: BOC-Elim}, respectively. {We observe that solving the inner infimum is crucial for implementing the ATBOC and LUCBBOC algorithms.} In this section, first, we discuss the implementation details of the inner infimum functions $\psi(\cdot,\cdot)$ and $\psi_{\text{mod}}(\cdot, \cdot)$ used to compute the stopping statistic $Z(t)$ in ATBOC and LUCBBOC (Line 11 of Algorithm \ref{Algo:AvgTrackingOC}), and the optimal arm-pull proportion $\boldsymbol{w}^*(t)$ in ATBOC (Line 12 of Algorithm \ref{Algo:AvgTrackingOC}). Then we derive the order of computational complexity in terms of the problem parameters: the number of arms $M$, the number of clusters $K$, and the dimension $d$. 
{As discussed in the previous sections, the gap-based algorithms LUCBBOC and BOC-ELIM are proposed as more computationally efficient alternatives.}
Finally, through simulations, we demonstrate the computational efficiency of these gap-based algorithms relative to ATBOC by comparing their per-sample runtime.

\subsection{Inner infimum simplification}

We represent the inner infimum $\psi(\cdot, \cdot)$ as the minimum among the solutions of a finite number of quadratic constrained optimization problems. We use the structure of $\text{Alt}(\boldsymbol{\mu})$ to show this. This result is presented in Lemma~ \ref{Lemma:finiteconvex}. Recall that \(D_k \coloneqq \{ m \mid c_m = k \}\) represents the set of all arms in the \(k^{\text{th}}\) cluster. Let $m $ and $n$ be two arms from different clusters. For any two non-empty partitions $P_1$ and $P_2$ of $D_k$, we define the set  $A_{nmP_1P_2}$ as the collection of problem instances $\boldsymbol{\lambda}$ such that $\|\boldsymbol{\lambda}_a - \boldsymbol{\lambda}_b\|^2$ is greater than or equal to $\|\boldsymbol{\lambda}_m - \boldsymbol{\lambda}_n\|^2$ for all $a \in P_1, b \in P_2$. We define $q_{\boldsymbol{w}}(\boldsymbol{\lambda}) \coloneqq \sum_{m=1}^M w_m d_{\text{KL}}\left(\boldsymbol{\mu}_m , \boldsymbol{\lambda}_m\right)$.

\begin{lemma} \label{Lemma:finiteconvex}
    For any $\boldsymbol{\mu} \in \mathbb{R}^{d\times M}$ and $\boldsymbol{w} \in \mathcal{P}_M$, 
    \begin{equation*} \label{eq:innerminfinconvex}
    \psi(\boldsymbol{w}, \boldsymbol{\mu}) = \min_{k \in [K]}  \min_{\substack{P_1 \in 2^{D_k}\setminus\{ \emptyset, D_k \} \\ P_2 = D_k \setminus P_1}}  \min_{\substack{m \in D_p, n \in D_q \\ p, q\in [K], p \neq q}}  \min_{\substack{\boldsymbol{\lambda} \in \\A_{nmP_1P_2}}} {q_{\boldsymbol{w}}(\boldsymbol{\lambda})}, \quad \text{where}
    \end{equation*}
    % where
    \begin{equation*}
    \begin{split}
        A_{nmP_1P_2} =  \left\{ \boldsymbol{\lambda} \mid \|\boldsymbol{\lambda}_a - \boldsymbol{\lambda}_b\|^2 \geq \|\boldsymbol{\lambda}_m - \boldsymbol{\lambda}_n\|^2, \forall a \in P_1, b\in P_2 \right\}.
        \end{split}
    \end{equation*}
\end{lemma}

% In this section, first, we represent the inner infimum problem as the finite minimization of a quadratic constrained optimization problem (Lemma \ref{Lemma:finiteconvex}). Then, for the multivariate Gaussian distribution, for the special case where each cluster has at most two arms, we obtain a closed-form expression as the solution to the quadratic constrained optimization problem. We use a combination of this result and the Alternating Direction Method of Multipliers technique to solve the inner infimum problem for the multivariate Gaussian distribution, which will be used in ATBOC-subGauss. In addition, we simplify the inner infimum problem for the single-parameter exponential family of distributions to a finite minimization of closed-form expressions.

% In Lemma \ref{Lemma:finiteconvex}, we represent the inner infimum problem $\psi\left(\boldsymbol{\mu}, \boldsymbol{w}\right)$ as the finite minimization of Quadratic Contrained Quadratic program (QCQP). 

\begin{proof}
See Appendix \ref{appsec: computation}.
\end{proof}

Note that, in terms of implementation, the modified inner infimum function $\psi_{\text{mod}}(\cdot, \cdot)$ is the special case of the inner infimum function $\psi(\cdot, \cdot)$, where the KL divergence in $\psi(\cdot, \cdot)$ is considered for the Gaussian distribution. Hence, Lemma \ref{Lemma:finiteconvex} also holds for the modified inner infimum problem $\psi_{\text{mod}}(\cdot, \cdot)$, i.e.
\begin{equation} \label{eq:modinnerminfinconvex}
    \psi_{\text{mod}}(\boldsymbol{w}, \boldsymbol{\mu}) = \min_{k \in [K]}  \min_{\substack{P_1 \in 2^{D_k}\setminus\{ \emptyset, D_k \} \\ P_2 = D_k \setminus P_1}}  \min_{\substack{m \in D_p, n \in D_q \\ p, q\in [K], p \neq q}}  \min_{\substack{\boldsymbol{\lambda} \in \\A_{nmP_1P_2}}} \frac{1}{2\sigma^2}\sum_{m=1}^Mw_m \left\|\boldsymbol{\mu}_m - \boldsymbol{\lambda}_m\right\|^2.
    \end{equation}
{\em ATBOC-Gauss/ATBOC-subGauss and LUCBBOC:} Now we discuss the implementation of the modified inner infimum problem $\psi_{\text{mod}}(\cdot, \cdot)$, which can be used in ATBOC-Gauss/ATBOC-subGauss and LUCBBOC.
From \eqref{eq:modinnerminfinconvex}, to solve the modified inner infimum problem $\psi_{\text{mod}}\left(\boldsymbol{w}, \boldsymbol{\mu}\right)$, for each valid combination of $n, m, P_1, P_2$, we need to solve a minimization problem of the form $\min_{\substack{\boldsymbol{\lambda} \in \bigcap_{\substack{i \in P_1 j \in P_2}}\\ \left\{  d_{i, j}(\boldsymbol{\lambda})\right.  \left. \geq d_{m, n}(\boldsymbol{\lambda}) \right\}}} \sum_{m=1}^Mw_m \left\|\boldsymbol{\mu}_m - \boldsymbol{\lambda}_m\right\|^2$, which is a Quadratic Constraint Quadratic Program (QCQP). We solve the QCQP problem using an  Alternating Direction Method of Multipliers (ADMM) algorithm. 
% First, we observe that
For the special case where each cluster has at most two arms, i.e., $|P_1|=|P_2|=1$, the above QCQP problem reduces to the following single-constrained QCQP problem (1QCQP).  
% This minimization problem for the multivariate Gaussian distribution and for the special case where each cluster has at most two arms, i.e., $|P_1|=|P_2|=1$, simplifies as follows.
\begin{equation} \label{eq: QCQP1cons}
   \boldsymbol{\lambda}^{*} = \argmin_{\substack{\boldsymbol{\lambda} \in  \left\{  d_{i, j}(\boldsymbol{\lambda})\right.  \left. \geq d_{m, n}(\boldsymbol{\lambda}) \right\}}} \sum_{m=1}^Mw_m \left\|\boldsymbol{\mu}_m - \boldsymbol{\lambda}_m\right\|^2
\end{equation}

It can be verified that the 1QCQP problem in \eqref{eq: QCQP1cons}, on simplification using the Lagrange multiplier method, yields the following closed-form solution.\\
\textbf{Case 1: } $\{i, j\}\cap\{m, n\} = \emptyset$\\
\begin{equation}
    \boldsymbol{\lambda}_i = \frac{w_i\boldsymbol{\mu}_i + \frac{\gamma w_j}{\gamma-w_j}\boldsymbol{\mu}_j}{w_i + \frac{\gamma w_j}{\gamma-w_j}}, 
    \boldsymbol{\lambda}_j = \frac{w_j\boldsymbol{\mu}_j + \frac{\gamma w_i}{\gamma-w_i}\boldsymbol{\mu}_i}{w_j + \frac{\gamma w_i}{\gamma-w_i}}, 
    \boldsymbol{\lambda}_n = \frac{w_n\boldsymbol{\mu}_n + \frac{\gamma w_m}{\gamma+w_m}\boldsymbol{\mu}_m}{w_n + \frac{\gamma w_m}{\gamma+w_m}},
    \boldsymbol{\lambda}_m = \frac{w_m\boldsymbol{\mu}_m + \frac{\gamma w_n}{\gamma+w_n}\boldsymbol{\mu}_n}{w_m + \frac{\gamma w_n}{\gamma+w_n}}
\end{equation}
where, $\gamma = \frac{w_iw_jw_mw_n\left[\|\boldsymbol{\mu}_m-\boldsymbol{\mu}_n\|\pm\|\boldsymbol{\mu}_i-\boldsymbol{\mu}_j\|\right]}{w_mw_n(w_i+w_j)\|\boldsymbol{\mu}_m-\boldsymbol{\mu}_n\|\mp w_iw_j(w_m+w_n)\|\boldsymbol{\mu}_i-\boldsymbol{\mu}_j\|}$.\\
\textbf{Case 2: } $j=m$
\begin{equation}
    \boldsymbol{\lambda}_i = \frac{\gamma\boldsymbol{\lambda}j - w_i\boldsymbol{\mu}_i}{\gamma-w_i}, \boldsymbol{\lambda}_n=\frac{\gamma\boldsymbol{\lambda}_j+w_n\boldsymbol{\mu}_n}{\gamma+w_n}, \boldsymbol{\lambda}_j = \frac{w_j\boldsymbol{\mu}_j + \frac{\gamma w_n}{\gamma+w_n}\boldsymbol{\mu}_n+\frac{\gamma w_i}{\gamma - w_i}\boldsymbol{\mu}_i}{w_j + \frac{\gamma w_n}{\gamma+w_n}+\frac{\gamma w_i}{\gamma - w_i}}
\end{equation}
where 
\begin{equation}
\begin{aligned}
    &\gamma^2\left[w_i^2(w_j+2w_n)\|\boldsymbol{\mu}_j-\boldsymbol{\mu}_i\|^2-w_n^2(w_j+2w_i)\|\boldsymbol{\mu}_j-\boldsymbol{\mu}_n\|^2+2w_iw_n(w_n-w_i)(\boldsymbol{\mu}_j-\boldsymbol{\mu}_i)^T(\boldsymbol{\mu}_j-\boldsymbol{\mu}_n)\right] + \\
    &\gamma\left[2w_iw_n\left( w_i(w_j+w_n)\|\boldsymbol{\mu}_j-\boldsymbol{\mu}_i\|^2+w_n(w_j+w_i)\|\boldsymbol{\mu}_j-\boldsymbol{\mu}_n\| - 2w_iw_n(\boldsymbol{\mu}_j-\boldsymbol{\mu}_i)^T(\boldsymbol{\mu}_j-\boldsymbol{\mu}_n) \right)\right] +\\ 
    &\left[w_i^2w_n^2w_j\left(\|\boldsymbol{\mu}_j-\boldsymbol{\mu}_i\|^2-\|\boldsymbol{\mu}_j-\boldsymbol{\mu}_n\|^2\right)\right] = 0
    \end{aligned}
\end{equation}

\begin{algorithm}
\caption{\textit{ADMM}} \label{Algo:ADMM}
\begin{algorithmic}[1]
\STATE \textbf{Hyper parameters:} $\epsilon>0$, $\rho>0$.
% \STATE \textbf{Input:} $\delta, M, K$
\STATE \textbf{Initialize:} $i=0$, $\underline{\boldsymbol{\lambda}}^{(0)} = \underline{\boldsymbol{u}}_l^{(0)} = \underline{\boldsymbol{0}}$.
\REPEAT
\STATE $\underline{\boldsymbol{\lambda}}^{(i+1)}  = (A_0+m\rho I)^{-1}\left(b_0+\rho\sum_{l=1}^L\left(\underline{\boldsymbol{\lambda}}^{(i)}+\underline{\boldsymbol{u}}_l^{(i)}\right)\right)$.
\STATE $\underline{\boldsymbol{\eta}}_l^{(i+1)} = \argmin_{\underline{\boldsymbol{\eta}}}\left\| \underline{\boldsymbol{\eta}}-\underline{\boldsymbol{\lambda}}^{(i+1)} + \underline{\boldsymbol{u}}_l^{(i)} \right\|$ subject to $\underline{\boldsymbol{\eta}}^TA_l\underline{\boldsymbol{\eta}} - 2\boldsymbol{b}_l^T\underline{\boldsymbol{\eta}}\leq 0$ for all $l \in [L]$.
\STATE $\underline{\boldsymbol{u}}_l^{(i+1)} = \underline{\boldsymbol{u}}_l^{(i)} + \underline{\boldsymbol{\eta}}_l^{(i+1)} - \underline{\boldsymbol{\lambda}}^{(i+1)}$ for all $l \in [L]$.
\STATE $i \leftarrow i+1$.
\UNTIL $\left\|\underline{\boldsymbol{\lambda}}^{(i)} - \underline{\boldsymbol{\lambda}}^{(i-1)}\right\|<\epsilon$.

\end{algorithmic}
\end{algorithm}

It can be verified that the 1QCQP problem in \eqref{eq: QCQP1cons} can be written as 
\begin{equation}
    \min_{\underline{\boldsymbol{\lambda}}, \underline{\boldsymbol{\eta}}} \underline{\boldsymbol{\lambda}}^TA_0\underline{\boldsymbol{\lambda}} - 2\boldsymbol{b}_0^T\underline{\boldsymbol{\lambda}} \text{ subject to } \underline{\boldsymbol{\eta}}^TA\underline{\boldsymbol{\eta}}-2\boldsymbol{b}^T\underline{\boldsymbol{\eta}}\leq 0, \underline{\boldsymbol{\eta}}=\underline{\boldsymbol{\lambda}},
\end{equation}
for matrices $A_0, A_1 \in \mathbb{R}^{Md\times Md}$, vectors $\boldsymbol{b}_0, \boldsymbol{b}_1\in \mathbb{R}^{Md}$ and $\underline{\boldsymbol{\lambda}} = [\boldsymbol{\lambda}_1^T \ldots \boldsymbol{\lambda}_M^T]^T, \underline{\boldsymbol{\eta}} \in \mathbb{R}^{Md}$. 
Similarly, the general QCQP for the general case of each cluster can have any number of arms, i.e., $|P_1|\geq 1, |P_2|\geq 1$, takes the form:
\begin{equation} \label{eq: QCQP}
    \min_{\underline{\boldsymbol{\lambda}}, \{\underline{\boldsymbol{\eta}}_l\}_{l=1}^L} \underline{\boldsymbol{\lambda}}^TA_0\underline{\boldsymbol{\lambda}} - 2\boldsymbol{b}_0^T\underline{\boldsymbol{\lambda}} \text{ subject to } \underline{\boldsymbol{\eta}}_l^TA\underline{\boldsymbol{\eta}}_l-2\boldsymbol{b}^T\underline{\boldsymbol{\eta}}_l\leq 0, \underline{\boldsymbol{\eta}}_l=\underline{\boldsymbol{\lambda}}, \forall l \in [L],
\end{equation}
where $L=|P_1||P_2|$ is the number of constraints.
The above form is commonly referred to as the consensus form in the optimization literature and can be solved using ADMM \cite{huang2016consensus}. We present the pseudo code for the ADMM algorithm in Algorithm \ref{Algo:ADMM}. 
The ADMM algorithm for $L$ constrained QCQP problem involves solving $L$ 1QCQP problems (Line 5 in Algorithm \ref{Algo:ADMM}) iteratively until the solution converges. More details of the ADMM algorithm can be found in \cite{huang2016consensus}.
We use this ADMM-based implementation in our proposed ATBOC-subGauss algorithm.

% \subsection{Inner infimum simplification for single parameter exponential family of distributions}
{\em ATBOC-1pExp:} For the case of ATBOC-1pExp where we have single parameter exponential family of distributions, we  simplify the inner infimum problem $\psi(\cdot, \cdot)$ further to a minimum of a finite number of closed-form expressions. The simplification involves solving the quadratic constrained minimization problem $\min_{\substack{\boldsymbol{\lambda} \in \\A_{nmP_1P_2}}} q_{\boldsymbol{w}}(\boldsymbol{\lambda})$ in Lemma \ref{Lemma:finiteconvex} using the Lagrange method of multipliers, to get a finite minimization of closed form expressions. Then, we combine this finite minimization of closed form expressions with the other minimization terms $\min_{k \in [K]}  \min_{\substack{P_1 \in 2^{D_k}\setminus\{ \emptyset, D_k \} \\ P_2 = D_k \setminus P_1}}  \min_{\substack{m \in D_p, n \in D_q \\ p, q\in [K], p \neq q}}$ in Lemma \ref{Lemma:finiteconvex}. We skip the details of the simplifications. We present the final expression as follows.

\begin{equation} \label{eq: innermin1pexp}
    \begin{aligned}
        \psi(\boldsymbol{w}, \boldsymbol{\mu}) = \min_{k \in [K]}  \min_{\substack{R^{'} \in 2^{D_k} \setminus \left\{ \emptyset, D_k \right\} \\ L^{'} \in 2^{D_k \setminus R^{'}}\setminus \left\{\emptyset\right\}}}  \min_{\substack{m \in D_p, n\in D_q \\ p = q+1 \\ n \notin L^{'}, m \notin R^{'}}} f_w\left( \boldsymbol{\lambda}_{nmL^{'}R^{'}} \right) 
        \left[ \mathds{1}_{\left\{ a \in D_k : \lambda_a \leq \lambda_{nmL^{'}R^{'}}^{(L^{'})}, \lambda_a \geq \lambda_{nmL^{'}R^{'}}^{(R^{'})} \right\}} \right]^{-1}
    \end{aligned}
    \end{equation}
    where the expressions of $\lambda_{nmL^{'}R^{'}}$ for different cases are as follows.\\
    \textbf{Case 1: $n \notin R^{'} \text{ and } m \notin L^{'}$}
    \begin{equation}
        \left(\lambda_{nmL^{'}R^{'}}\right)_a = \begin{cases}
            \dot{A}^{-1}\left(\mu_{L^{'}} - \frac{x_a^{(1)}}{w_{L^{'}}}\right) &\text{ if } a \in L^{'}\\
            \dot{A}^{-1}\left(\mu_{R^{'}} + \frac{x_a^{(1)}}{w_{R^{'}}}\right) &\text{ if } a \in R^{'}\\
            \dot{A}^{-1}\left(\dot{A}(\mu_{n}) + \frac{x_a^{(1)}}{w_{n}}\right) &\text{ if } a =n\\
            \dot{A}^{-1}\left(\dot{A}(\mu_{m}) - \frac{x_a^{(1)}}{w_{m}}\right) &\text{ if } a =m\\
            \mu_a &\text{ otherwise}
        \end{cases},
        \text{ with, }
        x_a^{(1)} = \frac{\left[ \left(\dot{A}(\mu_m) - \dot{A}(\mu_n)\right) - (\mu_{R^{'}} - \mu_{L^{'}}) \right]}{\frac{1}{w_m} + \frac{1}{w_n} + \frac{1}{w_{R^{'}}} + \frac{1}{w_{L^{'}}}}
    \end{equation}
    \textbf{Case 2: $m \in L^{'}$}
    \begin{equation}
        \left(\lambda_{nmL^{'}R^{'}}\right)_a = \begin{cases}
            \dot{A}^{-1}\left(\mu_{L^{'}} - \frac{2x_a^{(2)}}{w_{L^{'}}}\right) &\text{ if } a \in L^{'}\\
            \dot{A}^{-1}\left(\mu_{R^{'}} + \frac{x_a^{(2)}}{w_{R^{'}}}\right) &\text{ if } a \in R^{'}\\
            \dot{A}^{-1}\left(\dot{A}(\mu_{n}) + \frac{x_a^{(2)}}{w_{n}}\right) &\text{ if } a =n\\
            \mu_a &\text{ otherwise}
        \end{cases},
        \text{ with, }
        x_a^{(2)} = \frac{\left[ 2\mu_{L^{'}} - \mu_{R^{'}} - \dot{A}(\mu_n) \right]}{\frac{4}{w_{L^{'}}} + \frac{1}{w_{R^{'}}} + \frac{1}{w_{n}}}
    \end{equation}
    \textbf{Case 3: $n \in R^{'}$}
    \begin{equation}
        \left(\lambda_{nmL^{'}R^{'}}\right)_a = \begin{cases}
            \dot{A}^{-1}\left(\mu_{L^{'}} - \frac{x_a^{(3)}}{w_{L^{'}}}\right) &\text{ if } a \in L^{'}\\
            \dot{A}^{-1}\left(\mu_{R^{'}} + \frac{2x_a^{(3)}}{w_{R^{'}}}\right) &\text{ if } a \in R^{'}\\
            \dot{A}^{-1}\left(\dot{A}(\mu_{m}) - \frac{x_a^{(3)}}{w_{m}}\right) &\text{ if } a =m\\
            \mu_a &\text{ otherwise}
        \end{cases},
        \text{ with, }
        x_a^{(3)} = \frac{\left[ \mu_{L^{'}} + \dot{A}(\mu_m) - 2\mu_{R^{'}}  \right]}{\frac{4}{w_{R^{'}}} + \frac{1}{w_{L^{'}}} + \frac{1}{w_{m}}}
    \end{equation}
    where
    \begin{equation}
        \begin{aligned}
            w_S = \sum_{a \in S} w_a \text{ and } 
            \mu_S = \frac{\sum_{a \in S}w_a \dot{A}(\mu_a)}{w_S}, \text{ for } S \in \{L^{'}, R^{'}\}
        \end{aligned}
    \end{equation}

\noindent In the above expression, $D_k$ is the $k^{\text{th}}$ cluster from left to right. i.e., $\mu_{i_1}< \ldots <\mu_{i_K}$, $\forall i_k\in D_k, \forall k \in [K]$. Let $({\lambda}_{nmL^{'}R^{'}})_i$ be the $i^{\text{th}}$ coordinate of $\boldsymbol{\lambda}_{nmL^{'}R^{'}}$ and for $S\in \{R^{'}, L^{'}\}$, we use ${\lambda}_{nmL^{'}R^{'}}^{(S)} = ({\lambda}_{nmL^{'}R^{'}})_i$, for $i\in S$.

\subsection{Computational complexity analysis}

{\em Stopping rule complexity for ATBOC-Gauss/ATBOC-subGauss and LUCBBOC:} We now derive the computational complexity order for solving the modified inner infimum problem based on ADMM technique discussed in the previous section. 
From \eqref{eq:modinnerminfinconvex}, the number of QCQP problems of the form \eqref{eq: QCQP} to be solved is \newline $\sum_{k \in [K]} \sum_{k_1=1}^{K-1} \sum_{k_2=k_1+1}^K\left(2^{|D_k|}-2\right)\left|D_{k_1}\right|\left|D_{k_2}\right|$. This number can be upper-bounded by $K^3s^22^s$, where $s = \max_{k\in [K]}|D_k|$ is the maximum cluster size. 
We solve each QCQP problem of the form in \eqref{eq: QCQP} using the ADMM technique, which involves solving $L=|P_1||P_2|\leq s^2$ number of 1QCQP problems in each iteration. Let $T_A$ be the number of iterations taken by the ADMM before it converges to the solution. From the solution for the 1QCQP in \eqref{eq: QCQP1cons}, we observe that the computaion complexity order for this step is $d$. Hence, the overall computational complexity order to solve the modified inner infimum problem is $T_AK^3s^42^sd$.
The stopping rule of the ATBOC-Gauss/ATBOC-subGauss and LUCBBOC requires the computation of the modified inner infimum problem once, and hence its computational complexity order is $T_AK^3s^42^sd$.

{\em Stopping rule complexity for ATBOC-1pExp:} We now derive the computation complexity order for the inner infimum problem, which is simplified as the finite minimization of the closed form expression \eqref{eq: innermin1pexp}. It can be verified that from \eqref{eq: innermin1pexp}, the number of terms in the minimization is $2^{2s}K^2s^2$ and the computational complexity order for computing one term is $s^2$. Hence, the overall computational complexity to solve the inner infimum problem for the single-parameter exponential family of distributions is $2^{2s}K^2s^4$. The stopping rule of ATBOC-1pExp requires the computation of the inner infimum problem once, and hence its computational complexity order is $2^{2s}K^2s^4$.

{\em Sampling rule complexity for ATBOC-Gauss/ATBOC-subGauss/ATBOC-1pExp:} The sampling rule in the ATBOC algorithm involves solving an outer supremum problem, which optimizes the (modified) inner infimum problem over the arm pull proportions $w\in \mathcal{P}_M$. To solve this optimization problem, we use Sequential Least Squares Programming (SLSQP). It involves computing the gradient of the function with respect to the optimization variable $w \in \mathcal{P}_M$ at each iteration of the SLSQP. So, it requires computing the inner infimum problem $M$ times in each iteration. Let $T_S$ be the number of iterations required by SLSQP to solve the outer supremum problem. Hence, the computational complexity order of the sampling rule of ATBOC-Gauss/ATBOC-subGauss is $T_SMT_AK^3s^42^sd$ and ATBOC-1pExp is $T_SM2^{2s}K^2s^4$.

{\em Sampling rule complexity for LUCBBOC:} The sampling rule of LUCBBOC algorithm involves solving $K$ SLINK clustering and hence the computation complexity order is given by $K\times s^2d = Ks^2d$.

{\em Declaration rule complexity:} The declaration rules of ATBOC-Gauss/ATBOC-subGauss, ATBOC-1pExp, and LUCBBOC are the same, and they involve updating the distance matrix and performing SLINK clustering, and their computational complexity order is given by $Md$ and $M^2$, respectively. Hence, the computation complexity order of the declaration rule for both ATBOC and LUCBBOC is $Md + M^2$. 

 % In the sampling rule inner infimum problem is optimized over the arm pull proportions $\boldsymbol{w}$. Let $T_S$ be the number of iterations taken for the outer optimization. Then its complexity is $T_ST_AK^3s^42^sd$. The declaration rule involves the implementation of the SLINK algorithm, and its complexity is of the order of $M^2d$. Stopping rule and Declaration rule are the same for the LUCBBOC algorithm. The sampling rule involves implementing SLINK for each of the clusters to divide into two partitions, and its complexity order is given by $K\sum_{k\in [K]}|D_k|^2d\leq Ks^2d$. 

{\em Complexity of BOC-ELIM:} The computation in BOC-ELIM is to find the terms in Line 6 of Algorithm \ref{Algo:BOC-Elim}. It can be verified that its computational complexity order is $M^3$.
 
 A summary of the computational complexity orders of ATBOC, LUCBBOC and BOC-ELIM, for sampling, stopping, and declaration rules, is presented in Table \ref{table: complexityorder}.
\begin{table}[h!] 
\centering
\begin{tabular}{|c|c|c|c|c|}
\hline
 & ATBOC-Gauss/ATBOC-subGauss & ATBOC-1pExp & LUCBBOC & BOC-ELIM \\ \hline
Sampling rule & $T_ST_AMK^3s^42^sd$ & $T_SM2^{2s}K^2s^4$ & $Ks^2d$ & $M^3$ \\ \hline
Stopping rule & $T_AK^3s^42^sd$ & $2^{2s}K^2s^4$ & $T_AK^3s^42^sd$ & $1$ \\ \hline
Declaration rule & $Md+M^2$ & $M^2$ & $Md+M^2$ & $1$ \\ \hline
\end{tabular}
\caption{Complexity order of ATBOC, LUCBBOC and BOC-ELIM}
\label{table: complexityorder}
\end{table}

% {\em A further simplification:}
\begin{remark} \label{rem:partial}
Note that the number of QCQPs to be solved to solve the inner infimum problem is $K^3s^22^s$, which grows very fast with $K$ and $s$. Hence, for the implementation purposes, we propose the following simplification. We do the full search only at exponentially spaced time steps, that is, $t\in \{2^k:k \in \mathbb{N}\}$. At these exponential times, when we do a full search, the top $c$ optimization problems are recorded in terms of their minimum value, and in the other time steps, only the top $c$ optimization problems are solved. This approach reduces the number of QCQPs to be solved to $c$. For example, ATBOC-subGauss in Figure \ref{fig:2Dbernoulli} in Section \ref{sec:Simulations}, partial search is used. Even with a partial search, it satisfies the upper bound expected sample complexity guarantee. The computational gain in using partial search will be presented in the discussion that follows (Setup 2).
\end{remark}

\begin{figure}[htb]
    \centering
    % First image
    \begin{minipage}[t]{0.48\linewidth}
        \centering
        \begin{tikzpicture}
        \begin{axis}[
    xlabel = {Number of arms $(M)$}, 
    ylabel = {\small{$\text{per sample run time (seconds)}$}},
    xmin = 3, xmax = 21,
    ymin = -8, ymax=0.5,
    xtick = {4, 8, 14, 20},
    ytick = {-8, -6, -4, -2, 0},
    yticklabels = {$e^{-8}$, $e^{-6}$, $e^{-4}$, $e^{-2}$, $e^{0}$},
    grid = major,
    legend style = {at={(1,0.70)}, anchor=north east, nodes={scale=0.66}},
    width = 0.9\textwidth,
    height = 0.9\textwidth,
    tick label style = {font=\tiny},
    enlargelimits = false,
]
\addplot[red, thick, mark=square, mark size=2pt] coordinates {
    (4,-3.949) (8,-2.699) (14,-1.809) (20,-0.949)
};
\addplot[magenta, thick, mark=diamond, mark size=3pt] coordinates {
    (4,-3.858) (8,-2.395) (14,-1.054) (20,0.125)
};
\addplot[blue, thick, mark=o, mark size=2pt] coordinates {
    (4,-3.852) (8,-2.347) (14,-0.934) (20,0.289)
};
\addplot[green!70!black, thick, mark=triangle, mark size=3pt] coordinates {
    (4,-7.279) (8,-6.821) (14,-6.375) (20,-5.973)
};
\addplot[purple, thick, mark=+, mark size=2pt] coordinates {
    (4,-7.278) (8,-6.736) (14,-6.185) (20,-5.612)
};
\addplot[violet, thick, mark=x, mark size=2pt] coordinates {
    (4,-7.134) (8,-6.591) (14,-6.151) (20,-5.571)
};
\addplot[black, thick, mark=asterisk, mark size=2pt] coordinates {
    (4,-7.79660109) (8,-7.11601022) (14,-6.6233285) (20,-6.3206355)
};

\legend{ATBOC-1pExp $d=1$, ATBOC-subGauss $d=5$, ATBOC-subGauss $d=10$, LUCBBOC $d=1$, LUCBBOC $d=5$, LUCBBOC $d=10$, BOC-ELIM $d=1$}
\end{axis}
        \end{tikzpicture}
        \caption{Numerical run time comparison (Setup 1)}
        % Number of arms per cluster is $2$, i.e., $s=2$. Therefore, $M=2K$. Number of arms $M$ and dimension $d$ is varied.}
        \label{fig:complexitydimension}
    \end{minipage}
    \hfill
    % Second image
    \begin{minipage}[t]{0.48\linewidth}
        \centering
        \begin{tikzpicture}
        \begin{axis}[
    xlabel = {Number of arms $(M)$}, 
    ylabel = {\small{$\text{per sample run time (seconds)}$}},
    xmin = 3, xmax = 11,
    ymin = -8, ymax = 5, 
    xtick = {4, 6, 8, 10},
    ytick = {-6, -4, -2, 0, 2, 4},
    yticklabels = {$e^{-6}$, $e^{-4}$, $e^{-2}$, $e^{0}$, $e^{2}$, $e^{4}$},
    grid = major,
    legend style = {at={(1,0.74)}, anchor=north east, nodes={scale=0.66}},
    width = 0.9\textwidth,
    height = 0.9\textwidth,
    tick label style = {font=\tiny},
    enlargelimits = false,
]
\addplot[blue, thick, mark=o, mark size=2pt] coordinates {
    (4,-3.858) (6,1.468) (8,2.803) (10,4.501)
};
\addplot[green!70!black, thick, mark=triangle, mark size=3pt] coordinates {
    (4,-7.278) (6,-4.104) (8,-3.110) (10,-1.524)
};
\addplot[red, thick, mark=square, mark size=2pt] coordinates {
    (4,-3.858) (6,1.468) (8,2.062) (10,2.485)
};
\addplot[magenta, thick, mark=diamond, mark size=3pt] coordinates {
    (4,-7.278) (6,-4.104) (8,-3.917) (10,-3.607)
};
\legend{ATBOC-subGauss full, LUCBBOC full, ATBOC-subGauss partial, LUCBBOC partial}
\end{axis}
        \end{tikzpicture}
        \caption{Numerical run time comparison (Setup 2)}
        % Number of clusters $K=2$ and dimension $d=5$. Both clusters are of equal size $M=2s$. The number of arms $M$ is varied. }
        \label{fig:complexity-fullpartial}
    \end{minipage}
\end{figure}
\subsection{Numerical runtime comparison}
To compare the runtime of the algorithms through simulations, we considered the following two setups. In Setup 1, we present the per-sample runtime of the proposed algorithms as the number of clusters $K$ increases, and in Setup 2, we present the runtime as the cluster size $s$ increases.

{\em Setup 1: }In this experiment, we set the number of arms per cluster to $2$, and hence $s=2$. Therefore, in this setup, the number of arms is twice the number of clusters, i.e., $ M = 2K$. We vary the number of arms $M \in \{4, 18, 14, 20\}$ and the dimension of the samples $d \in \{ 1, 5, 10\}$. We assume that the arms follow multivariate Gaussian distributions and that their means across dimensions are equal. The means of the arms in each of the dimensions is $[0, 0.5, 10, 10.5, \ldots, 10K, 10K+0.5]$. Fig.~\ref{fig:complexitydimension} presents a comparison of the algorithms' per-sample runtimes. As expected, the runtime per sample increases monotonically with the increasing number of arms $M$ and the increasing number of dimensions $d$. We note that the LUCBBOC and BOC-ELIM algorithms achieve a significant reduction in runtime compared to ATBOC. 

{\em Setup 2: }In this experiment, we fix the number of clusters $K=2$ and the dimension $d=5$. We assume that both clusters have an equal number of arms and, hence, $M=2s$. We assume that the arms follow a multivariate Gaussian distribution and that their means across dimensions are equal. The means of the arms in each dimension is $[0, 0.5, \dots, 0.5(s-1), 10, 10.5, \dots, 10+0.5(s-1)]$.
We vary the number of arms $M\in \{2, 6, 8, 10\}$. In Fig~\ref{fig:complexity-fullpartial}, we observe that the partial search explained in Remark \ref{rem:partial} shows a significant reduction in runtime compared to the full search for both algorithms.

\section{Simulations} \label{sec:Simulations}

In this section, we compare the performance of the proposed algorithms ATBOC-Gauss, ATBOC-subGauss, ATBOC-1pExp, LUCBBOC, and BOC-ELIM with the following algorithms: BOC, Top2UCB, RR (which uses the round-robin sampling rule), and the Fixed Sample Size (FSS) algorithm. 
%ATBOC-Gauss, ATBOC-subGauss, ATBOC-1pExp, LUCBBOC, and BOC-ELIM are the algorithms proposed in Sections \ref{sec:AvgTrackingOC}, \ref{sec:LUCB}, and \ref{sec: BOC-Elim}. 
BOC was proposed in \cite{yang2024optimal}  for the scenario in which all arms in a cluster have identical distributions. Top2UCB is a Max-Gap identification algorithm proposed in \cite{katariya2019maxgap} that partitions the arm collection into two clusters. The proposed algorithms can work for more general settings than BOC and Top2UCB while providing comparable performance for the cases for which these algorithms are designed. The RR algorithm uses the round-robin sampling rule along with the same stopping and declaration rule as our proposed algorithms. This comparison shows the effectiveness of the proposed sampling rule. The FSS algorithm stops after a fixed number of samples and takes the same number of samples from each arm. The declaration rule for FSS is also SLINK clustering. The comparison with FSS is used to show the effectiveness of stopping rule in our sequential setting. 
  
\begin{figure}[htb]
    \centering
    % First image
    \begin{minipage}[t]{0.48\linewidth}
    \centering
        \begin{tikzpicture}
        \begin{axis}[
            xlabel = {$\log\left(1/\delta\right)$}, 
            ylabel = {Expected Sample Complexity},
            xmin =0, xmax=101,
            ymin=0, ymax=22000,
            xtick={0, 20, 40, 60, 80, 100},
            xticklabels={0, 20, 40, 60, 80, 100},
            grid=major,
            ytick = {0, 4000, 8000, 12000, 16000, 20000},
            yticklabels = {0, 4000, 8000, 12000, 16000, 20000},
            scaled y ticks = false,
            legend style={at={(1, 0)}, legend cell align=left,anchor=south east, nodes={scale=0.80}},
            width = 0.9\textwidth, % Adjust width to fit in minipage
            height = 0.9\textwidth, % Adjust height proportionally
            tick label style={font=\small},
            enlargelimits=false,
        ]
        \addplot[mark=square, red, thick] coordinates{
            (1, 7242.86)
            (12, 8214.66)
            (23, 9132.44)
            (34, 10022.06)
            (45, 10922.54)
            (56, 11858.4)
            (67, 12795.14)
            (78, 13758.08)
            (89, 14641.92)
            (100, 15540.68)
        }; 
        \addplot[mark=diamond,mark size=3pt, magenta, thick] coordinates{
            (1, 6094.58)
            (12, 7319.7)
            (23, 8621.62)
            (34, 9894.56)
            (45, 11056.96)
            (56, 12270.34)
            (67, 13522.18)
            (78, 14700.92)
            (89, 16067.94)
            (100, 17280.14)
        };
        \addplot[mark=o, blue, thick] coordinates{
            (1, 3459.4)
            (12, 5765.1)
            (23, 7515.6)
            (34, 9119)
            (45, 11026.3)
            (56, 12806.6)
            (67, 14497.2)
            (78, 16156.2)
            (89, 17756.8)
            (100, 19430.3)
        };
        \addplot[mark=triangle, mark size=3pt, green!50!black, thick] coordinates{
            (1, 2917.36)
            (12, 5673.4)
            (23, 7568.44)
            (34, 9441.58)
            (45, 11140.48)
            (56, 13078.38)
            (67, 15064.58)
            (78, 17068.32)
            (89, 19312.3)
            (100, 21689.76)
        };
        % \addplot[black, dashed, thick] coordinates{
        %     (20.80, 340.016)
        %     (24.76, 362.8652)
        %     (28.72, 385.7144)
        %     (32.68, 408.5636)
        %     (36.64, 431.4128)
        %     (40.60, 454.262)
        % };
        % \addplot[red, dashed, thick] coordinates{
        %     (20.80, 407.0176)
        %     (24.76, 432.15172)
        %     (28.72, 457.28584)
        %     (32.68, 482.41996)
        %     (36.64, 507.55408)
        %     (40.60, 532.6882)
        % };
        % \addplot[green!50!black, dashed, thick] coordinates{
        %     (20.80, 434.0232)
        %     (24.76, 467.15454)
        %     (28.72, 500.28588)
        %     (32.68, 533.41722)
        %     (36.64, 566.54856)
        %     (40.60, 599.6799)
        % };
        % \addplot[blue, dashed, thick] coordinates{
        %     (22.67, 361.5733)
        %     (27.00, 411.58)
        %     (31.33, 461.5867)
        %     (35.67, 511.5933)
        %     (40.00, 561.6)
        % };

        % \legend{Lower Bound, Lower Bound-Slope, ATBOC, LUCBBOC, Upper Bound-Slope}
        \legend{ATBOC-1pExp-0.1, ATBOC-1pExp-0.45, ATBOC-Gauss, LUCBBOC}
        % , UB Slope-0.1, UB Slope-0.45, ATBOC-SubGauss UB Slope}
        \end{axis}

        \end{tikzpicture}
        % \caption{Comparison between ATBOC-Gauss and ATBOC-1pExp. Synthetic Dataset 1. (Slopes: ATBOC=5.2,7.7,11.14, UB=6.34, 8.36, 11.54, LUCBBOC=35)}
        \caption{Asymptotic performance of ATBOC-Gauss, ATBOC-1pExp and LUCBBOC for a BOC problem with $M=7$ Gaussian arms, $K=3$ clusters, $d=1$ dimension.}
        \label{fig:Comparison between ATBOC-Gauss and ATBOC-1pExp}
        \end{minipage}
    \hfill
    % Second image
    \begin{minipage}[t]{0.48\linewidth}
    \centering
        \begin{tikzpicture}
        \begin{axis}[
            xlabel = {$\log\left(1/\delta\right)$},
            ylabel = {Expected Sample Complexity},
            xmin = 1, xmax = 40,
            ymin = 3000, ymax = 15000,
            grid=major,
            legend style={at={(0,1)}, anchor=north west, nodes={scale=0.80}},
            width=0.9\textwidth,
            height=0.9\textwidth,
            tick label style={font=\small},
            enlargelimits=false
        ]

        % --- Coordinates ---
        % DTRACK
        \addplot[mark=+,purple,thick] coordinates{
            (1,3386.9) (5.333,4092.6) (9.667,4816.1) (14,5465.5)
            (18.333,6160.4667) (22.667,6891.5667) (27,7566.0333)
            (31.333,8224.3333) (35.667,8926.6667) (40,9558.3)
        };

        % ATRACK
        \addplot[mark=o,blue,thick] coordinates{
            (1,3707.4643) (5.333,4568.6071) (9.667,5276.5357) (14,5991.5)
            (18.333,6738.6071) (22.667,7334.1429) (27,8055.3214)
            (31.333,8762.6786) (35.667,9493.9286) (40,10250.7857)
        };

        % WATRACK
        \addplot[mark=x,violet,thick] coordinates{
            (1,3252.4118) (5.333,3941.1176) (9.667,4613.1765) (14,5307.5294)
            (18.333,5985.1176) (22.667,6621.5882) (27,7274.5294)
            (31.333,7960.5294) (35.667,8646.1765) (40,9286.7059)
        };

        % % OPT-ATRACK
        % \addplot[mark=x,purple,thick] coordinates{
        %     (1,6362.1) (5.333,6793.54) (9.667,7259.3) (14,7691.06)
        %     (18.333,8056.5) (22.667,8415.44) (27,8849.02)
        %     (31.333,9237.48) (35.667,9627.5) (40,10038.24)
        % };

        % % BOC-ELIM
        % \addplot[mark=x,green,thick] coordinates{
        %     (1,4692.07) (5.333,5868.86) (9.667,6983.3) (14,8139.16)
        %     (18.333,9258.09) (22.667,10320.15) (27,11444.73)
        %     (31.333,12528.38) (35.667,13579) (40,14568.73)
        % };

        % % LUCBBOC
        % \addplot[mark=x,orange,thick] coordinates{
        %     (1,6199.6771) (5.333,8170.5625) (9.667,10034.8958) (14,12021.3594)
        %     (18.333,14147.0781) (22.667,16591.0208) (27,19071.099)
        %     (31.333,21806.5938) (35.667,24687.4010) (40,27482.3594)
        % };

        % UB SLOPE
        % \addplot[dashed,thick] coordinates{
        %     (1, 2*84*1 + 4000)
        %     (40, 2*84*40 + 4000)
        % };

        % \legend{DTBOC, ATBOC-SubGauss, WATBOC, ATBOC-1pExp, BOC-ELIM, LUCBBOC, UB-slope}
        \legend{DTBOC, ATBOC-Gauss, WATBOC}

        \end{axis}
        \end{tikzpicture}
        \caption{Asymptotic comparison of ATBOC-Gauss with direct tracking and weighted average tracking for a BOC problem with $M=7$ Gaussian arms, $K=3$ clusters, $d=1$ dimension. 
        % Synthetic Dataset 2. (Slopes: UB=168, ATBOC=168, DTBOC=155, WATBOC=155- last 4 points)
        }
        \label{fig:tracking comparison}
        \end{minipage}
\end{figure}

\subsection{Theoretical guarantees of the proposed algorithms}
To emphasize the asymptotic performance of the algorithms, we plot the sample complexity results versus $\log(1/\delta)$. Recall that $\delta$ is the upper bound on the empirical probability $P_e$. {For all plots, we average over $1000$ trials when plotting versus $\log(1/\delta)$ and over $100$ trials when plotting versus $\log(1/P_e)$.}

% \textbf{Synthetic Dataset 2:} Consider the clustering problem in Fig.~\ref{fig: SLINKdescription}, where the arms follow Gaussian distributions with unit variance.\\
% \textbf{Synthetic Dataset 3:} Consider $M=6$ arms, following $d=2$ dimensional Bernoulli distributions with means $\boldsymbol{\mu} = \begin{bmatrix}
%          0.1 & 0.1 & 0.85 & 0.9 & 0.85 & 0.9\\
%          0.45 & 0.5 & 0.9 & 0.9 & 0.15 & 0.1
%     \end{bmatrix}$. $M=6$ arms form $K=3$ clusters with the true cluster index vector $\boldsymbol{c} = [1, 1, 2, 2, 3, 3]$.
% \begin{itemize}
%     \item $M=6$ arms, $d=2$ dimension, $K=3$ clusters
%     \item Each arm follows a 2 dimensional Bernoulli distribution with means 
%     \item True cluster index vector $\boldsymbol{c} = [1, 1, 2, 2, 3, 3]$.
% \end{itemize}
\subsubsection{One-dimensional samples $(d=1)$} 

{\em Synthetic Dataset 1:} Consider $M=7$ arms, following Gaussian distributions with means $\boldsymbol{\mu} = [0, 0.5, 1, 2.5, 3, 4.5, 5]$ and unit variance, generating $d=1$ dimensional samples. The $M=7$ arms form $K=3$ clusters with the true cluster index vector $\boldsymbol{c} = [1, 1, 1, 2, 2, 3, 3]$.

 {\em Asymptotic Performance -- Gaussian:} 
Fig.~\ref{fig:Comparison between ATBOC-Gauss and ATBOC-1pExp} presents the expected sample complexity of ATBOC-Gauss, ATBOC-1pExp, and LUCBBOC for Synthetic Dataset 1 over the wide range of values of $\log(1/\delta)$. 
To numerically evaluate the asymptotic expected sample complexity slope, we use linear regression to fit a line to the last 5 points in the asymptotic regime. From Table \ref{Table: slope}, we observe that the expected sample complexity slope of ATBOC-1pExp (for two different values of $\zeta=0.1, 0.45$) is lower than ATBOC-Gauss and LUCBBOC. Additionally, the numerical asymptotic slopes of ATBOC-Gauss, ATBOC-1pExp, and LUCBBOC satisfy their corresponding theoretical asymptotic expected sample complexity slope guarantees. In Fig.~\ref{fig:Comparison between ATBOC-Gauss and ATBOC-1pExp}, we observe that ATBOC-Gauss and LUCBBOC perform better for higher error probability $\delta$, while ATBOC-1pExp (shown for two different values of $\zeta=0.1, 0.45$) performs better in the asymptotic regime, due to its better slope (lower slope). 
% Figure~\ref{fig:Comparison between ATBOC-Gauss and ATBOC-1pExp} shows the performance of the two variants of ATBOC along with their theoretical bounds. We denote the theoretical upper-bound slope by \textit{UB Slope} for $\zeta = 0.1, 0.45$. The ATBOC-SubGauss algorithm has an upper-bound theoretical guarantee of twice the lower-bound slope, denoted by \textit{ATBOC-SubGauss UB Slope}. From Figure~\ref{fig:Comparison between ATBOC-Gauss and ATBOC-1pExp}, we observe that ATBOC-1pExp satisfies the theoretical guarantees for $\zeta = 0.1$ and $0.45$; moreover, due to its improved slope, ATBOC-1pExp outperforms ATBOC-SubGauss and LUCBBOC in the asymptotic regime. 
\begin{table}[h!]
\centering
\begin{tabular}{|c|c|c|c|c|c|c|}
\hline
  & LUCBBOC & ATBOC-Gauss & DTBOC & WATBOC & \makecell{ATBOC-1pExp\\$\zeta=0.45$} & \makecell{ATBOC-1pExp\\$\zeta=0.1$}   \\ \hline
 Numerical & 186 & 150 & 155 & 155 & 114.2 & 83.7   \\ \hline
 Theoretical Upper Bound & 190 & 164 & - & - & 121.8 & 92.4    \\ \hline
 Theoretical Lower Bound & \multicolumn{6}{c|}{84} \\ \hline
\end{tabular}
\caption{Sample complexity slopes for Figures \ref{fig:Comparison between ATBOC-Gauss and ATBOC-1pExp} and \ref{fig:tracking comparison}.}
\label{Table: slope}
\end{table}

% \begin{table}[h!]
% \centering
% \begin{tabular}{|c|c|c|c|}
% \hline
%   & ATBOC-Gauss & \makecell{ATBOC-1pExp\\$\zeta=0.45$} & \makecell{ATBOC-1pExp\\$\zeta=0.1$}   \\ \hline
%  Numerical & 150 & 114.2 & 83.7   \\ \hline
%  Theoretical Upper Bound & 164 & 121.8 & 92.4    \\ \hline
%  Theoretical Lower Bound & \multicolumn{3}{c|}{84} \\ \hline
% \end{tabular}
% \caption{Sample complexity slopes for Figures \ref{fig:Comparison between ATBOC-Gauss and ATBOC-1pExp} and \ref{fig:tracking comparison}.}
% \label{Table: slope}
% \end{table}

{\em Comparison between different tracking strategies: }
For the Synthetic Dataset 1, we also compare ATBOC-Gauss with two other tracking schemes - Direct Tracking Bandit Online Clustering (DTBOC) and Weighted Average Tracking Bandit Online Clustering (WATBOC). Unlike in ATBOC, where the empirical proportions $\frac{\boldsymbol{N}(t)}{t}$ track the average of the estimates $\frac{1}{t}\sum_{s=1}^t\boldsymbol{w}(s)$ (Line 7 of Algorithm~\ref{Algo:AvgTrackingOC}), in DTBOC, the empirical proportions directly track the estimate $\boldsymbol{w}(t)$, and in WATBOC, the empirical proportions track an exponentially weighted average of estimates $\frac{1}{2(1-0.5^t)}\sum_{s=1}^t \frac{1}{2^{t-s}}\boldsymbol{w}(s)$. 
From Fig.~\ref{fig:tracking comparison}, we observe that all tracking schemes show comparable performance in terms of expected sample complexity. 
Although our work contains no theoretical guarantees for the D-tracking scheme due to technical difficulties highlighted in Remark \ref{rem:dtboc}, from Table \ref{Table: slope}, we observe that, for Synthetic Dataset 1, the numerical asymptotic slopes of all tracking schemes are comparable and satisfy the theoretical asymptotic slope derived for the average tracking scheme.

% We consider Synthetic Dataset 2. 
% Figure \ref{fig:tracking comparison} shows the performance comparison of the various algorithms ATBOC-SubGauss, DTBOC, and WATBOC, which differ only by their tracking scheme. In DTBOC, which is expanded as the Direct Tracking Bandit Online Clustering algorithm, at any time step $t$, empirical arm pull proportions will track $w(t)$.
% In the WATBOC, expanded as Weighted Average Tracking Bandit Online Clustering Algorithm, we do an exponential weighted average by giving more weights to the recently computed arm pull proportions with a weighting factor being $0.5$. For this problem setup, we observe that WATBOC shows better performance.
% Figure \ref{fig:TrackingStrategies} shows the performance of algorithms with various tracking strategies for synthetic dataset 1. We observe that the ATBOC-SubGauss satisfies the theoretical upper bound Guarantees. In DTBOC (Direct Tracking BOC), at each time $t$, we directly track the optimal arm pull proportion $\boldsymbol{w}^{*}(t)$ computed at time $t$. In WATBOC (Weighted Average Tracking BOC), we give exponential weight to arm pull proportions computed till time $t$ and track the weighted average of these arm pull proportions. We observe that all tracking strategies have a similar slope, and the performance of WATBOC is slightly better than DTBOC and ATBOC. 
% Figure \ref{fig:TrackingStrategies}, Synthetic Dataset 2 
\subsubsection{Multi-dimensional samples $(d>1)$}
{\em Asymptotic performance -- Gaussian: } 
We consider the BOC problem illustrated in Fig.~\ref{fig: SLINKdescription}, where there are $M = 6$ arms and $K = 3$ clusters, and the arms are assumed to be multivariate Gaussian distributed with dimension $d = 2$.
For the plots in Fig.~\ref{fig:optimalityd2}, the numerical slopes of ATBOC-Gauss (40), LUCBBOC (48), and RR (75) satisfy the corresponding theoretical upper-bound slopes of 40, 48, and 73, respectively.
We observe that the LUCBBOC algorithm, while being computationally efficient (as discussed in Section \ref{sec: compcomplx}), exhibits performance comparable to that of the ATBOC algorithm. Additionally, both ATBOC and LUCBBOC algorithms show considerably better performance than RR. 

%-----------------------------------------------------------------------------
\begin{figure}[htb]
    \centering
    % First image
    \begin{minipage}[t]{0.48\linewidth}
        \centering
        \begin{tikzpicture}
            \begin{axis}[
    xlabel = {$\log\left(1/\delta\right)$},
    ylabel = {Expected Sample Complexity},
    xmin = 1, xmax = 205,
    ymin = 1000, ymax = 20000,
    grid = major,
    legend style = {at={(0,1)}, anchor=north west, nodes={scale=0.80}},
    width = 0.9\textwidth,
    height = 0.9\textwidth,
    tick label style = {font=\small},
    enlargelimits = false,
]

% Upper bound slopes
% \addplot[
%     color=red,
%     thick,
%     dashed
% ] coordinates {
%     (1,1840) (200,9800)
% };

% \addplot[
%     color=blue,
%     thick,
%     dashed
% ] coordinates {
%     (1,2248) (200,11800)
% };

% \addplot[
%     color=green!70!black,
%     thick,
%     dashed
% ] coordinates {
%     (1,3872) (200,17900)
% };

% Algorithms
\addplot[
    color=blue,
    thick,
    mark=o,
    mark size=2pt
    % error bars/.cd,
    %     y dir=both,
    %     y explicit
] coordinates {
    (1,1216.685) (23,2280.63) (45,3230.715) (67,4207.625)
    (89,5085.445) (111,6008.605) (133,6894.17)
    (155,7792.87) (177,8685.21) (200,9568.005)
%     (1,1247.01) +- (0,131.9737)
% (23,2268.66) +- (0,154.4332)
% (45,3228.34) +- (0,180.6957)
% (67,4134.57) +- (0,191.9377)
% (89,5061.04) +- (0,215.7569)
% (111,5952.06) +- (0,221.1047)
% (133,6851.54) +- (0,229.9739)
% (155,7724.38) +- (0,244.3666)
% (177,8624.17) +- (0,260.0051)
% (200,9506.78) +- (0,257.4967)
};

\addplot[
    color=green!60!black,
    thick,
    mark=triangle,
    mark size=3pt
    % error bars/.cd,
    %     y dir=both,
    %     y explicit
] coordinates {
    (1,1485.375) (23,2746.84) (45,3899.91) (67,5005.465)
    (89,6123.06) (111,7242.93) (133,8312.55)
    (155,9401.47) (177,10498.74) (200,11586.845)
%     (1,1459.84) +- (0,143.7250)
% (23,2696.44) +- (0,148.8741)
% (45,3896.85) +- (0,160.3709)
% (67,4962.61) +- (0,174.2728)
% (89,6065.09) +- (0,190.5981)
% (111,7133.68) +- (0,201.7094)
% (133,8254.98) +- (0,214.1964)
% (155,9333.86) +- (0,233.5389)
% (177,10438.46) +- (0,233.0016)
% (200,11519.84) +- (0,244.3196)
};

\addplot[
    color=violet,
    thick,
    mark=x,
    mark size=2pt
] coordinates {
    (1,2705.364) (23,4608.028) (45,6303.794) (67,7983.346)
    (89,9639.336) (111,11254.028) (133,12905.028)
    (155,14607.664) (177,16281.047) (200,17951.523)
};

\legend{
% UB-slope, UB-LUCBBOC-slope, UB-RR-slope, 
ATBOC-Gauss, LUCBBOC, RR}

\end{axis}

        \end{tikzpicture}
        % \caption{Asymptotic guarantees of ATBOC, LUCBBOC, RR}
        % \label{fig:slopeLUCBBOC}
        \caption{Asymptotic comparison of ATBOC-Gauss and LUCBBOC with RR for a BOC problem with $M=6$ Gaussian arms, $K=3$ clusters, $d=2$ dimensions.
        % Multivariate Gaussian. Synthetic Dataset 3. (Slopes: LB=20, UB=40, ATBOC=40, LUCBBOCUB = 48, LUCBBOC=48, RRUB = 72, RR=75)
        }
        \label{fig:optimalityd2}
    \end{minipage}
    \hfill
    % Second image
    \begin{minipage}[t]{0.48\linewidth}
    \centering
        \begin{tikzpicture}[scale=1]
        \begin{axis}[
    xlabel = {$\log\left(1/\delta\right)$}, 
    ylabel = {Expected Sample Complexity},
    grid=major,
    xmin = 0, xmax = 210,
    ymin = 2000, ymax = 20000,
    legend style={at={(0,1)}, anchor=north west, nodes={scale=0.8}},
    width=0.9\textwidth,
    height=0.9\textwidth,
    tick label style={font=\small}
]

% --- UB-slope ---
% \addplot[dashed,black,thick] coordinates{
%     (1.000, 3566.00)
%     (23.111, 5025.33)
%     (45.222, 6484.67)
%     (67.333, 7944.00)
%     (89.444, 9403.33)
%     (111.556, 10862.67)
%     (133.667, 12322.00)
%     (155.778, 13781.33)
%     (177.889, 15240.67)
%     (200.000, 16700.00)
% };

% --- ATBOC-subGauss ---
% \addplot[mark=x,blue,thick] coordinates{
%     (1.000, 2573.44)
%     (23.111, 4171.72)
%     (45.222, 5875.96)
%     (67.333, 7308.08)
%     (89.444, 8968.42)
%     (111.556, 10591.49)
%     (133.667, 12013.79)
%     (155.778, 13435.84)
%     (177.889, 14813.44)
%     (200.000, 16269.36)
% };
\addplot[
    mark=o,
    blue,
    thick,
    error bars/.cd,
        y dir=both,
        y explicit
] coordinates{
    % (1.000, 2573.44) +- (0,233.713)
    % (23.111, 4171.72) +- (0,316.165)
    % (45.222, 5875.96) +- (0,317.477)
    % (67.333, 7308.08) +- (0,401.066)
    % (89.444, 8968.42) +- (0,626.500)
    % (111.556, 10591.49) +- (0,487.560)
    % (133.667, 12013.79) +- (0,511.051)
    % (155.778, 13435.84) +- (0,604.857)
    % (177.889, 14813.44) +- (0,722.473)
    % (200.000, 16269.36) +- (0,876.110)

    (1.000, 2573.44) +- (0,116.8565)
(23.111, 4171.72) +- (0,158.0825)
(45.222, 5875.96) +- (0,158.7385)
(67.333, 7308.08) +- (0,200.5330)
(89.444, 8968.42) +- (0,313.2500)
(111.556, 10591.49) +- (0,243.7800)
(133.667, 12013.79) +- (0,255.5255)
(155.778, 13435.84) +- (0,302.4285)
(177.889, 14813.44) +- (0,361.2365)
(200.000, 16269.36) +- (0,438.0550)
};

% --- LUCBBOC ---
% \addplot[mark=x,red,thick] coordinates{
%     (1.000, 2972.06)
%     (23.111, 4887.27)
%     (45.222, 6673.66)
%     (67.333, 8453.26)
%     (89.444, 10241.73)
%     (111.556, 12017.82)
%     (133.667, 13732.14)
%     (155.778, 15495.51)
%     (177.889, 17183.88)
%     (200.000, 18981.09)
% };
\addplot[mark=triangle,green!70!black,thick, mark size = 3pt,
    error bars/.cd,
        y dir=both,
        y explicit] coordinates{
    % (1.000, 3071.18) +- (0, 543.413)
    % (23.111, 4921.03) +- (0, 572.199)
    % (45.222, 6706.64) +- (0, 617.455)
    % (67.333, 8443.76) +- (0, 607.869)
    % (89.444, 10174.22) +- (0, 642.032)
    % (111.556, 11977.21) +- (0, 672.400)
    % (133.667, 13729.37) +- (0, 764.198)
    % (155.778, 15424.89) +- (0, 782.582)
    % (177.889, 17206.28) +- (0, 892.852)
    % (200.000, 18831.92) +- (0, 855.869)

    (1.000, 3071.18) +- (0, 271.7065)
(23.111, 4921.03) +- (0, 286.0995)
(45.222, 6706.64) +- (0, 308.7275)
(67.333, 8443.76) +- (0, 303.9345)
(89.444, 10174.22) +- (0, 321.0160)
(111.556, 11977.21) +- (0, 336.2000)
(133.667, 13729.37) +- (0, 382.0990)
(155.778, 15424.89) +- (0, 391.2910)
(177.889, 17206.28) +- (0, 446.4260)
(200.000, 18831.92) +- (0, 427.9345)
};

\legend{
% UB-slope,
ATBOC-subGauss,
LUCBBOC
}

\end{axis}

        \end{tikzpicture}
        \caption{Asymptotic Performance of ATBOC-subGauss and LUCBBOC for a BOC problem with $M=6$ Bernoulli arms and $K=3$ clusters, $d=2$ dimensions.
        % Multivariate Bernoulli. Synthetic Dataset 4. (slopes: UB = 66, ATBOC=64, LUCBBOC=79).
        }
        \label{fig:2Dbernoulli}
        \end{minipage}
\end{figure}
% \begin{table}[h!]
% \centering
% \begin{tabular}{|c|c|c|c|c|}
% \hline
%   & RR & LUCBBOC & ATBOC-Gauss &  Lower Bound  \\ \hline
%  Numerical & 75 & 48 & 40 &  - \\ \hline
%  Theoretical &  73 & 48 & 40  & 20  \\ \hline
% \end{tabular}
% \caption{Sample complexity slopes for Figure \ref{fig:optimalityd2}.}
% \label{Table: slope2}
% \end{table}

% We consider Synthetic Dataset 3. We show the asymptotic behavior of our proposed algorithms in Figure \ref{fig:optimalityd2}.
% We plot the lower bound in Theorem \ref{Theorem:lowerbound} (marked as Lower Bound) and a shifted version (marked as Lower Bound-Slope) for ease of comparison with the proposed algorithms. 
% The asymptotic upper bound of the ATBOC algorithm, as dictated by Theorems \ref{Theorem3:AlmostSureOptimal} and \ref{Theorem4:AsymptoticOptimal}, is plotted with a slope twice that of the Lower Bound-Slope; we call it as Upper Bound-Slope. We observe that the proposed ATBOC algorithm has a slope less than that of our derived upper bound (Theorem \ref{Theorem4:AsymptoticOptimal}). The more computationally efficient LUCBBOC algorithm has a slightly higher slope but performs comparably to the ATBOC algorithm. 

\noindent {\em Asymptotic Performance -- sub-Gaussian:} Consider $M=6$ arms, following $d=2$ dimensional Bernoulli distributions with means $\boldsymbol{\mu} = \begin{bmatrix}
         0.1 & 0.1 & 0.85 & 0.9 & 0.85 & 0.9\\
         0.45 & 0.5 & 0.9 & 0.9 & 0.15 & 0.1
    \end{bmatrix}$. The $M=6$ arms form $K=3$ clusters with the true cluster index vector $\boldsymbol{c} = [1, 1, 2, 2, 3, 3]$. 
In Fig.~\ref{fig:2Dbernoulli}, we observe that the performance of the ATBOC algorithm is comparable to that of the LUCBBOC algorithm. 
    {The vertical markers represent the standard deviation of the empirical sample complexity at each point in the plot. We do not include them in the other plots to avoid clutter.}
    As expected, the numerical asymptotic slopes of both ATBOC (64) and LUCBBOC (79) satisfy their corresponding theoretical asymptotic slopes of 66 and 80, respectively.
    % shows the performance of ATBOC-SubGauss and LUCBBOC algorithm. We observe that ATBOC-SubGauss satisfies the theoretical upper bound guarantee for this dataset.
\subsection{Performance comparison with existing literature}
To compare performance with the algorithms in the literature, we plot the sample complexity results versus $\log(1/P_e)$, where $P_e$ is the empirical probability of error. Moreover, unlike in the previous section, where the emphasis was on validating asymptotic guarantees in the low error probability regime, here we focus on the more practical non-asymptotic regime, where the probability of error is between $e^{-1}$ and $e^{-4}$.
\subsubsection{One-dimensional samples $(d=1)$}
{\em Comparison with RR and FSS: }
We use Synthetic Dataset 1. From Fig.~\ref{fig:SeqImproves}, we observe that ATBOC shows the best performance, followed by LUCBBOC. As expected, both ATBOC and LUCBBOC, as a consequence of their sampling and stopping rules, perform better than RR and FSS.

% Figure \ref{fig:SeqImproves} presents a performance comparison of various algorithms for synthetic dataset 1. The Round Robin algorithm has the same stopping rule as that of our proposed ATBOC-subGauss algorithm, except that there is no sampling rule; instead, at each time, arms are chosen to get a sample in the round robin fashion. In the FSS Round Robin, there is no stopping rule as well, and the algorithm stops after some fixed number of predetermined samples. All the sequential algorithms perform better than the Fixed sample size approach FSS Round Robin. In addition, our sampling strategy in our proposed algorithms ATBOC and LUCBBOC improves performance over the round robin sampling strategy.\\
\noindent {\em Comparison with MaxGap Algorithms: }
Consider $M=6$ arms, where each arm follows Gaussian distribution with means $\boldsymbol{\mu} = [0, 1, 2, 3, 5, 6]$, generating $d=1$ dimensional samples. The $M=6$ arms form $K=2$ clusters, where the true cluster index vector is $\boldsymbol{c} = [1, 1, 1, 1, 2, 2]$. From Fig.~\ref{fig:differentmean-literaturecomp}, we observe that the proposed algorithms ATBOC-Gauss, LUCBBOC and BOC-ELIM show comparable performance to Top2UCB, which is designed for the special case of $K=2$ clusters with $d=1$ dimensional samples. Since BOC is designed for the setting where the arms in a cluster have an identical distribution, it shows poor performance.

% We consider the synthetic dataset 4. We compared the performance of our proposed algorithms with the max-gap algorithms (MaxGapElim, MaxGapUCB, MaxGapTop2UCB) proposed in \cite{katariya2019maxgap} for $2$-cluster problems. For the synthetic problem instance, we found MaxGapTop2UCB to be the most effective among the three MaxGap algorithms. To keep the plot uncluttered, we compared our proposed algorithms solely with MaxGapTop2UCB, in Figure \ref{fig:differentmean-literaturecomp}. We plot the expected sample complexity of our proposed algorithms ATBOC-SubGauss, LUCBBOC, and BOC-Elim. We observe that our proposed algorithms ATBOC-SubGauss and LUCBBOC shows comparable performance to MaxGapTop2UCB, which is specifically designed for $2$-cluster problems. We observe that BOC shows poor performance for the more spread clusters.

\begin{figure}[htb]
    \centering
    % First image
    \begin{minipage}[t]{0.32\linewidth}
    \centering
        \begin{tikzpicture}
        \begin{axis}[
            xlabel = {$\log\left(1/P_e\right)$}, 
            ylabel = {Expected Sample Complexity},
            xmin =0.0, xmax=4,
            ymin=10, ymax=240,
            xtick={0,1,2,3,4},
            xticklabels={0,1,2,3,4},
            grid=major,
            ytick = {0,40,80,120,160,200,240},
            legend style={at={(0, 1)}, anchor=north west, nodes={scale=0.80}},
            width = \textwidth,
            height = \textwidth,
            tick label style={font=\small},
            enlargelimits=false,
        ]

        % ---------------- ATBOC ----------------
        \addplot[mark=o,blue,thick] coordinates{
            (0.268, 18.46)
            (0.616, 34.39)
            (1.094, 53.385)
            (1.609, 69.545)
            (2.251, 89.46)
            % (2.465, 102.865)
            (3.101, 122.11)
            (3.507, 138.275)
        };

        % ---------------- Round Robin ----------------
        \addplot[mark=x,violet,thick] coordinates{
            (0.699, 41.0505)
            (1.313, 71.5609)
            (1.840, 97.1524)
            (2.324, 118.8336)
            (2.908, 144.7899)
            (3.485, 170.5007)
        };

        % ---------------- FSS ----------------
        \addplot[mark=+,purple,thick] coordinates{
            (0.297, 20)
            (0.627, 40)
            (1.156, 70)
            (1.574, 100)
            (2.004, 130)
            (2.454, 170)
            (2.914, 200)
        };

        % ---------------- LUCBBOC ----------------
        \addplot[mark=triangle,green!70!black,thick, mark size=3pt] coordinates{
            (0.359, 18.732)
            (0.769, 36.875)
            (1.186, 56.518)
            (1.593, 76.9)
            (2.079, 96.786)
            (2.564, 117.815)
            (3.057, 144.093)
            (3.575, 167.264)
        };

%         \addplot[mark=x,magenta, thick] coordinates{
% % \addplot[mark=x, thick] coordinates{
% (0.168 , 11.6672)
% (0.327 , 21.0583)
% (0.603 , 39.8212)
% (1.584 , 101.8499)
% (2.214 , 134.0151)
% (3.038 , 166.3669)
% (3.879 , 198.0093)
% % };
% };

        \legend{ATBOC-Gauss, RR, FSS, LUCBBOC
        % , BOC-ELIM
        }

        \end{axis}
        \end{tikzpicture}
        \caption{Comparison of ATBOC-Gauss and LUCBBOC with RR and FSS for a BOC problem with $M=7$ Gaussian arms, $K=3$ clusters, $d=1$ dimension.}
        \label{fig:SeqImproves}
    \end{minipage}
    \hfill
    % Second image
    \begin{minipage}[t]{0.32\linewidth}
    \centering
        \begin{tikzpicture}
        \begin{axis}[
    xlabel = {$\log\left(1/P_e\right)$},
    ylabel = {Expected Sample Complexity},
    xmin = 0, xmax = 4,
    ymin = 0, ymax = 450,
    grid=major,
    legend style={at={(0,1)}, anchor=north west, nodes={scale=0.80}},
    width=\textwidth,
    height=\textwidth,
    tick label style={font=\small},
    enlargelimits=false
]

% --- BOC ---
\addplot[mark=+,purple,thick] coordinates{
    (0.535, 7.8974)
    (0.747, 17.2132)
    (1.022, 40.5068)
    (1.232, 73.2442)
    (1.446, 114.3862)
    (1.690, 160.5877)
    (1.966, 210.6666)
    (2.450, 314.8081)
    (2.964, 422.3237)
};

% --- ATBOC ---
\addplot[mark=o,blue,thick] coordinates{
    (0.255, 12.3122)
    (0.985, 37.1823)
    (1.663, 65.0599)
    (2.215, 82.7066)
    (2.665, 98.7770)
    (3.052, 111.2877)
    (3.456, 123.2602)
};

% --- BOC-ELIM ---
\addplot[mark=asterisk,black,thick] coordinates{
    (0.616, 8.26)
    (0.579, 17.93)
    (1.714, 68.39)
    (3.912, 144.12)
};

% --- LUCBBOC ---
\addplot[mark=triangle,green!70!black, mark size=3pt, thick] coordinates{
    (0.976, 21.173)
    (1.527, 41.534)
    (2.046, 66.159)
    (2.674, 85.544)
    (3.379, 107.462)
};

% --- MaxGapTop2UCB ---
\addplot[mark=x,violet,thick] coordinates{
(0.460  , 6.134)
(0.667  ,10.251)
(0.857  ,14.943)
(1.050  ,20.474)
(1.252  ,26.767)
(1.666  ,36.269)
% (2.128  ,48.506)
(2.476  ,70.975)
};
% % --- LB slope ---
% \addplot[dashed,thick] coordinates{
%     (0.5, 58*0.5)
%     (3.5, 58*3.5)
% };

\legend{BOC, ATBOC-Gauss, BOC-ELIM, LUCBBOC, Top2UCB}

\end{axis}

        \end{tikzpicture}
        \caption{Comparison with Top2UCB and BOC for a BOC problem with $M=6$ Gaussian arms, $K=2$ clusters and $d=1$ dimensional samples.}
        \label{fig:differentmean-literaturecomp}
        \end{minipage}
\hfill
\begin{minipage}[t]{0.32\linewidth}
    \centering
        \begin{tikzpicture}
            \begin{axis}[
    xlabel = {$\log\left(1/P_e\right)$}, 
    ylabel = {Expected Sample Complexity},
    xmin =1.8, xmax=6.5,
    ymin=10, ymax=50,
    xtick={2,3,4, 5, 6},
    grid=major,
    legend style={at={(0,1)}, anchor=north west, nodes={scale=0.80}},
    width = \textwidth,
    height = \textwidth,
    tick label style={font=\small},
    enlargelimits=false,
]

% ---------------- BOC ----------------
\addplot[mark=+,purple,thick] coordinates{
    (2.00767918, 18.5115)
    (2.91323105, 23.9462)
    (3.90207267, 30.0744)
    (5.18498868, 36.9301)
    (6.31996861, 42.8602)
};

% ---------------- BOC-SLINK ----------------
\addplot[mark=x,violet,thick] coordinates{
    (2.903, 18.6432)
    (4.057, 23.7117)
    (5.389, 29.8923)
};
% ---------------- ATBOC ----------------
\addplot[mark=o,blue,thick] coordinates{
    (1.515, 12.5708)
    (2.003, 16.0747)
    % (2.246, 19.0870)
    % (2.392, 21.1951)
    % (2.757, 24.3322)
    % (3.502, 27.1505)
    % (3.907, 28.8829)
    (4.497, 30.4203)
    (4.852, 32.4950)
    (5.412, 34.7179)
};

% ---------------- LUCBBOC ----------------
\addplot[mark=triangle, mark size=3pt, green!70!black,thick] coordinates{
    (1.559, 11.7764)
    (2.576, 14.7419)
    (3.540, 17.7473)
    (4.303, 20.4976)
    (5.424, 24.9202)
    (6.268, 29.8242)
};

\legend{BOC, BOC-SLINK, ATBOC-Gauss, LUCBBOC}

\end{axis}

        \end{tikzpicture}
%         \begin{tikzpicture}
%         \begin{axis}[
%     xlabel = {$\log\left(\frac{1}{\mathbb{E}[P_e]}\right)$}, 
%     ylabel = {\small{Empirical Stopping times}},
%     xmin =1.8, xmax=3.5,
%     ymin=300, ymax=2100,
%     xtick={0,1,2,3,4},
%     grid=major,
%     legend style={at={(0,1)}, anchor=north west, nodes={scale=0.66}},
%     width = \textwidth,
%     height = \textwidth,
%     tick label style={font=\tiny},
%     enlargelimits=false,
% ]

% % ---------------- BOC ----------------
% \addplot[mark=x,blue,thick] coordinates{
%     (1.863, 1087.7352)
%     (2.388, 1283.4598)
%     (2.914, 1463.2403)
%     (3.437, 1629.8434)
% };

% % ---------------- ATBOC ----------------
% \addplot[mark=x,red,thick] coordinates{
%     (1.882, 523.1739)
%     (2.121, 664.8152)
%     (2.219, 799.2065)
%     (2.322, 941.2065)
%     (2.577, 1213.2283)
%     (2.910, 1365.1087)
%     (3.419, 1561.2174)
% };

% % ---------------- LUCBBOC ----------------
% \addplot[mark=x,purple,thick] coordinates{
%     (1.817, 331.5119)
%     (1.837, 360.3869)
%     (1.907, 398.4658)
%     (1.969, 444.8735)
%     (2.070, 498.5491)
%     (2.195, 607.2262)
%     (2.288, 725.5551)
%     (2.433, 836.5923)
%     (2.485, 956.5848)
%     (2.618, 1068.6324)
%     (2.708, 1181.0432)
%     (2.750, 1300.8646)
%     (2.986, 1535.4524)
%     (3.147, 1769.7813)
%     (3.288, 2005.9152)
% };

% \legend{BOC, ATBOC, LUCBBOC}

% \end{axis}

%         \end{tikzpicture}
        \caption{Comparison of ATBOC-Gauss and LUCBBOC with BOC and BOC-SLINK for a BOC problem with $M=11$ Gaussian arms, $K=4$ clusters, $d=3$ dimension.}
        \label{fig:moderate}
        \end{minipage}
        
\end{figure}

\subsubsection{Multi-dimensional samples $(d>1)$}
{\em Comparison with BOC: }
We consider the same BOC problem as the one considered in \cite{yang2024optimal} where the BOC algorithm is designed for the setting in which all the arms in a cluster have an identical distribution. The BOC problem consists of $M=11$ arms, where each arm follows Gaussian distribution with means $$\boldsymbol{\mu} = \begin{bmatrix}
        0 & 0 & 0 & 0 & 0 & 0 & 0 & 0 & 0 & 5 & 5\\
        0 & 0 & 10 & 10 & 10 & 10 & 0 & 0 & 0 & 0 & 0\\
        0 & 0 & 0 & 0 & 0 & 0 & 10 & 10 & 10 & 0 & 0
    \end{bmatrix},$$ generating $d=3$ dimensional samples. The $M=6$ arms form $K=4$ clusters, where the true cluster index vector is $\boldsymbol{c} =[1, 1, 2, 2, 2, 2, 3, 3, 3, 4, 4]$.
    From Fig.~\ref{fig:moderate}, we observe that for this problem instance, the LUCBBOC algorithm shows the best performance. The performance of ATBOC-Gauss is better than that of BOC. However, when replacing the K-Means-based recommendation rule in BOC with the SLINK clustering algorithm, which we call BOC-SLINK, it shows better performance than ATBOC-Gauss. From this behavior, we infer that SLINK clustering is better suited to this problem instance.

\subsection{Performance on real world datasets}
We consider three real-world datasets to present the performance of our proposed algorithms in the problems of traffic monitoring, pattern recognition, and market segmentation. Here, we emphasize that the observations in each of these real-world problems do not strictly follow any probability distribution.
\subsubsection{One-dimensional samples $(d=1)$} 

{\em NYC TLC Dataset: }
We consider a traffic monitoring problem with the New York City TLC Trip Record Data \cite{nyc_tlc_trip_data} of January 2025, which consists of taxi pickup times across Manhattan. 
We observe the inter-booking times of the taxis as the samples ($d=1$).
We construct a BOC problem by partitioning Manhattan into six regions and grouping these $M=6$ regions into $K=3$ classes: \textit{Busy Zone}, \textit{Moderate Zone}, and \textit{Quiet Zone}. This results in a problem with $M=6$ arms and $K=3$ clusters with $d=1$ dimensional observations. Fig.~\ref{fig:taxi} shows that ATBOC-Gauss performs better compared to RR and FSS. 
%4thheree

\subsubsection{Multi-dimensional samples $(d>1)$} 

{\em MNIST Dataset: }
We consider a pattern recognition problem with the Modified National Institute of Standards and Technology (MNIST) dataset \cite{lecun2002gradient}, which consists of images of digits from 0 to 9. 
We consider each of the digits as a cluster. We divide the set of all images corresponding to the digit $i$ into two subsets, where each of these 2 subsets is a data sequence corresponding to the digit $i$. Hence, we have $20$ data sequences and $10$ clusters. Each image is a $28 \times 28$ grayscale image, resulting in a $784$-dimensional data point. We reduce the dimensionality to $d = 10$ for our simulations using Principal Component Analysis (PCA).
Fig.~\ref{fig:MNIST} shows that ATBOC-Gauss performs better than LUCBBOC.

%5thhere
\begin{figure}[htb]
    \centering
    % First image
    \begin{minipage}[t]{0.48\linewidth}
    \centering
        \begin{tikzpicture}
        \begin{axis}[
    xlabel = {$\log\left(1/P_e\right)$}, 
    ylabel = {Expected Sample Complexity},
    xmin =0, xmax=2.5,
    ymin=0, ymax=2100,
    grid=major,
    legend style={at={(0,1)}, anchor=north west, nodes={scale=0.8}},
    width=0.85\textwidth,
    height=0.85\textwidth,
    tick label style={font=\small}
]

% --- ATBOC ---
\addplot[mark=square,red,thick] coordinates{
    (0.885, 420.9637)
    (1.070, 544.4752)
    (1.139, 665.6601)
    (1.333, 862.8812)
    (1.466, 1047.9340)
    (1.690, 1258.7063)
    (2.053, 1640.0924)
    (2.382, 2031.4554)
};

% --- FSS ---
\addplot[mark=+,purple,thick] coordinates{
    (0.373, 201)
    (0.692, 501)
    (0.885, 701)
    (1.118, 1001)
    (1.372, 1501)
    (1.680, 2001)
    % (2.026, 2501)
    % (2.242, 3001)
    % (2.488, 3501)
    % (2.725, 4001)
};

% --- Round Robin ---
\addplot[mark=x,violet,thick] coordinates{
    (1.091, 673.105)
    (1.239, 907.458)
    (1.420, 1215.395)
    (1.621, 1546.592)
    (1.876, 1988.345)
    % (1.994, 2421.201)
    % (2.307, 3188.059)
    % (2.769, 4078.970)
};

\legend{ATBOC-1pExp, FSS, RR}

\end{axis}

        \end{tikzpicture}
        \caption{Comparison with FSS and RR for a traffic monitoring problem with NYC TLC Dataset to group the $M=6$ regions of Manhattan city into $K=3$ clusters on observing taxi inter-booking times ($d=1$).}
        \label{fig:taxi}
    \end{minipage}
    \hfill
    % Second image
    \begin{minipage}[t]{0.48\linewidth}
    \centering
        \begin{tikzpicture}
        \begin{axis}[
    xlabel = {$\log\left(1/\delta\right)$}, 
    ylabel = {Expected Sample Complexity},
    grid=major,
    xmin = 0, xmax = 45,
    ymin = 0, ymax = 300,
    legend style={at={(0,1)}, anchor=north west, nodes={scale=0.8}},
    width=0.85\textwidth,
    height=0.85\textwidth,
    tick label style={font=\small}
]

% --- ATBOC ---
\addplot[mark=o,blue,thick] coordinates{
    (1, 110.23)
    (5.333, 114.43)
    (9.667, 118.4)
    (14, 122.28)
    (18.333, 125.85)
    (22.667, 129.93)
    (27, 133.43)
    (31.333, 137.1)
    (35.667, 140.98)
    (40, 145.11)
};

% --- LUCBBOC ---
\addplot[mark=triangle, mark size=3pt, green!70!black,thick] coordinates{
    (1, 207.46)
    (5.333, 215.86)
    (9.667, 223.36)
    (14, 228.03)
    (18.333, 232.42)
    (22.667, 236.22)
    (27, 239.42)
    (31.333, 242.92)
    (35.667, 245.64)
    (40, 248.82)
};

\legend{
ATBOC-Gauss,
LUCBBOC
}

\end{axis}

        \end{tikzpicture}
        \caption{Performance for a pattern recognition problem with MNIST Dataset to group $M=20$ sequences of digit images into $K=10$ digits on observing PCA reduced images of digits ($d=10$).}
        \label{fig:MNIST}
    \end{minipage}
\end{figure}

\begin{figure}[htb]
    \centering
    % First image
    \begin{minipage}[t]{0.48\linewidth}
    \centering
    \begin{tikzpicture}
\begin{axis}[
    xlabel={Comedy},
    ylabel={Drama},
    grid=major,
    width=0.88\textwidth,
    height=0.85\textwidth,
    legend style={at={(0,1)},anchor=north west}
]

% First two points: filled red triangles
\addplot[
    only marks,
    mark=triangle*,
    mark size=4pt,
    draw=red,
    fill=red
] coordinates {
    (2.62950971, 2.9787234)
    (2.55365194, 2.97565374)
};
\addlegendentry{Cluster 1}

% Third and fourth points: filled blue squares
\addplot[
    only marks,
    mark=square*,
    mark size=3pt,
    draw=blue,
    fill=blue
] coordinates {
    (3.31913499, 3.32830931)
    (3.20696203, 3.17278481)
};
\addlegendentry{Cluster 2}

% Fifth point: filled green circle
\addplot[
    only marks,
    mark=*,
    mark size=4pt,
    draw=green!60!black,
    fill=green!60!black
] coordinates {
    (3.16509434, 3.88207547)
};
\addlegendentry{Cluster 3}

% Sixth point: filled orange diamond
\addplot[
    only marks,
    mark=diamond*,
    mark size=4pt,
    draw=orange!80!black,
    fill=orange!80!black
] coordinates {
    (3.87246722, 4.16686532)
};
\addlegendentry{Cluster 4}

\end{axis}
\end{tikzpicture}
\caption{Average user ratings for the Drama and Comedy genres for the considered $M = 6$ users in the simulation shown in Figure~\ref{fig:MovieLens}.
}
\label{fig: movielens means}
    
    \end{minipage}
    \hfill
    % Second image
    \begin{minipage}[t]{0.48\linewidth}
    \centering
    \begin{tikzpicture}
            \begin{axis}[
    xlabel = {$\log\left(1/P_e\right)$}, 
    ylabel = {Expected Sample Complexity},
    xmin =0.5, xmax=4.5,
    ymin=10, ymax=70,
    xtick={0.5, 1, 1.5, 2, 2.5, 3, 3.5, 4, 4.5},
    grid=major,
    legend style={at={(0,1)}, anchor=north west, nodes={scale=0.80}},
    width = 0.8\textwidth,
    height = 0.8\textwidth,
    tick label style={font=\small},
    enlargelimits=false,
]

% ---------------- RR ----------------
\addplot[mark=+,violet,thick] coordinates{
    (0.962, 13.798)
    (1.714, 23.850)
    (2.603, 34.866)
    (3.323, 42.196)
    (4.268, 56.526)
};

% ---------------- ATBOC ----------------
\addplot[mark=o,blue,thick] coordinates{
    (0.927, 12.394)
    (1.704, 19.762)
    (2.365, 25.354)
    (3.170, 30.504)
    (4.268, 38.536)
};

% ---------------- LUCBBOC ----------------
\addplot[mark=triangle, mark size = 3pt, green!70!black,thick] coordinates{
    (0.839, 12.063)
    (1.742, 21.862)
    (2.659, 30.518)
    (3.353, 37.367)
    (4.200, 47.657)
};

% ---------------- FSS ----------------
\addplot[mark=+,purple,thick] coordinates{
    (0.652, 10)
    (1.002, 15)
    (1.241, 20)
    (1.556, 25)
    (1.891, 30)
    (2.197, 35)
    (2.955, 50)
    (3.575, 60)
    (4.342, 70)
};

\legend{RR, ATBOC-Gauss, LUCBBOC, FSS}

\end{axis}

        \end{tikzpicture}
%         \begin{tikzpicture}[scale=1]
%         \begin{axis}[
%     xlabel = {$\log\left(\frac{1}{\mathbb{E}[P_e]}\right)$}, 
%     ylabel = {Empirical stopping time},
%     grid=major,
%     xmin = 0, xmax = 4.5,
%     ymin = 0, ymax = 500,
%     legend style={at={(0,1)}, anchor=north west, nodes={scale=0.7}},
%     width=\textwidth,
%     height=\textwidth,
%     tick label style={font=\tiny}
% ]

% % --- ATBOC ---
% \addplot[mark=x,blue,thick] coordinates{
%     (0.176, 7.766)
%     (0.347, 16.718)
%     (0.840, 50.954)
%     (1.326, 85.717)
%     (1.871, 123.303)
%     (2.686, 176.994)
%     (3.376, 211.028)
% };

% % --- BOCELIM ---
% \addplot[mark=x,black,thick] coordinates{
%     (0.232, 12.0514)
%     (0.464, 35.7099)
%     (0.990, 59.5720)
%     (2.617, 175.1862)
% };

% % --- FSS ---
% \addplot[mark=x,red,thick] coordinates{
%     (0.744, 50)
%     (1.229, 100)
%     (1.732, 150)
%     (2.129, 200)
%     (2.502, 250)
%     (2.882, 300)
%     (3.102, 350)
%     (3.444, 400)
%     (3.863, 450)
% };

% % --- Round Robin ---
% \addplot[mark=x,purple,thick] coordinates{
%     (0.182, 8.092)
%     (0.392, 22.134)
%     (1.112, 81.09)
%     (1.796, 142.81)
%     (2.465, 200.447)
%     (3.477, 276.504)
%     (4.075, 332.143)
% };

% % --- LUCBBOC ---
% \addplot[mark=x,cyan,thick] coordinates{
%     (0.166, 8)
%     (0.438, 20.37)
%     (1.174, 72.456)
%     (1.833, 123.235)
%     (2.590, 174.443)
%     (3.352, 234.367)
%     (3.819, 279.009)
% };

% \legend{
% ATBOC,
% BOCElim,
% FSS,
% Round Robin,
% LUCBBOC
% }

% \end{axis}

%         \end{tikzpicture}
        \caption{Comparison of ATBOC-Gauss and LUCBBOC with RR and FSS for a market segmentation problem with MovieLens Dataset to group $M=6$ users into $K=4$ clusters on observing $d=2$ genre ratings.}
        \label{fig:MovieLens}
        \end{minipage}
\end{figure}

%--------------------------------------------------------------------------------

% Figure \ref{fig:moderate} shows the performance comparison of our proposed algorithms with the BOC algorithm in \cite{yang2024optimal}, which is designed specifically for the scenario where the arms in the same clusters share the same mean. 
% Figure \ref{fig:moderate} shows the performance comparison between ATBOC-SubGauss and BOC \cite{yang2024optimal}. BOC algorithm is designed specifically for the scenario where the arms in the same clusters share the same mean. We observe that ATBOC-SubGauss shows better performance than BOC for this problem setup. In BOC in \cite{yang2024optimal}, the recommendation rule is $K$-Means algorithm, whereas in our algorithm we use SLINK. We modified only the recommendation rule in BOC from $K$-Means to SLINK and we call as BOC-SLINK. We observe than SLINK clustering recommendation gives better performance for this problem setup compared to $K$-Means.

.% We observe that our proposed algorithm, ATBOC-SubGauss, performs better than BOC. LUCBBOC shows better performance in the high error probability regime.

{\em Movie Lens Dataset: }
We consider a market segmentation problem with the MovieLens dataset which consists of ratings of users for different movies/genres. For simulation, we consider $M=6$ users and group these users into $K=4$ clusters. To group users, we observe their ratings for the genres Comedy and Drama. Hence, at each round, the algorithm observes $d=2$ dimensional samples.
The average ratings given by the users for the comedy and drama genres are shown in Fig.~\ref{fig: movielens means}.
From Fig.~\ref{fig:MovieLens}, we observe that the ATBOC-Gauss shows the best performance, followed by the LUCBBOC. As expected, both ATBOC-Gauss and LUCBBOC perform better than RR and FSS.

\section{Conclusion and future directions} \label{sec:Conclusion}

We propose the BOC algorithms ATBOC-Gauss, ATBOC-subGauss, ATBOC-1pExp, LUCBBOC, and BOC-ELIM to cluster a collection of data sequences following either multivariate sub-Gaussian distributions or single-parameter exponential family distributions, under a more general clustering setup that allows for multi-dimensional samples, multiple clusters ($K>2$) and the non-identical distribution of arms within a cluster. 
The proposed algorithms are proved to be $\delta$-PC. We prove that the ATBOC-Gauss is asymptotically order-optimal and ATBOC-1pExp is asymptotic optimal, and we validated our asymptotic optimality guarantees through simulations on synthetic datasets. 
We derive the theoretical order of computational complexity for the proposed algorithms and, through simulations, show that LUCBBOC and BOC-ELIM are computationally more efficient. Moreover, we demonstrate that, while being computationally efficient, LUCBBOC and BOC-ELIM exhibit performance comparable to the variants of the ATBOC algorithm. Through simulations on both synthetic and real-world datasets, we show that the proposed algorithms outperform other algorithms in the literature in our more general clustering setup.
Our proposed algorithms can be applied for problems such as market segmentation, pattern recognition, and traffic monitoring, and we show their performance through simulation on real-world datasets. 
Future directions include extending the framework to scenarios where the number of clusters $K$ is unknown and where the arms generate samples from a non-parametric family of distributions.

\bibliographystyle{IEEEtran}
\bibliography{bibliofile}
% \clearpage
%  \begin{center}
%     \textbf{\huge  Technical Appendix}
% \end{center}    
\appendices

When the context $\boldsymbol{\mu}$ and $\pi$ is clear, we represent $\mathbb{E}_{\boldsymbol{\mu}}^{\pi}[\cdot]$ and $\mathbb{P}_{\boldsymbol{\mu}}^{\pi}[\cdot]$ informally as $\mathbb{E}[\cdot]$ and $\mathbb{P}[\cdot]$ respectively. Here $\|\boldsymbol{x}\|$ represents the euclidean vector norm when $\boldsymbol{x}$ is a vector and represents the frobenius norm when $\boldsymbol{x}$ is a matrix.

\section{Proof of Theorem \ref{Theorem:lowerbound}} \label{appsec: lowerbound}
\begin{proof}
For our problem instance $\boldsymbol{\mu}$, consider an arbitrary instance $\boldsymbol{\lambda} \in \text{Alt}(\boldsymbol{\mu})$.  By Lemma 1 (Transportation inequality) of \cite{kaufmann2016complexity}, for any $\delta-$PC policy $\pi$, we have 
\begin{equation} \label{eqn:transportaion_lemma}
    \sum_{m=1}^M \mathbb{E}\left[ N_m(\tau_\delta) \right]{d_{\text{KL}}\left(\boldsymbol{\mu}_m, \boldsymbol{\lambda}_m\right)} \geq (1-2\delta)\log\left(\frac{1-\delta}{\delta}\right),
\end{equation}
where $N_m(\tau_\delta)$ denotes the number of times the arm $m$ has been sampled until the time of stopping.
Since \eqref{eqn:transportaion_lemma} holds for all $\boldsymbol{\lambda} \in \text{Alt}(\boldsymbol{\mu})$, we have 
\begin{align}
    \inf_{\boldsymbol{\lambda} \in \text{Alt}(\boldsymbol{\mu})} \sum_{m=1}^M \mathbb{E}\left[ N_m(\tau_\delta) \right]{d_{\text{KL}}\left(\boldsymbol{\mu}_m, \boldsymbol{\lambda}_m\right)} &\geq (1-2\delta)\log\left(\frac{1-\delta}{\delta}\right).\\
% \end{equation*}
% Therefore,
% \begin{align}
    \iff \mathbb{E}\left[\tau_\delta \right]  \inf_{\boldsymbol{\lambda} \in \text{Alt}(\boldsymbol{\mu})} \sum_{m=1}^M 
    \frac{\mathbb{E}\left[ N_m(\tau_\delta) \right]}{\mathbb{E}\left[\tau_\delta\right]}{d_{\text{KL}}\left(\boldsymbol{\mu}_m, \boldsymbol{\lambda}_m\right)} 
    &\geq (1-2\delta)\log\left(\frac{1-\delta}{\delta}\right).
\end{align}

    Since $\left[ \mathbb{E}\left[ N_1(\tau_\delta) \right], \dots, \mathbb{E}\left[ N_M(\tau_\delta) \right] \right]^T/\mathbb{E}\left[ \tau_\delta \right]$ forms a probability distribution in $\mathcal{P}_M$, we have
    \begin{align}
    \mathbb{E}\left[\tau_\delta \right] \sup_{w \in \mathcal{P}_M} \inf_{\boldsymbol{\lambda} \in \text{Alt}(\boldsymbol{\mu})} 
    \sum_{m=1}^M w_m{d_{\text{KL}}\left(\boldsymbol{\mu}_m, \boldsymbol{\lambda}_m\right)}
    \geq (1-2\delta)\log\left(\frac{1-\delta}{\delta}\right).
\end{align}

\begin{equation}
    \begin{aligned}
       \therefore \mathbb{E}\left[\tau_\delta\right]  \geq (1-2\delta)\log\left(\frac{1-\delta}{\delta}\right) T^*(\boldsymbol{\mu}).
    \end{aligned}
\end{equation}
%Since $\lim_{\delta}\rightarrow 0 \frac{(1-2\delta)\log\left(\frac{1-\delta}{\delta}\right)}{\log\left(\frac{1}{\delta}\right)}=1$, letting $\delta\rightarrow 0$ yields
 Letting $\delta\rightarrow 0$, we obtain $\liminf_{\delta \rightarrow 0} \frac{\mathbb{E}_{\boldsymbol{\mu}}^{\pi}\left[ \tau_\delta(\pi) \right]}{\log\left( \frac{1}{\delta} \right)} \geq T^*(\boldsymbol{\mu}).$
% \begin{equation}
%         \liminf_{\delta \rightarrow 0} \frac{\mathbb{E}_{\boldsymbol{\mu}}^{\pi}\left[ \tau_\delta(\pi) \right]}{\log\left( \frac{1}{\delta} \right)} \geq T^*(\boldsymbol{\mu}).
%     \end{equation}
%Hence proved.
\end{proof}

% continuity of inner min proof -------------------------
\section{Proof of Lemma \ref{proposition:innermincont}} \label{appsec: innermincont}
% \begin{lemma} \label{proposition:innermincont}
%     $\psi(\boldsymbol{w}, \boldsymbol{\mu})$ is continuous in its domain $\mathcal{P}_M \times \mathbb{R}^{d \times M}$.
% \end{lemma} 
\begin{proof}
The proof of Lemma \ref{proposition:innermincont} proceeds analogously to the proof of Proposition 3 of \cite{prabhu2022sequential}. We provide the proof here for the sake of completeness. 
Consider a point $\left( \boldsymbol{w}, \boldsymbol{\mu} \right) \in \mathcal{P}_M\times {\Theta^{M}}$. 
Let $\left(\boldsymbol{w}(n), \boldsymbol{\mu}(n)\right)$ be any arbitrary sequence in $\mathcal{P}_M\times {\Theta^{M}}$ that converges to $\left( \boldsymbol{w}, \boldsymbol{\mu} \right)$.
Now to show that $\psi$ is continuous at $\left( \boldsymbol{w}, \boldsymbol{\mu} \right)$, we need to show that the sequence $\psi\left( \boldsymbol{w}(n), \boldsymbol{\mu}(n) \right)$ converges to $\psi\left( \boldsymbol{w}, \boldsymbol{\mu} \right)$, i.e., $\lim_{n \rightarrow \infty} \psi\left( \boldsymbol{w}(n), \boldsymbol{\mu}(n) \right) = \psi\left( \boldsymbol{w}, \boldsymbol{\mu} \right)$.
Equivalently, we need to show that $\limsup_{n \rightarrow \infty} \psi\left( \boldsymbol{w}(n), \boldsymbol{\mu}(n) \right) \leq \psi\left( \boldsymbol{w}, \boldsymbol{\mu} \right)$ and $\liminf_{n \rightarrow \infty} \psi\left( \boldsymbol{w}(n), \boldsymbol{\mu}(n) \right) \geq \psi\left( \boldsymbol{w}, \boldsymbol{\mu} \right)$.

We define Alt$\left(\boldsymbol{\mu}\right)^C \coloneqq \left\{ \boldsymbol{\lambda} \in {\Theta}^{M}: \mathcal{C}\left(\boldsymbol{\lambda}\right) \sim \mathcal{C}\left(\boldsymbol{\mu}\right) \right\}$. In other words, it is the set of all $\boldsymbol{\lambda} \in {\Theta}^{M}$ such that $d_{INTER}$ calculated with the pair $\left( \boldsymbol{\lambda},\mathcal{C}\left( \boldsymbol{\mu}\right) \right)$ is strictly greater than $d_{INTRA}$ calculated with the pair $\left( \boldsymbol{\lambda},\mathcal{C}\left( \boldsymbol{\mu}\right) \right)$.
By following a procedure similar to that of the initial steps of the proof of Lemma \ref{Lemma:finiteconvex}, it can be verified that Alt$\left(\boldsymbol{\mu}\right)^C$ is the finite intersection of the finite union of the open sets. Hence, Alt$\left(\boldsymbol{\mu}\right)^C$ is an open set. 
Since $\boldsymbol{\mu} \in $ Alt$\left(\boldsymbol{\mu}\right)^C$ and Alt$\left(\boldsymbol{\mu}\right)^C$ is an open set, we can say that there exist $\epsilon > 0$ such that $\mathcal{B}\left( \boldsymbol{\mu}, \epsilon \right)\subset$ Alt$\left(\boldsymbol{\mu}\right)^C$, where $\mathcal{B}\left( \boldsymbol{\mu}, \epsilon \right)\coloneqq \left\{ 
\boldsymbol{\lambda} \in {\Theta}^{M} : \|\boldsymbol{\lambda}-\boldsymbol{\mu}\|<\epsilon \right\}$. 

Since $\boldsymbol{\mu}(n)$ converges to $\boldsymbol{\mu}$,  $\exists N_0 \in \mathbb{N}$ such that for all $n\geq N_0$, we have $\boldsymbol{\mu}(n) \in \mathcal{B}\left( \boldsymbol{\mu}, \epsilon \right)$. 
Hence, for all $n\geq N_0$, $\left( \boldsymbol{w}(n), \boldsymbol{\mu}(n) \right) \subset \mathcal{P}_M\times \text{Alt}\left(\boldsymbol{\mu}\right)^C$.
% \begin{itemize}
    Note that, in $\psi(\boldsymbol{w}, \boldsymbol{\mu})$, infimum is taken over the alternative space $\text{Alt}(\boldsymbol{\mu})$. Then there  exists a $\boldsymbol{\lambda} \in \text{Alt}(\boldsymbol{\mu})$ such that
    \begin{equation} \label{eq:13}
        \sum_{m = 1}^M w_m {d_{\text{KL}}\left(\boldsymbol{\mu}_m, \boldsymbol{\lambda}_m\right)} \leq \psi(\boldsymbol{w}, \boldsymbol{\mu}) + \epsilon.
    \end{equation}
    We have the sequence $\left( \boldsymbol{w}(n), \boldsymbol{\mu}(n) \right)$ converges to a point $\left( \boldsymbol{w}, \boldsymbol{\mu} \right)$ as $n \rightarrow \infty$. Hence, for $\epsilon$, there exists an $N_1$ such that for all $n \geq N_1$ we have the following.
    \begin{equation} \label{eq:14}
        \boldsymbol{w}(n) \leq \boldsymbol{w} + \epsilon\boldsymbol{1}, \text{ where } \boldsymbol{1} \text{ is the vector of all $1$'s and }
    \end{equation}
    \begin{equation} \label{eq:15}
        {d_{\text{KL}}\left(\boldsymbol{\mu}_m(n), \boldsymbol{\lambda}_m\right)} \leq {d_{\text{KL}}\left(\boldsymbol{\mu}_m, \boldsymbol{\lambda}_m\right)} + \epsilon \text{ for all } m \in [M].
    \end{equation}
    Note that for $n \geq N_0$, we have $\boldsymbol{\mu}(n)\in \text{Alt}\left(\boldsymbol{\mu}\right)^C$ and hence $\boldsymbol{\lambda}$ considered above also lies in the alternative space of $\boldsymbol{\mu}(n)$ and hence we have, 
    \begin{equation}
        \psi\left( \boldsymbol{w}(n), \boldsymbol{\mu}(n) \right) \leq \sum_{m = 1}^M w_m(n) {d_{\text{KL}}\left(\boldsymbol{\mu}_m(n), \boldsymbol{\lambda}_m\right)}.
    \end{equation}
    Now using \eqref{eq:13}, \eqref{eq:14} and \eqref{eq:15} in the above equation, we get
    \begin{equation}
    \begin{split}
        \psi\left( \boldsymbol{w}(n), \boldsymbol{\mu}(n) \right) 
        &\leq \sum_{m = 1}^M (w_m + \epsilon) \left[{d_{\text{KL}}\left(\boldsymbol{\mu}_m, \boldsymbol{\lambda}_m\right)} + \epsilon \right] \ \ \text{from \eqref{eq:14} and \eqref{eq:15}} \\ 
        &\leq \sum_{m = 1}^M  w_m{d_{\text{KL}}\left(\boldsymbol{\mu}_m, \boldsymbol{\lambda}_m\right)} + \epsilon\left[ 1 + \sum_{m=1}^M {d_{\text{KL}}\left(\boldsymbol{\mu}_m, \boldsymbol{\lambda}_m\right)} + \epsilon M \right] \\ 
        &\leq \psi(\boldsymbol{w}, \boldsymbol{\mu}) + \epsilon\left[ 2 + \sum_{m=1}^M {d_{\text{KL}}\left(\boldsymbol{\mu}_m, \boldsymbol{\lambda}_m\right)} + \epsilon M \right] \ \ \text{from \eqref{eq:13}}.
    \end{split}
    \end{equation}
    Since $\sum_{m=1}^M {d_{\text{KL}}\left(\boldsymbol{\mu}_m, \boldsymbol{\lambda}_m\right)}$ is finite for any $\boldsymbol{\lambda} \in \text{Alt}(\boldsymbol{\mu})$ and $\epsilon$ is arbitrary, letting $\epsilon$ tend to $0$, we get 
    \begin{equation}
        \psi\left( \boldsymbol{w}(n), \boldsymbol{\mu}(n) \right) \leq \psi(\boldsymbol{w}, \boldsymbol{\mu}) \text{ for } n\geq \max\{N_0, N_1\}.
    \end{equation}
    By taking $\limsup$ in the above equation, we get 
    \begin{equation} \label{eq:17}
        \limsup_{n \rightarrow \infty} \psi\left( \boldsymbol{w}(n), \boldsymbol{\mu}(n) \right) \leq \psi\left( \boldsymbol{w}, \boldsymbol{\mu} \right).
    \end{equation}

    For a given $\epsilon$, for each $n\geq N_0$, there exist $\boldsymbol{\lambda}(n) \in \text{Alt}(\boldsymbol{\mu}(n))$ (Note: for $n\geq N_0$, Alt$\left(\boldsymbol{\mu}(n)\right)$ = Alt$\left(\boldsymbol{\mu}\right)$) such that
    \begin{equation} \label{eq:18}
        \psi\left( \boldsymbol{w}(n), \boldsymbol{\mu}(n) \right) \geq \sum_{m=1}^M w_m(n) {d_{\text{KL}}\left(\boldsymbol{\mu}_m(n), \boldsymbol{\lambda}_m(n)\right)} - \epsilon.
    \end{equation}
    Note that, for $n\geq N_0$, $\boldsymbol{\mu}(n)$ lies in $\epsilon$ neighborhood around $\boldsymbol{\mu}$. Hence, $\boldsymbol{\lambda}(n)$ is a bounded sequence, otherwise, right side of the above equation becomes $\infty$, which is not possible. Also, we know for any bounded sequence there exists a converging subsequence. Hence, without loss of generality, let us consider $\boldsymbol{\lambda}(n)$ converges to some point $\boldsymbol{\lambda} \in \text{Alt}(\boldsymbol{\mu})$. 
    We have the sequence $\left( \boldsymbol{w}(n), \boldsymbol{\mu}(n) \right)$ converges to a point $\left( \boldsymbol{w}, \boldsymbol{\mu} \right)$ as $n \rightarrow \infty$. Hence for $\epsilon$,  $\exists N_1$ such that for all $n \geq N_1$ we have
    \begin{equation} \label{eq:19}
        \boldsymbol{w}(n) \geq \boldsymbol{w} - \epsilon \boldsymbol{1} \text{ and }
    \end{equation}
    \begin{equation} \label{eq:20}
        {d_{\text{KL}}\left(\boldsymbol{\mu}_m(n), \boldsymbol{\lambda}_m(n)\right)} \leq {d_{\text{KL}}\left(\boldsymbol{\mu}_m, \boldsymbol{\lambda}_m(n)\right)} - \epsilon \text{ for all } m \in [M].
    \end{equation}
    Using \eqref{eq:18}, \eqref{eq:19} and \eqref{eq:20}, we get
    \begin{equation}
    \begin{split}
        \psi\left( \boldsymbol{w}(n), \boldsymbol{\mu}(n) \right) 
        &\geq \sum_{m = 1}^M w_m {d_{\text{KL}}\left(\boldsymbol{\mu}_m, \boldsymbol{\lambda}_m(n)\right)} - \epsilon\left[ 2 + \sum_{m=1}^M {d_{\text{KL}}\left(\boldsymbol{\mu}_m, \boldsymbol{\lambda}_m(n)\right)} - \epsilon M \right] \\
        &\geq \psi(\boldsymbol{w}, \boldsymbol{\mu}) - \epsilon\left[ 2 + \sum_{m=1}^M {d_{\text{KL}}\left(\boldsymbol{\mu}_m, \boldsymbol{\lambda}_m(n)\right)} - \epsilon M \right] \ \ \text{($\because \boldsymbol{\lambda}(n) \in \text{Alt}(\boldsymbol{\mu})$)}.
    \end{split}
    \end{equation}
    Since $\sum_{m=1}^M {d_{\text{KL}}\left(\boldsymbol{\mu}_m, \boldsymbol{\lambda}_m(n)\right)}$ is finite for any $\boldsymbol{\lambda} \in \text{Alt}(\boldsymbol{\mu})$ and $\epsilon$ is arbitrary, letting $\epsilon$ tend to $0$, we get 
    \begin{equation}
        \psi\left( \boldsymbol{w}(n), \boldsymbol{\mu}(n) \right) \geq \psi\left( \boldsymbol{w}, \boldsymbol{\mu} \right) \text{ for } n\geq \max\{N_0, N_1\}.
    \end{equation}
    By taking $\liminf$ in the above equation, we get 
    \begin{equation} \label{eq:22}
        \liminf_{n \rightarrow \infty} \psi\left( \boldsymbol{w}(n), \boldsymbol{\mu}(n) \right) \geq \psi\left( \boldsymbol{w}, \boldsymbol{\mu} \right).
    \end{equation}
% \end{itemize}
From \eqref{eq:17} and \eqref{eq:22}, we prove the continuity of $\psi\left(\boldsymbol{w}, \boldsymbol{\mu}\right)$.
\end{proof}          
% ------------------------------------------------------------

\section{Proof of Lemma \ref{prop:armpullpropconverge}} \label{appsec: armpullpropconverge}

To prove Lemma \ref{prop:armpullpropconverge}, we first show that if there exists a sequence $(\boldsymbol{w}(t))_{t \geq 1}$ which converges to a convex, compact and non-empty set $\mathcal{S}$, then using such a sequence $(\boldsymbol{w}(t))_{t \geq 1}$ in the sampling rule of our proposed ATBOC algorithm to track will make the arm pull proportions $\frac{\boldsymbol{N}(t)}{t}$ to converge to the same set $\mathcal{S}$ (Lemma \ref{Lemma1:TrakingC}). Then, by using Lemma \ref{Lemma1:TrakingC} and from the consequence of Berge's Maximum Theorem (Lemma \ref{Lemma2:BergeMaximum}), we prove Lemma \ref{prop:armpullpropconverge}.
\begin{lemma} \label{Lemma1:TrakingC}
    Let $(\boldsymbol{w}(t))_{t \geq 1}$ be a sequence taking values in $\mathcal{P}_M$ and consider a compact, convex, and non-empty subset $\mathcal{S} \subseteq \mathcal{P}_M$. Let's say $\boldsymbol{w}(t)$ and $\mathcal{S}$ has been chosen in such a way that for every $\epsilon > 0$, there exist $t_0(\epsilon) \geq 1$ such that for all $t \geq t_0(\epsilon)$, $d_\infty\left( \boldsymbol{w}(t), \mathcal{S} \right) \leq \epsilon$.\\
    Consider a sampling rule as follows: 
    \begin{equation}
        A_{t+1} = \begin{cases}
            \displaystyle \argmin_{ m \in [M]}N_m(t) &\text{ if } \min_{m \in [M]}N_m(t) < \sqrt{\frac{t}{M}} \\
            b_t &\text{ otherwise }
        \end{cases}
    \end{equation}with, 
    \begin{equation}
        b_t = \argmin_{m \in supp\left( \sum_{s = 1}^t \boldsymbol{w}(s) \right)} \left( N_m(t) - \sum_{s=1}^t w_m(s) \right)
    \end{equation}
    where $N_m(0) = 0$ and for $t \geq 0$, $N_m(t+1) = N_m(t) + \mathds{1}_{\{A_t = m\}}$. 
    Then there exists a $t_1(\epsilon) \geq t_0(\epsilon)$ such that $\forall t \geq t_1(\epsilon)$, $d_\infty\left( \frac{\boldsymbol{N}(t)}{t}, \mathcal{S} \right) \leq (M-1)\epsilon$.
\end{lemma}
\begin{proof}
    We have 
    \begin{equation} \label{trackC_eq1}
         d_\infty\left( w(t), \mathcal{S} \right) \leq \epsilon \quad \forall t \geq t_0.
    \end{equation}
    Although $t_0$ depends on $\epsilon$, we suppress this dependence in the notation and write $t_0$ instead of $t_0(\epsilon)$. This convention is followed throughout the remainder of the proof.
    For all $t\geq 1$, we define $\boldsymbol{\overline{w}}(t) = \frac{1}{t} \sum_{s=1}^t \boldsymbol{w}(t)$. Define $\boldsymbol{\hat{w}}(t) \coloneqq \argmin_{\boldsymbol{w} \in \mathcal{S}} \left\| \boldsymbol{\overline{w}}(t) - \boldsymbol{w} \right\|_{\infty}$. 
    Since $\mathcal{S}$ is nonempty and compact, there exists a minimizer in the set $\mathcal{S}$.\\
    \textbf{Step 1: } We show that there exists a $t_0^{'} \geq t_0$ such that  $\left\| \boldsymbol{\overline{w}}(t)- \boldsymbol{\hat{w}}(t) \right\|_{\infty} \leq 2 \epsilon$ $\forall t \geq t_0^{'}.$
    
     Let $\boldsymbol{v}(t) \coloneqq \argmin_{\boldsymbol{w} \in \mathcal{S}} \left\| \boldsymbol{w} - \boldsymbol{w}(t) \right\|_{\infty}$  $\forall t\geq 1$,.
    Now, for all $m \in [M]$, 
    \begin{equation}
        \begin{aligned}
            &\left| \frac{1}{t}\sum_{s=1}^t w_m(s) - \frac{1}{t}\sum_{s=1}^t v_m(s) \right| \leq \frac{1}{t} \sum_{s=1}^{t_0} \left| w_m(s)-v_m(s) \right| + \frac{1}{t} \sum_{s = t_0+1}^t \left| w_m(s) - v_m(s) \right|.
        \end{aligned}
    \end{equation}
    Note that $\left| w_m(s) - v_m(s) \right|\leq 1$. From \eqref{trackC_eq1}, for $t\geq t_0$, we have $\left| w_m(s) - v_m(s) \right| < \epsilon$. Hence, we get
    \begin{equation}
        \left| \frac{1}{t}\sum_{s=1}^t w_m(s) - \frac{1}{t}\sum_{s=1}^t v_m(s) \right| \leq \frac{t_0}{t} + \frac{t-t_0}{t}\epsilon.
    \end{equation}
    Let $t_0^{'} = \frac{t_0}{\epsilon}$. For $t\geq t_0^{'}$, we get 
    \begin{equation} \label{eq:2epbound}
        \begin{aligned}
            \left| \frac{1}{t}\sum_{s=1}^t w_m(s) - \frac{1}{t}\sum_{s=1}^t v_m(s) \right| \leq 2\epsilon.
        \end{aligned}
    \end{equation}
    Since $\boldsymbol{v}(t) \in \mathcal{S}$ and $\mathcal{S}$ is a convex set, the convex combination $\sum_{s =1}^t \frac{1}{t}\boldsymbol{v}(s) \in \mathcal{S}$. Hence, for all $t\geq t_0^{'}$, 
    % \begin{equation} \label{trackC:eq2}
        \begin{align*}%\label{trackC:eq2}
            \left\| \boldsymbol{\overline{w}}(t) - \boldsymbol{\hat{w}}(t) \right\|_{\infty} &= \min_{\boldsymbol{w} \in \mathcal{S}} \left\| \boldsymbol{\overline{w}}(t) - \boldsymbol{w} \right\|_{\infty} 
            \leq \left\| \boldsymbol{\overline{w}}(t) - \frac{1}{t}\sum_{s =1}^t \boldsymbol{v}(s) \right\|_{\infty} \\
            &=\max_{m \in [M]} \left| \frac{1}{t}\sum_{s=1}^t w_m(s) - \frac{1}{t}\sum_{s=1}^t v_m(s) \right| \leq 2\epsilon.  \text{  (from \eqref{eq:2epbound})}
        \end{align*}
    % \end{equation}
    \textbf{Step 2: } We show that $\exists t_0^{''}$ such that $\forall t \geq t_0^{''}$, $\left\{ A_{t+1} = m \right\} \subseteq \left\{ \varepsilon_{m, t} \leq  \epsilon t \right\}$, where $\varepsilon_{m, t} = N_m(t) - t\boldsymbol{\overline{w}}(t)$. 
    
    Let us define the following two events; 
    \begin{equation}
        \begin{aligned}
            \mathcal{E}_1(t) &= \left\{ m = \argmin_{b \in supp\left( \overline{w}(t) \right)} \left[ N_b(t) - t\overline{w}_b(t) \right] \right\} \\
            \mathcal{E}_2(t) &= \left\{ m = \argmin_{m \in [M]} N_m(t) \text{ and } N_m(t) \leq \sqrt{\frac{t}{M}} \right\}.
        \end{aligned}
    \end{equation}
    Arm $m$ will be pulled at time $t+1$ if one of the above events occurs. Hence, we have $\{ A_{t+1} = m \} = \mathcal{E}_1(t) \cup \mathcal{E}_2(t)$.\\
    {\em case 1: } If $\mathcal{E}_1(t)$ holds, then
    \begin{equation}
    \begin{aligned}
        \varepsilon_{m, t} &= N_m(t) - t\overline{w}_m(t) = \min_{b \in supp\left( \boldsymbol{\overline{w}}(t) \right)} \left[ N_b(t) - \overline{w}_b(t) \right]. 
    \end{aligned}
    \end{equation}
    Using the fact that $\sum_{b \in [M]} \left[ N_b(t) - t \overline{w}_b(t) \right] = 0$, we can show that $\min_{b \in supp\left( \boldsymbol{\overline{w}}(t) \right)} \left[ N_b(t) - \overline{w}_b(t) \right]\leq 0$.
    Hence, under the event $\mathcal{E}_1(t)$, we get $\varepsilon_{m, t} \leq 0$. That is $\mathcal{E}_1(t) \subseteq \{ \varepsilon_{m, t} \leq 0 \} \subseteq \{ \varepsilon_{m, t} \leq t\epsilon \}$.\\
    {\em case 2: } If $\mathcal{E}_2(t)$ holds, then
  $$\varepsilon_{m, t} = N_m(t) - t\overline{w}_m(t) \leq N_m(t) \leq \sqrt{\frac{t}{M}}  = t \sqrt{\frac{1}{tM}}.$$
    Consider $t_0^{''} = \max\left\{ \frac{1}{M\epsilon^2}, t_0^{'} \right\}$.  For $t \geq t_0^{''}$, we have $\varepsilon_{m, t} \leq \epsilon t$.
    Hence,  we get $\mathcal{E}_2(t) \subseteq \left\{ \varepsilon_{m, t} \leq \epsilon t \right\}$  $\forall t\geq t_0^{''}$.\\
    Therefore, from both cases, we have
    \begin{equation} \label{trackC:eq14}
        \left\{ A_{t+1} = m \right\} \subseteq \left\{ \varepsilon_{m, t} \leq \epsilon t \right\} \quad \forall t\geq t_0^{''}.
    \end{equation}
    \noindent
    \textbf{Step 3: } We show that for all $m \in [M]$, $\varepsilon_{m, t} \leq \max\left\{ \varepsilon_{m, t_0^{''}}, , 1 + t\epsilon \right\}$, $\forall t \geq t_0^{''}$. 
    
    First, we upper bound $\varepsilon_{m, t+1}$ using $\varepsilon_{m, t}$.
    %\begin{equation} 
    \begin{align}\label{trackC:eq16}
        \varepsilon_{m, t+1} &= N_{m}(t+1) - (t+1)\overline{w}_m(t+1) \nonumber
        \leq N_m(t) + \mathds{1}_{\{ A_{t+1} = m \}}-t\overline{w}_m(t) \nonumber\\
        &= \varepsilon_{m, t} + \mathds{1}_{\{ A_{t+1} = m \}} 
        \leq \varepsilon_{m, t} + \mathds{1}_{\{ \varepsilon_{m, t} \leq t\epsilon \}} \quad  \text{(from \eqref{trackC:eq14})}.
    \end{align}
    %\end{equation}
    
    Now, we prove step 3 by induction.
    For $t=t_0^{''}$, step 3 holds trivially as $\varepsilon_{m, t_0^{''}} \leq \max\left\{ \varepsilon_{m, t_0^{''}}, , 1 + t\epsilon \right\}$ for all $m \in [M]$.
    Assume that Step 3 holds for some $t>t_0^{''}$. Therefore, we have 
    \begin{equation} \label{trackC:eq15}
        \varepsilon_{m, t} \leq \max\left\{ \varepsilon_{m, t_0^{''}}, , 1 + t\epsilon \right\} \text{ for all } m \in [M].
    \end{equation}
    Now, we prove that step 3 holds for $t+1$.\\
    {\em case 1: } Let $\varepsilon_{m, t} \leq t\epsilon$.
    \begin{equation} \label{trackC:eq17}
        \begin{aligned}
            \varepsilon_{m, t+1} &\leq \varepsilon_{m, t} + \mathds{1}_{\{ \varepsilon_{m, t} \leq t\epsilon \}} \ \ \text{(from \eqref{trackC:eq16})}\\
            &\leq t\epsilon + 1 \ \ \text{( $\because \varepsilon_{m, t} \leq t\epsilon$ )} \leq \max\left\{ \varepsilon_{m, t_0^{''}}, (t+1)\epsilon + 1\right\}.
        \end{aligned}
    \end{equation}
    {\em case 2: } Let $\varepsilon_{m, t} > t\epsilon$.
    \begin{equation} \label{trackC:eq18}
        \begin{aligned}
            \varepsilon_{m, t+1} &\leq \varepsilon_{m, t} + 0 \ \ \text{( $\because \varepsilon_{m, t} > t\epsilon$ )} 
            \leq \max\left\{ \varepsilon_{m, t_0^{''}},  1 + t\epsilon \right\} \ \ \text{(from \eqref{trackC:eq15})} \\
            &\leq \max\left\{ \varepsilon_{m, t_0^{''}},  1 + (t+1)\epsilon \right\}.
        \end{aligned}
    \end{equation}
    Hence, step 3 is true for $t+1$. 
    This proves step 3.\\
    \textbf{Step 4: } We show that  $d_\infty\left( 
    \frac{\boldsymbol{N}(t)}{t}, \mathcal{S} \right) \leq 2M\epsilon$, $\forall t\geq t_1$.\\
    We upper bound $\varepsilon_{m, t_0^{''}}$ as $\varepsilon_{m, t_0^{''}} \leq N_m\left( t_0^{''} \right) - t_0^{''}\hat{w}_m\left( t_0^{''} \right) \leq  N_m\left( t_0^{''} \right) \leq t_0^{''}.$
    \\From step 3 and from the above equation, we have $\text{for all } m \in [M],$ %we have the upper bound for $\varepsilon_{m, t}$ as, 
    \begin{equation} \label{trackC:eq19}
         \varepsilon_{m, t} \leq \max\left\{ t_0^{''}, 1 + t\epsilon \right\}, \forall t \geq t_0^{''}.
    \end{equation}
    Since $\varepsilon_{m^{'}, t} = N_{m^{'}}(t) - t\hat{w}_{m^{'}}(t)$, we have $\sum_{m^{'} = 1}^M \varepsilon_{m^{'}, t} = 0$. Hence, we lower bound $\varepsilon_{m, t}$ for all $m \in [M]$ as 
    % We lower bound the $\varepsilon_{m, t}$ for all $m \in [M]$ as follows: We have $\sum_{m^{'} = 1}^M \varepsilon_{m^{'}, t} = 0 \ \ \text{($\because \varepsilon_{m^{'}, t} = N_{m^{'}}(t) - t\hat{w}_{m^{'}}(t)$)}$. Hence, we get
    %\begin{equation} 
        \begin{align} \label{trackC:eq20}
            \varepsilon_{m, t} &= - \sum_{m^{'} \neq m} \varepsilon_{m^{'}, t} 
            \geq -(M-1)\max\left\{ t_0^{''}, 1 + t\epsilon \right\}. \ \ \text{(from \eqref{trackC:eq19})}
        \end{align}
    %\end{equation}
    From \eqref{trackC:eq19} and \eqref{trackC:eq20},  for all $m \in [M]$, we upper bound $|\varepsilon_{m, t}|$ as, 
    \begin{equation} \label{trackC:eq21}
        \begin{aligned}
            |\varepsilon_{m, t}| &\leq (M-1)\max\left\{ t_0^{''}, 1 + t\epsilon \right\}, \forall t \geq t_0^{''}. 
        \end{aligned}
    \end{equation}
    Finally, we upper bound $d_\infty\left(\frac{\boldsymbol{N}(t)}{t}, \mathcal{S} \right)$ as follows:
   
    \begin{align*}
        d_\infty\left(\frac{\boldsymbol{N}(t)}{t}, \mathcal{S} \right) 
        &\leq \left\|\frac{\boldsymbol{N}(t)}{t}- \boldsymbol{\overline{w}}(t) \right\|_{\infty} + d_{\infty}\left( \boldsymbol{\overline{w}}(t), \mathcal{S} \right) \\%\text{ (Triangle Inequality)}\\
        &\leq \left\|\frac{\boldsymbol{N}(t)}{t}- \boldsymbol{\overline{w}}(t) \right\|_{\infty} + \left\| \boldsymbol{\overline{w}}(t)- \boldsymbol{\hat{w}}(t) \right\|_{\infty} \text{ ($\because \boldsymbol{\hat{w}}(t) \in \mathcal{S}$)} \\ 
        &\leq \max_{m \in [M]} \left| \frac{N_m(t)}{t} - \overline{w}_m(t) \right| + 2 \epsilon \ \ \text{(from step 1)}\\
        &= \max_{m \in [M]} \left| \frac{\varepsilon_{m, t}}{t} \right| + 2\epsilon
        \leq (M-1)\max\left\{ \frac{t_0^{''}}{t}, \frac{1}{t}+\epsilon \right\} + 2\epsilon \text{ (from \eqref{trackC:eq21})}.
    \end{align*}
    Let $t_1 = \max\left\{ \frac{1}{\epsilon}, \frac{t_0^{''}}{2\epsilon} \right\}$. For all $t \geq t_1$, we get $d_\infty\left(\frac{\boldsymbol{N}(t)}{t}, \mathcal{S} \right) \leq 2M\epsilon$,
    where the expression for $t_1$ in terms of $t_0$ is given as $t_1 = \max\left\{ \frac{1}{\epsilon}, \frac{1}{M\epsilon^3}, \frac{t_0}{\epsilon^2} \right\}$. This proves Lemma \ref{Lemma1:TrakingC}.
\end{proof}

\begin{lemma} \label{Lemma2:BergeMaximum}
    Let $\boldsymbol{\mu} \in \Theta^{M}$. Define $\psi^*(\boldsymbol{\mu}) = \max_{w \in \mathcal{P}_M} \psi(\boldsymbol{w}, \boldsymbol{\mu})$ and $\mathcal{S}^*(\boldsymbol{\mu}) = \argmax_{w \in \mathcal{P}_M}\psi(\boldsymbol{w}, \boldsymbol{\mu})$. Then $\psi^*(\boldsymbol{\mu})$ is continuous at $\boldsymbol{\mu}$, and $\mathcal{S}^*(\boldsymbol{\mu})$ is convex, compact and non-empty. Furthermore, $\mathcal{S}^*(\boldsymbol{\mu})$ is upper hemi continuous, i.e., for any open neighborhood $\nu$ of $\mathcal{S}^*(\boldsymbol{\mu})$, there exists an open neighborhood $\mathcal{U}$ of $\boldsymbol{\mu}$, such that for all $\boldsymbol{\mu}^{'} \in \mathcal{U}$, we have $\mathcal{S}^*(\boldsymbol{\mu}^{'}) \subseteq \nu$.
\end{lemma}
\begin{proof}
    From Lemma \ref{proposition:innermincont}, we have the function $\psi(\cdot, \cdot)$ is continuous in its domain. Also the probability simplex $\mathcal{P}_M$ is a compact set. Hence, Lemma \ref{Lemma2:BergeMaximum} follows from Berge's Maximum Theorem in \cite{berge1877topological, sundaram1996first}.
\end{proof}

\begin{corollary} \label{corollary: berge}
    Let $\boldsymbol{\mu} \in \Theta^{M}$. Define $\psi_{\text{mod}}^*(\boldsymbol{\mu}) = \max_{w \in \mathcal{P}_M} \psi_{\text{mod}}(\boldsymbol{w}, \boldsymbol{\mu})$ and $\mathcal{S}_{\text{mod}}^*(\boldsymbol{\mu}) = \argmax_{w \in \mathcal{P}_M}\psi_{\text{mod}}(\boldsymbol{w}, \boldsymbol{\mu})$. Then $\psi_{\text{mod}}^*(\boldsymbol{\mu})$ is continuous at $\boldsymbol{\mu}$, and $\mathcal{S}_{\text{mod}}^*(\boldsymbol{\mu})$ is convex, compact and non-empty. Furthermore, $\mathcal{S}_{\text{mod}}^*(\boldsymbol{\mu})$ is upper hemi continuous, i.e., for any open neighborhood $\nu$ of $\mathcal{S}_{\text{mod}}^*(\boldsymbol{\mu})$, there exists an open neighborhood $\mathcal{U}$ of $\boldsymbol{\mu}$, such that for all $\boldsymbol{\mu}^{'} \in \mathcal{U}$, we have $\mathcal{S}^*(\boldsymbol{\mu}^{'}) \subseteq \nu$.
\end{corollary}
\begin{proof}
    The modified inner infimum $\psi_{\text{mod}}(\cdot, \cdot)$ is the special case of the inner infimum $\psi(\cdot, \cdot)$, where the KL divergence in $\psi(\cdot, \cdot)$ is considered for the Gaussian distribution. Hence, this result follows from Lemma \ref{Lemma2:BergeMaximum}.
\end{proof}

Now, we discuss proof of Lemma \ref{prop:armpullpropconverge}.
\begin{proof}[Proof of Lemma \ref{prop:armpullpropconverge}]
From Lemma \ref{Lemma2:BergeMaximum}, $\mathcal{S}^*(\boldsymbol{\mu})$ is upper hemi continuous and hence, for all $\epsilon > 0$,  there exists $\zeta(\epsilon)>0$ such that for all $\boldsymbol{\mu}^{'} \in \mathbb{R}^{d\times M}$, $\text{if } \| \boldsymbol{\mu} - \boldsymbol{\mu}^{'} \| \leq \zeta(\epsilon), 
        \text{ then, }$
\begin{equation} \label{eq:usca}
    \begin{aligned}
         \max_{\boldsymbol{w^{''}} \in \mathcal{S}^*(\boldsymbol{\mu}^{'})} d_\infty(\boldsymbol{w^{''}}, \mathcal{S}^*(\boldsymbol{\mu})) \leq \epsilon.
    \end{aligned}
\end{equation}
At any time $t$, forced exploration in the ATBOC-Gauss/ATBOC-1pExp algorithms ensures that each arm is sampled at least by order of $\sqrt{t}$. Hence, by the strong law of large numbers, we have, $\lim_{t \rightarrow \infty} \hat{\boldsymbol{\mu}}(t) = \boldsymbol{\mu} \ \ a.s.$
Hence, there exists a $t_0>0$ such that for all $t \geq t_0$, we have, $\| \boldsymbol{\mu} - \hat{\boldsymbol{\mu}}(t) \| \leq \zeta(\epsilon)$. 
Hence by using \eqref{eq:usca}, $\forall t\geq t_0(\epsilon)$,
\begin{equation} \label{eq:boundepsilon}
\begin{aligned}
    d_\infty\left( \boldsymbol{w}^{*}(t), \mathcal{S}^*(\boldsymbol{\mu}) \right) &\leq \max_{\boldsymbol{w} \in \mathcal{S}^*(\hat{\boldsymbol{\mu}}(t))} d_\infty(\boldsymbol{w}, \mathcal{S}^*(\boldsymbol{\mu})) 
    \ \ \text{($\because \boldsymbol{w}^{*}(t) \in \mathcal{S}^*(\hat{\boldsymbol{\mu}}(t)$))} 
    \leq \epsilon \ \ (\text{from \ref{eq:usca}}).
\end{aligned}
\end{equation}
Hence, we have shown that, as $t\rightarrow \infty$, $d_\infty\left( \boldsymbol{w}^{*}(t), \mathcal{S}^*(\boldsymbol{\mu}) \right) \rightarrow 0 \ a.s.$, where $w^{*}(t)$ is the tracking sequence used in the ATBOC-Gauss/ATBOC-1pExp algorithm.  From Lemma \ref{Lemma2:BergeMaximum}, $\mathcal{S}^*(\boldsymbol{\mu})$ is non-empty, compact and convex set. By applying Lemma \ref{Lemma1:TrakingC} on \eqref{eq:boundepsilon}, we get $d_\infty\left( \frac{\boldsymbol{N}(t)}{t}, \mathcal{S}^*(\boldsymbol{\mu}) \right) \rightarrow 0 \ a.s. \text{, as } t \rightarrow \infty.$

{For ATBOC-subGauss algorithm, we follow a procedure similar to the one described above, except that we use Corollary~\ref{corollary: berge} in place of Lemma~\ref{Lemma2:BergeMaximum}. This yields $d_\infty\left( \frac{\boldsymbol{N}(t)}{t}, \mathcal{S}_{\text{mod}}^*(\boldsymbol{\mu}) \right) \rightarrow 0 \ a.s. \text{, as } t \rightarrow \infty.$
% This completes the proof of Lemma~\ref{prop:armpullpropconverge}.
}
\end{proof}

% -------------------------------------------------------------------

\section{Proof of Theorem \ref{Theorem2:deltaPAC} - $\delta-$PC} \label{appsec: deltapac}
\begin{proof}
    \textbf{Finite Stopping time Proof:} \\
{First, we upper bound the threshold chosen for ATBOC-Gauss/ATBOC-subGauss,  as follows:}
    %\begin{equation} 
        \begin{align} \label{almost:eq4}
            \beta(\delta, t) &= 2 \log\left( \frac{\left( \prod_{m=1}^M(N_m(t)+1) \right)^{\frac{d}{2}}}{\delta} \right) 
            \leq 2 \log\left( \frac{t^{\frac{Md}{2}}}{\delta} \right) \ \ (\because N_m(t)+1 \leq t, \forall t\geq 2).
        \end{align}
    %\end{equation}
    Now, we upper bound the threshold chosen for ATBOC-1pExp,  as follows:
    \begin{align} \label{almost:eq4-1pexp}
            \beta(\delta, t) &= 3\sum_{m=1}^M \log\left(1+\log\left(N_m(t)\right)\right) + (1+\zeta)M\log\left(\frac{\pi^2/3}{\left(\log(1+\zeta)\right)^2}\right) + (1+\zeta)\log\left(\frac{1}{\delta}\right) \nonumber \\
            &\leq 3\sum_{m=1}^M \log(t) + (1+\zeta)\log\left(\frac{(\pi^2/3)^M}{\left(\log(1+\zeta)\right)^{2M}}\right) + (1+\zeta)\log\left(\frac{1}{\delta}\right) \nonumber \\
            &= (1+\zeta) \log\left( \frac{t^{\frac{3M}{1+\zeta}}}{\delta} \left(\frac{\pi^2/3}{\left(\log(1+\zeta)\right)^2}\right)^{M} \right).
        \end{align}
    From \eqref{almost:eq4} and \eqref{almost:eq4-1pexp}, the upper bound on the threshold can be written more generally as
    \begin{equation} \label{eq:thresholdub}
        \beta(\delta, t) \leq a \log\left(\frac{bt^c}{\delta}\right),
    \end{equation}
    where $(a, b, c) = \begin{cases}
        \left(2, 1, \frac{Md}{2}\right) &\text{for ATBOC-SubGauss}\\
        \left(1+\zeta, \left(\frac{\pi^2/3}{\left(\log(1+\zeta)\right)^2}\right)^{M}, \frac{3M}{1+\zeta}\right) &\text{for ATBOC-1pExp}
    \end{cases}.$

{\em ATBOC-Gauss/ATBOC-1pExp:} 
        Let $\mathcal{E}$ be the event defined as
    \begin{equation} \label{fine stopping: event}
    \begin{aligned}
        \mathcal{E} = &\left\{ \lim_{t \rightarrow \infty} d_\infty\left( \frac{\boldsymbol{N}(t)}{t}, \mathcal{S}^*(\boldsymbol{\mu}) \right) = 0  \ \
        \& \lim_{t \rightarrow \infty } \hat{\boldsymbol{\mu}}(t) = \boldsymbol{\mu}\right\}.
    \end{aligned}
    \end{equation}
    From Lemma \ref{prop:armpullpropconverge} and the strong law of large numbers, we have $\mathbb{P}\left[\mathcal{E} \right] = 1$. 
    Consider $\epsilon > 0$. By the continuity of $\psi$ (Lemma \ref{proposition:innermincont}), there exists an open neighborhood $\nu(\epsilon)$ of $\{\boldsymbol{\mu}\}\times \mathcal{S}^*(\boldsymbol{\mu})$ such that for all $(\boldsymbol{\mu^{'}}, \boldsymbol{w^{'}}) \in \nu(\epsilon)$, we have 
    \begin{equation} \label{eq:almostconti}
        \psi(\boldsymbol{w^{'}}, \boldsymbol{\mu}^{'}) \geq (1-\epsilon)\psi(\boldsymbol{w^{*}}, \boldsymbol{\mu}) \text{, where, } \boldsymbol{w^*} \in \mathcal{S}^*(\boldsymbol{\mu}).
    \end{equation}
    Under the event $\mathcal{E}$, there exists $t_0\geq 1$ such that $\left(\hat{\boldsymbol{\mu}}(t), \frac{ \boldsymbol{N}(t)}{t} \right) \in \nu(\epsilon)$ $\forall t \geq t_0$. Hence, using \eqref{eq:almostconti}, we get $\psi\left( \frac{\boldsymbol{N}(t)}{t}, \hat{\boldsymbol{\mu}}(t) \right) \geq (1-\epsilon)\psi(\boldsymbol{w^{*}}, \boldsymbol{\mu})$ $\forall t \geq t_0$.
    Hence, for all $t \geq t_0$, the test statistics $Z(t)$ can be lower bounded as
    \begin{equation} \label{almost:eq3}
        \begin{aligned}
            Z(t) = t \psi\left( \frac{\boldsymbol{N}(t)}{t}, \hat{\boldsymbol{\mu}}(t) \right) 
            \geq t(1-\epsilon)\psi(\boldsymbol{w^{*}}, \boldsymbol{\mu}).
        \end{aligned}
    \end{equation}
    Now, we upper bound the sample complexity $\tau_\delta$ as follows.
    \begin{equation} \label{eq: finite stopping time}
        \begin{aligned}
            \tau_\delta &= \inf\left\{ t \in \mathbb{N} : Z(t) \geq \beta(\delta, t) \right\}  \\
            &\leq \max\Bigg\{ t_0, \inf\bigg\{ t \in \mathbb{N} :t(1-\epsilon)\psi(\boldsymbol{w^{*}}, \boldsymbol{\mu}) >a \log\left( \frac{bt^c}{\delta} \right) \bigg\} \Bigg\} \ \ \text{(from \eqref{almost:eq3} and \eqref{eq:thresholdub})}.\\
            &\leq \max\Bigg\{ t_0, \inf\bigg\{ t \in \mathbb{N} : t(1-\epsilon)\psi(\boldsymbol{w^{*}},\boldsymbol{\mu}) >a \log\left( \frac{bt^c}{\delta} \right) \bigg\} \Bigg\} \\
            &= \max\Bigg\{ t_0, \inf\left\{ t \in \mathbb{N} : t\frac{1-\epsilon}{T^*(\boldsymbol{\mu})ac} >  \log\left( \left(\frac{b}{\delta}\right)^{\frac{1}{c}}t \right) \right\} \Bigg\} \ \ \text{($\because \psi\left(\boldsymbol{w^*}, \boldsymbol{\mu}\right) = T^*(\boldsymbol{\mu})^{-1}$)} \\
            &\leq \max\Bigg\{t_0, \frac{ac T^*(\boldsymbol{\mu})}{1-\epsilon} \left( \log\left( \frac{ab^{\frac{1}{c}}cT^*(\boldsymbol{\mu})e}{\delta^{\frac{1}{c}} (1-\epsilon)} \right) + \log\log\left( \frac{ab^{\frac{1}{c}}cT^*(\boldsymbol{\mu})}{\delta^{\frac{1}{c}} (1-\epsilon)} \right) \right) \Bigg\} \text{ (Lemma \ref{Lemma:infbound} in Section \ref{appsec:otherclaims})}\\
            &\leq \max\Bigg\{ t_0, \frac{a T^*(\boldsymbol{\mu})}{1-\epsilon} \log\left( \frac{1}{\delta} \right) + \frac{ac T^*(\boldsymbol{\mu})}{1-\epsilon} \left( \log\left( \frac{ab^{\frac{1}{c}}cT^*(\boldsymbol{\mu})e}{(1-\epsilon)} \right) + \log\log\left( \frac{ab^{\frac{1}{c}}cT^*(\boldsymbol{\mu})}{\delta^{\frac{1}{c}} (1-\epsilon)} \right) \right) \Bigg\}.
        \end{aligned}
    \end{equation}
    % The sample complexity is further upper bounded as, 
    % \begin{equation} \label{eq: finite stopping time}
    %     \begin{aligned}
    %         \tau_\delta &\leq \max\Bigg\{ t_0, \inf\bigg\{ t \in \mathbb{N} : t(1-\epsilon)\psi(\boldsymbol{w^{*}},\boldsymbol{\mu}) >a \log\left( \frac{bt^c}{\delta} \right) \bigg\} \Bigg\} \\
    %         &= \max\Bigg\{ t_0, \inf\left\{ t \in \mathbb{N} : t\frac{1-\epsilon}{T^*(\boldsymbol{\mu})ac} >  \log\left( \left(\frac{b}{\delta}\right)^{\frac{1}{c}}t \right) \right\} \Bigg\} \ \ \text{($\because \psi\left(\boldsymbol{w^*}, \boldsymbol{\mu}\right) = T^*(\boldsymbol{\mu})^{-1}$)} \\
    %         &\leq \max\Bigg\{t_0, \frac{ac T^*(\boldsymbol{\mu})}{1-\epsilon} \left( \log\left( \frac{b^{\frac{1}{c}}aceT^*(\boldsymbol{\mu})}{\delta^{\frac{1}{c}} (1-\epsilon)} \right) + \log\log\left( \frac{b^{\frac{1}{c}}acT^*(\boldsymbol{\mu})}{\delta^{\frac{1}{c}} (1-\epsilon)} \right) \right) \Bigg\} \text{ (Lemma \ref{Lemma:infbound} in Section \ref{appsec:otherclaims})}\\
    %         &\leq \max\Bigg\{ t_0, \frac{a T^*(\boldsymbol{\mu})}{1-\epsilon} \log\left( \frac{1}{\delta} \right) + \frac{ac T^*(\boldsymbol{\mu})}{1-\epsilon} \left( \log\left( \frac{b^{\frac{1}{c}}aceT^*(\boldsymbol{\mu})}{(1-\epsilon)} \right) + \log\log\left( \frac{b^{\frac{1}{c}}acT^*(\boldsymbol{\mu})}{\delta^{\frac{1}{c}} (1-\epsilon)} \right) \right) \Bigg\}.
    %     \end{aligned}
    % \end{equation}}

    Note that all the terms in the upper bound on $\tau_\delta$ are finite. Hence, under the event $\mathcal{E}$, we have $\tau_\delta < \infty$, 
    i.e., $\mathcal{E} \subseteq \{\tau_\delta < \infty\}$. Since $\mathbb{P}\left[ \mathcal{E} \right] = 1$, we get $\mathbb{P}\left[\{ \tau_\delta < \infty \}\right] =1$. Therefore, ATBOC-Gauss/ATBOC-1pExp stops in finite time almost surely.

    {\em ATBOC-subGauss:}
    We can follow a procedure similar to that described above, except that the lower bound of $Z(t)$ is $t(1-\epsilon)T_{\text{mod}}^{*}(\boldsymbol{\mu})$. Consequently, the sample complexity is upper bounded as 
    \begin{equation} \label{eq: sampcompupbndatbocsubgauss}
        \tau_\delta \leq \max\Bigg\{ t_0, \frac{a T_{\text{mod}}^*(\boldsymbol{\mu})}{1-\epsilon} \log\left( \frac{1}{\delta} \right) + \frac{ac T_{\text{mod}}^*(\boldsymbol{\mu})}{1-\epsilon} \left( \log\left( \frac{ab^{\frac{1}{c}}cT_{\text{mod}}^*(\boldsymbol{\mu})e}{(1-\epsilon)} \right) + \log\log\left( \frac{ab^{\frac{1}{c}}cT_{\text{mod}}^*(\boldsymbol{\mu})}{\delta^{\frac{1}{c}} (1-\epsilon)} \right) \right) \Bigg\}.
    \end{equation}
    Hence, we get $\mathbb{P}\left[\{ \tau_\delta < \infty \}\right] =1$. This proves that ATBOC-subGauss stops in finite time almost surely. 

    \textbf{$\delta-$PC Proof: } \\ 
    We can write the probability of error as:
    \begin{equation} \label{eq: peini}
        \begin{aligned}
        \mathbb{P} \left[ \tau_\delta < \infty \text{ and } \mathcal{C}\left(\hat{\boldsymbol{\mu}}\left(\tau_\delta\right)\right)\nsim \mathcal{C}(\boldsymbol{\mu}) \right] &\leq \mathbb{P}\left[ \exists t \in \mathbb{N} : \left\{ Z(t)>\beta(\delta, t) \text{ and } \mathcal{C}(\hat{\boldsymbol{\mu}}(t))\nsim \mathcal{C}(\boldsymbol{\mu}) \right\} \right] \\
        &\leq  \mathbb{P}\left[ \exists t \in \mathbb{N} : \left\{ Z(t)> \beta(\delta, t) \right\} \mid \{\mathcal{C}(\hat{\boldsymbol{\mu}}(t))\nsim \mathcal{C}(\boldsymbol{\mu})\} \right].
        \end{aligned}
    \end{equation}
    {\em Case 1:} \textit{ATBOC-Gauss/ATBOC-subGauss} \\
    From the expression of the test statistics $Z(t)$ for ATBOC-Gauss/ATBOC-subGauss from \eqref{eq: z} and from the definition of the stopping time $\tau_\delta = \inf\left\{ t \in \mathbb{N}: Z(t) \geq \beta(\delta, t) \right\}$, the probability of error can be further upper-bounded as
    \begin{equation} \label{eq: del1}
        \begin{aligned}
            &\mathbb{P} \left[ \tau_\delta < \infty \text{ and } \mathcal{C}\left(\hat{\boldsymbol{\mu}}\left(\tau_\delta\right)\right)\nsim \mathcal{C}(\boldsymbol{\mu}) \right] \\
            &\leq  \mathbb{P}\Bigg[ \exists t \in \mathbb{N} :\bigg\{ \inf_{\boldsymbol{\lambda} \in \text{Alt}(\hat{\boldsymbol{\mu}}(t))}  \sum_{m=1}^M \frac{N_m(t)}{2\sigma^2} \left\|\boldsymbol{\lambda}_m - \boldsymbol{\hat{\mu}}_m(t)\right\|^2   > \beta(\delta, t) \bigg\}\mid \{\mathcal{C}(\boldsymbol{\hat{\mu}}(t))\nsim \mathcal{C}(\boldsymbol{\mu})\} \Bigg] .
        \end{aligned}
    \end{equation}
    
    Given $\mathcal{C}(\hat{\boldsymbol{\mu}}(t))\nsim \mathcal{C}(\boldsymbol{\mu})$, we can say $\boldsymbol{\mu} \in$ Alt$(\hat{\boldsymbol{\mu}}(t))$. Hence, we have 
    \begin{equation} \label{eq: del2}
    \inf_{\boldsymbol{\lambda} \in \text{Alt}(\hat{\boldsymbol{\mu}}(t))}  \sum_{m=1}^M \frac{N_m(t)}{2\sigma^2} \left\|\boldsymbol{\lambda}_m - \boldsymbol{\hat{\mu}}_m(t)\right\|^2 
    \leq \sum_{m=1}^M \frac{N_m(t)}{2\sigma^2} \left\|\boldsymbol{\mu}_m - \boldsymbol{\hat{\mu}}_m(t)\right\|^2. 
    \end{equation}
    Using this inequality in \eqref{eq: del1}, we get
    \begin{equation} \label{del:eq4}
    \begin{aligned}
        &\mathbb{P} \left[ \tau_\delta < \infty \text{ and } \mathcal{C}\left(\hat{\boldsymbol{\mu}}\left(\tau_\delta\right)\right)\nsim \mathcal{C}(\boldsymbol{\mu}) \right]  \leq  \mathbb{P}\left[ \exists t \in \mathbb{N} : \left\{ \sum_{m=1}^M \frac{N_m(t)}{2\sigma^2} \|\boldsymbol{\mu}_m - \boldsymbol{\hat{\mu}}_m(t)\|^2> \beta(\delta, t) \right\}  \right].
    \end{aligned}
    \end{equation}
    Let us denote the identity matrix of dimension $n \times n$ as $\boldsymbol{I}_n$ and the matrix of all zeros of dimension $n \times n$ as $\boldsymbol{O}_n$. Hence, we can write $N_m(t) \left\|\boldsymbol{\mu}_m - \boldsymbol{\hat{\mu}}_m(t)\right\|^2 = \left(\boldsymbol{\mu}_m - \boldsymbol{\hat{\mu}}_m(t)\right)^T N_m(t) \boldsymbol{I}_d\left(\boldsymbol{\mu}_m - \boldsymbol{\hat{\mu}}_m(t)\right)$.
    % \begin{equation}
    % \begin{split}
    %     &N_m(t) \left\|\boldsymbol{\mu}_m - \boldsymbol{\hat{\mu}}_m(t)\right\|^2 = \left(\boldsymbol{\mu}_m - \boldsymbol{\hat{\mu}}_m(t)\right)^T N_m(t) \boldsymbol{I}_d\left(\boldsymbol{\mu}_m - \boldsymbol{\hat{\mu}}_m(t)\right)
    %     \end{split}
    % \end{equation}
   Let 
    \begin{equation}
    \begin{split}
        \underline{\boldsymbol{\mu}} = \begin{bmatrix}
        \boldsymbol{\mu}_1 \\
        \vdots \\
        \boldsymbol{\mu}_M \\
        \end{bmatrix}, 
        \underline{\boldsymbol{\hat{\mu}}}(t) = \begin{bmatrix}
        \boldsymbol{\hat{\mu}}_1(t) \\
        \vdots \\
        \boldsymbol{\hat{\mu}}_M(t) \\
        \end{bmatrix} 
        \text{ and } 
        \boldsymbol{D}(t) = \begin{bmatrix}
        N_1(t)\boldsymbol{I}_d & \boldsymbol{O}_d  & \cdots & \boldsymbol{O}_d \\
        \boldsymbol{O}_d & \ddots  & \cdots & \boldsymbol{O}_d \\
        \vdots  & \vdots & \ddots & \vdots \\
        \boldsymbol{O}_d & \boldsymbol{O}_d  & \cdots & N_M(t) \boldsymbol{I}_d \\
        \end{bmatrix}
        \end{split}
    \end{equation}
    where $\underline{\boldsymbol{\mu}}$ and $\underline{\boldsymbol{\hat{\mu}}}(t)$ are the column vectors with dimension $Md \times 1$ and $\boldsymbol{D}(t)$ is the diagonal matrix with dimension $Md \times Md$. 
    Now we write 
    \begin{equation} \label{del:eq1}
    \begin{aligned}
        \sum_{m=1}^M &N_m(t) \|\boldsymbol{\mu}_m - \boldsymbol{\hat{\mu}}_m(t)\|^2 =  \left(\underline{\boldsymbol{\mu}} - \underline{\hat{\boldsymbol{\mu}}}(t)\right)^T \boldsymbol{D}(t)\left(\underline{\boldsymbol{\mu}} - \underline{\hat{\boldsymbol{\mu}}}(t)\right).
    \end{aligned}
    \end{equation}
    Let $\boldsymbol{a}(r) = [a_1(r) \ldots a_{Md}(r)]^T$ be a column vector with $a_j(r) = \mathds{1}_{\left\{\left[A_{\left\lceil \frac{r}{d} \right\rceil}-1\right]d + \bmod{(r-1, d)+1} = j\right\}}$. Note that
    \begin{equation} \label{del:eqq1}
        \boldsymbol{D}(t) = \sum_{r=1}^{td} \boldsymbol{a}(r) \boldsymbol{a}(r)^T.
    \end{equation}
    Let $\{\eta(r), r\geq 1\}$ be an i.i.d. Gaussian process with zero mean and unit variance. It can be verified that
    \begin{equation} \label{del:eqq2}
         \underline{\hat{\boldsymbol{\mu}}}(t) - \underline{\boldsymbol{\mu}} = D(t)^{-1} \left[ \sum_{r=1}^{td} \boldsymbol{a}(r) \eta(r) \right].
    \end{equation}
    Since each arm is sampled at least once before stopping, We have
    \begin{equation} \label{del:eqq3}
        \begin{aligned}
            \left(\sum_{r=1}^{td} \boldsymbol{a}(r) \boldsymbol{a}(r)^T\right) \succ I_{Md}
            \implies 2\left(\sum_{r=1}^{td} \boldsymbol{a}(r) \boldsymbol{a}(r)^T\right) \succ \left(\sum_{r=1}^{td} \boldsymbol{a}(r) \boldsymbol{a}(r)^T\right) + I_{Md}.
        \end{aligned}
    \end{equation}
    Using \eqref{del:eq1}, \eqref{del:eqq1}, \eqref{del:eqq2}, and \eqref{del:eqq3}, the probability of error in \eqref{del:eq4} is bounded as 
    % \begin{equation}
    \begin{align}
        &\mathbb{P} \left[ \tau_\delta < \infty \text{ and } \mathcal{C}\left(\hat{\boldsymbol{\mu}}\left(\tau_\delta\right)\right)\nsim \mathcal{C}(\boldsymbol{\mu}) \right] \\
        &\leq  \mathbb{P}\left[ \exists t \in \mathbb{N} :  \left\{ \frac{1}{\sigma^2}\left[ \sum_{r=1}^{td} \boldsymbol{a}(r) \eta(r) \right]^T \left[ \sum_{r=1}^{td} \boldsymbol{a}(r) \boldsymbol{a}(r)^T + I_{Md} \right]^{-1} \left[ \sum_{r=1}^{td} \boldsymbol{a}(r) \eta(r) \right]> \beta(\delta, t) \right\}  \right] \\
        &\leq  \mathbb{P}\left[ \exists t \in \mathbb{N} : \left\{ \left\| \sum_{r=1}^{td} \boldsymbol{a}(r) \eta(r) \right\|_{\left[ \sum_{r=1}^{td} \boldsymbol{a}(r) \boldsymbol{a}(r)^T + I_{Md} \right]^{-1}} > 2 \sigma^2 \log\left( \frac{\left| \sum_{r=1}^{td} \boldsymbol{a}(r) \boldsymbol{a}(r)^T + I_{Md} \right|^{\frac{1}{2}}}{\delta} \right) \right\}  \right] \\
        &\leq \delta. \text{ (Lemma \ref{delprop} in Section \ref{appsec:otherclaims})}
    \end{align}
    % \end{equation}
    % Hence proved.\\
    {
    {\em Case 2: } \textit{ATBOC-1pExp}\\
    From the expression of the test statistics $Z(t)$ for ATBOC-1pExp from \eqref{eq: z} and from the definition of the stopping time $\tau_\delta = \inf\left\{ t \in \mathbb{N}: Z(t) \geq \beta(\delta, t) \right\}$, the probability of error in \eqref{eq: peini} can be further upper-bounded as
    % In ATBOC-1pExp, the test statistic $Z(t)$ is
    % \begin{equation}
    %     Z(t)  =  \inf_{\boldsymbol{\lambda} \in \text{Alt}(\boldsymbol{\hat{\mu}}(t))}  \sum_{m=1}^M N_m(t) d_{\text{KL}}\left(\boldsymbol{\hat{\mu}}_m(t), \boldsymbol{\lambda}_m\right).
    % \end{equation}
    % By similar arguments as in the equations \eqref{eq: del1}, \eqref{eq: del2} and \eqref{del:eq4}, the probability of error can be upper bounded as follows.
    \begin{equation} \label{eq: dell3}
    \begin{aligned}
        &\mathbb{P} \left[ \tau_\delta < \infty \text{ and } \mathcal{C}\left(\hat{\boldsymbol{\mu}}\left(\tau_\delta\right)\right)\nsim \mathcal{C}(\boldsymbol{\mu}) \right]  \leq  \mathbb{P}\left[ \exists t \in \mathbb{N} : \left\{ \sum_{m=1}^M N_m(t) d_{\text{KL}}\left(\boldsymbol{\hat{\mu}}_m(t), \boldsymbol{\lambda}_m\right)> \beta(\delta, t) \right\}  \right].
    \end{aligned}
    \end{equation}    
    We use the KL-based concentration inequality of natural parameters for a single-parameter exponential family of distributions, presented in Lemma \ref{lemma: conc KL bound}. 
    % Now, we discuss the proof of Lemma \ref{lemma: conc KL bound}.
    On applying Lemma \ref{lemma: conc KL bound} on \eqref{eq: dell3}, we get 
        $\mathbb{P} \left[ \tau_\delta < \infty \text{ and } \mathcal{C}\left(\hat{\boldsymbol{\mu}}\left(\tau_\delta\right)\right)\nsim \mathcal{C}(\boldsymbol{\mu}) \right] \leq \delta.$
    }
\end{proof}

\begin{lemma} \label{lemma: conc KL bound}
    For some fixed $\zeta \in \left(0, \frac{1}{2}\right)$, we have
    $
    \mathbb{P}\left[ \exists\, t \in \mathbb{N}: \sum_{m=1}^M N_m(t) d_{\text{KL}}\left(\hat{\boldsymbol{\mu}}_m(t), \boldsymbol{\mu}_m\right) \geq \beta(\delta, t) \right] \leq \delta
    $, where $\beta(\delta, t) = 3\sum_{m=1}^M \log\left(1+\log\left(N_m(t)\right)\right) + (1+\zeta)M\log\left(\frac{\pi^2/3}{\left(\log(1+\zeta)\right)^2}\right) + (1+\zeta)\log\left(\frac{1}{\delta}\right)$.
\end{lemma}
    \begin{proof}[Proof]
        Define $X_m(t) \coloneqq N_m(t) d_{\text{KL}}\left( \hat{\mu}_m(t), \mu_m \right) - 3\log\left( 1+\log{N_m(t)} \right)$. For any $\lambda$ and $z>0$, we have 
        \begin{equation} \label{eq: mixmart}
            \begin{aligned}
                \left\{ e^{\lambda(X_m(t)-c)} \geq z \right\} &= \left\{ X_m(t)-c \geq \frac{\log(z)}{\lambda} \right\} 
                \subseteq \left\{ Z_m(t) \geq e^{\frac{\log(z)}{\lambda(1+\zeta)}} \right\} \ \ \text{(from Lemma \ref{lemma: mixmart})} \\
                &= \left\{ Z_m(t) \geq z^{\frac{1}{\lambda(1+\zeta)}} \right\} 
                = \left\{ Z_m(t)^{\lambda(1+\zeta)} \geq z \right\},
            \end{aligned}
        \end{equation}
        where $Z_m(t)$ is a martingale with $Z_m(0)=1$.
    Now, let us simplify the following probability term.
    % \begin{equation}
        \begin{align}
            \mathbb{P}\left[ \exists t \in \mathbb{N} : \sum_{m=1}^M X_m(t) > y \right] &= \mathbb{P}\left[ \exists t \in \mathbb{N} : e^{\lambda\sum_{m=1}^M X_m(t)} > e^{\lambda y} \right] \\
            % &= \mathbb{P}\left[ \exists t \in \mathbb{N} : e^{\lambda\sum_{m=1}^M (X_m(t)-c)} > e^{\lambda (y-cM)} \right] \\
            &= \mathbb{P}\left[ \exists t \in \mathbb{N} : \prod_{m=1}^M e^{\lambda (X_m(t)-c)} > e^{\lambda (y-cM)} \right] \\
            &\leq \mathbb{P}\left[ \exists t \in \mathbb{N} : \prod_{m=1}^M Z_m(t)^{\lambda(1+\zeta)} > e^{\lambda (y-cM)} \right] \quad \text{(from \eqref{eq: mixmart})}\\
            % &= \mathbb{P}\left[ \exists t \in \mathbb{N} : \left[\prod_{m=1}^M Z_m(t)\right]^{\lambda(1+\zeta)} > e^{\lambda (y-cM)} \right] \\
            &= \mathbb{P}\left[ \sup_{t \in \mathbb{N}} \left[\prod_{m=1}^M Z_m(t)\right]^{\lambda(1+\zeta)} > e^{\lambda (y-cM)} \right]. \\
        \end{align}
    % \end{equation}
    On choosing $\lambda \leq \frac{1}{1+\zeta}$, the term $\left[\prod_{m=1}^M Z_m(t)\right]^{\lambda(1+\zeta)}$ is a super martingales and on applying the Vile's inequality,
    \begin{equation}
        \mathbb{P}\left[ \exists t \in \mathbb{N} : \sum_{m=1}^M X_m(t) > y \right] \leq e^{-\lambda (y-cM)}.
    \end{equation}
    On choosing the value of $\lambda$ which minimizes the above upper bound $\left( \lambda = \frac{1}{1+\zeta} \right)$, we get,
    \begin{equation}
        \mathbb{P}\left[ \exists t \in \mathbb{N} : \sum_{m=1}^M X_m(t) > y \right] \leq e^{-\frac{(y-cM)}{1+\zeta}} .
    \end{equation}
    Let $\delta$ be such that $e^{-\frac{(y-cM)}{1+\zeta}} \leq \delta$. 
    % On solving for $y$ in terms of $\delta$, we get $y \geq cM + (1+\zeta)\log\left(\frac{1}{\delta}\right)$. 
    Now, on replacing $y$ in terms of $\delta$, we get
    \begin{equation}
        \mathbb{P}\left[ \exists t \in \mathbb{N} : \sum_{m=1}^M X_m(t) > cM + (1+\zeta)\log\left(\frac{1}{\delta}\right) \right] \leq \delta.
    \end{equation}
    This proves Lemma \ref{lemma: conc KL bound}.
    \end{proof}

% ------------------------------------------------------------------

\section{Proof of Theorem \ref{Theorem3:AlmostSureOptimal}-Almost sure sample complexity upper bound} \label{appsec: almostsureoptimal}
\begin{proof}
    For ATBOC-Gauss/ATBOC-1pExp, the upper bound on the stopping time $\tau_\delta$ in \eqref{eq: finite stopping time} holds under the event $\mathcal{E}$ in \eqref{fine stopping: event}. 
On dividing \eqref{eq: finite stopping time} by $\log\left(\frac{1}{\delta}\right)$ and taking $\limsup_{\delta \rightarrow 0}$, and then letting $\epsilon$ tend to $0$, we get 
    \begin{equation} \label{eq:2ts}
        \limsup_{\delta \rightarrow 0} \frac{\tau_\delta}{\log\left( 
        \frac{1}{\delta} \right)} \leq  a T^*(\boldsymbol{\mu}).  
    \end{equation}
    We have $\mathcal{E} \subseteq \left\{ \limsup_{\delta \rightarrow 0} \frac{\tau_\delta}{\log\left(\frac{1}{\delta}\right)} \leq aT^*(\boldsymbol{\mu}) \right\}.$
    % \begin{equation}
    %     \mathcal{E} \subseteq \left\{ \limsup_{\delta \rightarrow 0} \frac{\tau_\delta}{\log\left(\frac{1}{\delta}\right)} \leq aT^*(\boldsymbol{\mu}) \right\}.
    % \end{equation}
    Since $\mathbb{P}\left[\mathcal{E}\right] = 1$, we get $\mathbb{P}\left[ \limsup_{\delta \rightarrow 0} \frac{\tau_\delta}{\log\left(\frac{1}{\delta}\right)} \leq aT^*(\boldsymbol{\mu}) \right] = 1.$
    % \textcolor{blue}{\begin{equation}
    %     \mathbb{P}\left[ \limsup_{\delta \rightarrow 0} \frac{\tau_\delta}{\log\left(\frac{1}{\delta}\right)} \leq aT^*(\boldsymbol{\mu}) \right] = 1.
    % \end{equation}}

    For the ATBOC-subGauss algorithm, we can follow a procedure similar to the one described above, except that we use \eqref{eq: sampcompupbndatbocsubgauss} in place of \eqref{eq: finite stopping time}. This yields $\mathbb{P}\left[ \limsup_{\delta \rightarrow 0} \frac{\tau_\delta}{\log\left(\frac{1}{\delta}\right)} \leq aT_{\text{mod}}^*(\boldsymbol{\mu}) \right] = 1.$
% This completes the proof of Theorem~\ref{Theorem2:deltaPAC}.
\end{proof}
% -------------------------------------------------------------------

\section{Proof of Theorem \ref{Theorem4:AsymptoticOptimal}-Expected sample complexity upper bound} \label{appsec: asymptoticoptimal}
% \begin{proof}
\noindent {\em Case 1: ATBOC-Gauss/ATBOC-1pExp}

    Consider the problem instance $\boldsymbol{\mu} \in \Theta^{M}$ and $\mathcal{S}^{*}\left( \boldsymbol{\mu} \right)$ to be its corresponding optimal set of arm pull proportions. Consider a real number $\epsilon > 0$.
    
    By {continuity of $\psi$} (Lemma \ref{proposition:innermincont}), there exists $\zeta_1(\epsilon) > 0$ such that for all $\boldsymbol{\mu}^{'} \in \Theta^{M}$ and $\boldsymbol{w^{'}} \in \mathcal{P}_M$, if $\| \boldsymbol{\mu}^{'}-\boldsymbol{\mu} \| \leq \zeta_1(\epsilon)$ and $d_\infty\left( \boldsymbol{w^{'}}, \mathcal{S}^*(\boldsymbol{\mu}) \right) \leq \zeta_1(\epsilon),$ then, for some $\boldsymbol{w^*} \in \argmin_{\boldsymbol{w} \in \mathcal{S}^*(\boldsymbol{\mu})} \|\boldsymbol{w}- \boldsymbol{w^{'}}\|_{\infty}$, we have
\begin{equation} \label{37}
%\begin{aligned}
   \left| \psi(\boldsymbol{w^*}, \boldsymbol{\mu})-\psi(\boldsymbol{w^{'}}, \boldsymbol{\mu}^{'}) \right| \leq \epsilon \psi(\boldsymbol{w^*}, \boldsymbol{\mu}) = \epsilon T^*(\boldsymbol{\mu})^{-1} .
%\end{aligned}
\end{equation}
% for $\boldsymbol{w^*} \in \argmin_{\boldsymbol{w} \in \mathcal{S}^*(\boldsymbol{\mu})} d_\infty(\boldsymbol{w}, \boldsymbol{w^{'}}).$
By {upper hemi continuity of $\mathcal{S}^*$} (Lemma \ref{Lemma2:BergeMaximum}), $\exists \zeta_2(\epsilon)>0$ such that for all $\boldsymbol{\mu}^{'} \in \Theta^{M}$, if $\| \boldsymbol{\mu} - \boldsymbol{\mu}^{'} \| \leq \zeta_2(\epsilon),$ then 
\begin{equation} \label{38}
    %\begin{aligned}
        \max_{\boldsymbol{w^{''}} \in \mathcal{S}^*(\boldsymbol{\mu}^{'})} d_\infty(\boldsymbol{w^{''}}, \mathcal{S}^*(\boldsymbol{\mu})) \leq \frac{\zeta_1(\epsilon)}{M-1}.
    %\end{aligned}
\end{equation}

Let $\zeta(\epsilon) = \min[\zeta_1(\epsilon), \zeta_2(\epsilon)]$. For $T\geq 1$, define the event, 
\begin{equation} \label{asymp:eq1}
    \mathcal{E}_{1, T} \coloneqq \bigcap_{t = T}^\infty \left\{ \| \boldsymbol{\mu} - \hat{\boldsymbol{\mu}}(t) \| \leq \zeta(\epsilon) \right\}.
\end{equation}
First, we state and prove the following two claims (\ref{claim:1} and \ref{claim:2}). Then we proceed to prove Theorem \ref{Theorem4:AsymptoticOptimal}.

\begin{claim} \label{claim:1}
    For all $T\geq T_3^{*}(\delta)$, we have $\mathcal{E}_T \subseteq \left\{ \tau \leq T \right\}$, where $T_3^{*}(\delta) = \max\left\{T_0, T_2^{*}(\delta)\right\}$ with
    \begin{align*}
        \mathcal{E}_T&\coloneqq \mathcal{E}_{1, \left\lfloor \epsilon_1T \right\rfloor} \text{ with } \epsilon_1 = \frac{\zeta_1(\epsilon)}{M-1}, 
        T_0 = \max\left\{ \frac{1}{\epsilon_1}+ \frac{1}{\epsilon_1^2}, \frac{1}{M\epsilon_1^3}+\frac{1}{\epsilon_1^2}, \frac{2}{\epsilon_1^2} \right\} \text{ and}\\
        T_2^{*}(\delta) &= \frac{2 T^*(\boldsymbol{\mu})}{1-\epsilon} \log\left( \frac{1}{\delta} \right) + \frac{Md T^*(\boldsymbol{\mu})}{1-\epsilon} +\left( \log\left( \frac{MdT^*(\boldsymbol{\mu})e}{(1-\epsilon)} \right)  \log\log\left( \frac{MdT^*(\boldsymbol{\mu})}{\delta^{\frac{2}{Md}} (1-\epsilon)} \right) \right).
    \end{align*}
\end{claim}
\begin{proof}
In ATBOC-Gauss / ATBOC-1pExp, at any time $t$, we have $w(t) \in \mathcal{S}^{*}\left(\hat{\boldsymbol{\mu}}(t)\right)$, and hence we can write
\begin{equation} \label{eq68}
    d_\infty\left( \boldsymbol{w}^{*}(t), \mathcal{S}^*(\boldsymbol{\mu}) \right) \leq \max_{\boldsymbol{w^{'}}\in \mathcal{S}^*(\hat{\boldsymbol{\mu}}(t))} d_\infty\left( \boldsymbol{w^{'}}, \mathcal{S}^*(\boldsymbol{\mu}) \right)
\end{equation}

For $T\geq 1$, under the event $\mathcal{E}_{1, T}$, we have, for all $t \geq T$, 
\begin{equation} \label{eq69}
    \max_{\boldsymbol{w^{'}}\in \mathcal{S}^*(\hat{\boldsymbol{\mu}}(t))} d_\infty\left( \boldsymbol{w^{'}}, \mathcal{S}^*(\boldsymbol{\mu}) \right) \leq \frac{\zeta_1(\epsilon)}{M-1}. \ \ (\text{from } \eqref{38} \text{ and }\eqref{asymp:eq1})
\end{equation}
From \eqref{eq68} and \eqref{eq69}, for $T\geq 1$, under the event $\mathcal{E}_{1, T}$, we have, for all $t \geq T$, 
\begin{equation}
    d_\infty\left( \boldsymbol{w}(t), \mathcal{S}^*(\boldsymbol{\mu}) \right) \leq \epsilon_1 \text{, where } \epsilon_1 = \frac{\zeta_1(\epsilon)}{M-1}.
\end{equation}
Now applying Lemma \ref{Lemma1:TrakingC} to the above equation, we get for $T\geq 1$, under the event $\mathcal{E}_{1, T}$, we have
\begin{equation}
    d_\infty\left( \frac{\boldsymbol{N}(t)}{t}, \mathcal{S}^*(\boldsymbol{\mu}) \right) \leq (M-1) \epsilon_1 = \zeta_1(\epsilon) \quad \text{for all } t\geq t_1, \quad t_1 = \max\left\{ \frac{1}{\epsilon_1}, \frac{1}{M\epsilon_1^3}, \frac{T}{\epsilon_1^2}\right\}.
\end{equation}
 Since for $T \geq \max\left\{ \epsilon_1, \frac{1}{M\epsilon_1}\right\}$, we have $t_1 \leq  \frac{T}{\epsilon_1^2}$, we say that for $T \geq \max\left\{ 1, \epsilon_1, \frac{1}{M\epsilon_1}\right\}$, under the event $\mathcal{E}_{1, T}$, 
\begin{equation}
    d_\infty\left( \frac{\boldsymbol{N}(t)}{t}, \mathcal{S}^*(\boldsymbol{\mu}) \right) \leq \zeta_1(\epsilon) \quad \text{for all } t\geq \frac{T}{\epsilon_1^2}.
\end{equation}
Changing $T$ with $\left\lfloor \epsilon_1^2T \right\rfloor-1$, we get, for $\left\lfloor \epsilon_1^2T \right\rfloor-1 \geq \max\left\{ 1, \epsilon_1, \frac{1}{M\epsilon_1}\right\}$, under the event $\mathcal{E}_{1, \left\lfloor \epsilon_1^2T \right\rfloor-1}$,
\begin{equation}
    d_\infty\left( \frac{\boldsymbol{N}(t)}{t}, \mathcal{S}^*(\boldsymbol{\mu}) \right) \leq \zeta_1(\epsilon) \quad \text{for all } t \geq \frac{\left\lfloor \epsilon_1^2T \right\rfloor-1}{\epsilon_1^2}.
\end{equation}
Define $\mathcal{E}_T = \mathcal{E}_{1, \lfloor T\epsilon_1^2\rfloor-1}$. 
Now, we say that for $T\geq T_0$ with $T_0 = \max\left\{ \frac{1}{\epsilon_1}+ \frac{1}{\epsilon_1^2}, \frac{1}{M\epsilon_1^3}+\frac{1}{\epsilon_1^2}, \frac{2}{\epsilon_1^2} \right\}$ under event $\mathcal{E}_T$,  
\begin{equation} \label{eq:asym155}
    d_\infty\left( \frac{\boldsymbol{N}(t)}{t}, \mathcal{S}^*(\boldsymbol{\mu}) \right) \leq  \zeta_1(\epsilon), \forall t \geq T.
\end{equation}
Also, under the event $\mathcal{E}_T$, we have $
    \| \boldsymbol{\mu} - \hat{\boldsymbol{\mu}}(t) \| \leq \zeta(\epsilon), \forall t \geq \lfloor \epsilon_1^2T \rfloor$.
Let us assume $\epsilon$ is small enough, such that $\epsilon_1 < 1$.
Hence, for $T \geq T_0$, under the event $\mathcal{E}_T$,  we have 
\begin{equation} \label{eq:asym157}
    \| \boldsymbol{\mu} - \hat{\boldsymbol{\mu}}(t) \| \leq \zeta(\epsilon), \quad \forall t \geq T.
\end{equation}
From  \eqref{eq:asym155}, \eqref{eq:asym157} and \eqref{37}, for $T\geq T_0$, under the event $\mathcal{E}_T$, we have
\begin{equation}
    \psi\left( \frac{\boldsymbol{N}(t)}{t}, \hat{\boldsymbol{\mu}}(t) \right) \geq (1-\epsilon) T^*(\boldsymbol{\mu})^{-1}, \quad \forall t \geq T. 
\end{equation}
Hence, for $T\geq T_0$, under the event $\mathcal{E}_T$, for any $t\geq T$, we can lower bound the test statistics as 
\begin{equation}
\begin{aligned}
    Z(t) &= t\psi\left( \frac{\boldsymbol{N}(t)}{t}, \hat{\boldsymbol{\mu}}(t) \right) \geq t(1-\epsilon) T^*(\boldsymbol{\mu})^{-1}.
\end{aligned}
\end{equation}
From \eqref{eq:thresholdub}, we can upper bound the threshold as 
\begin{equation}
    \begin{aligned}
        % \beta(\delta, t) &= 2 \log\left[ \frac{\prod_{m=1}^M\left( N_m(t)+1 \right)^{\frac{d}{2}}}{\delta} \right] \\
        % &\leq 2 \log\left[ \frac{t^{\frac{Md}{2}}}{\delta} \right]. \ \ (\because N_m(t+1) \leq t)
        \beta(\delta, t) \leq a \log\left(\frac{bt^c}{\delta}\right).
    \end{aligned}
\end{equation}
Following similar steps as those of \eqref{eq: finite stopping time} (except $t_0$ in \eqref{eq: finite stopping time} is replaced with $T$), we get for $T\geq T_0$, under the event $\mathcal{E}_T$, the stopping time $\tau \leq \max\left\{ T, T_2^*(\delta) \right\}$, where
% \begin{equation}
%     \tau \leq \max\left\{ T, T_2^*(\delta) \right\}
% \end{equation}
% where, 
\begin{equation} \label{T2star}
\begin{aligned}
    T_2^*(\delta) =  &\frac{a T^*(\boldsymbol{\mu})}{1-\epsilon} \log\left( \frac{1}{\delta} \right) + \frac{ac T^*(\boldsymbol{\mu})}{1-\epsilon} \left( \log\left( \frac{ab^{\frac{1}{c}}cT^*(\boldsymbol{\mu})e}{(1-\epsilon)} \right) + \log\log\left( \frac{ab^{\frac{1}{c}}cT^*(\boldsymbol{\mu})}{\delta^{\frac{1}{c}} (1-\epsilon)} \right) \right).
\end{aligned}
\end{equation}
Define $T_3^*(\delta) = \max\left\{T_0, T_2^*(\delta)\right\}$. Now, for $T\geq T_3^{*}(\delta)$, under $\mathcal{E}_T$, we have $\tau \leq T$, i.e., $\mathcal{E_T} \subseteq \{  \tau \leq T\}.$
% Therefore, for $T \geq T_3^*(\delta)$, we have
% \begin{equation} \label{asym:eq2}
%     \mathcal{E_T} \subseteq \{  \tau \leq T\}.
% \end{equation}
% Hence proved.
\end{proof}

\begin{claim} \label{claim:2}
    \begin{equation*} 
    \begin{split}
        &\sum_{T=T_3^*(\delta)}^{\infty} \mathbb{P}\left[ \mathcal{E}_T^C \right] \leq \sum_{T=T_3^*(\delta)}^{T_5-1} \mathbb{P}\left[ \mathcal{E}_T^C \right] + c_2(\epsilon) \sum_{T=T_5}^{T_6} \frac{T^{f+0.5}}{\exp{\left(c_3(\epsilon)T^{\frac{\gamma}{2}}\right)}} + \frac{c_2(\epsilon)}{c_3(\epsilon)^{\frac{2f+3}{\gamma}}} \frac{2}{\gamma} \Gamma\left[ \frac{2f+3}{\gamma} \right], 
        \end{split}
\end{equation*}
where $T_6>T_5>T_3^{*}$, with $T_5$ and $T_6$ being finite integers.
\end{claim}
\begin{proof}
    First, we upper bound $\mathbb{P}\left[ \mathcal{E}_T^C \right]$ as 
\begin{equation} \label{asym:eq5}
    \begin{aligned}
        \mathbb{P}\left[ \mathcal{E}_T^C \right] &= \mathbb{P}\left[ \mathcal{E}_{1, \left\lceil 
        \epsilon_1^2T \right\rceil-1}^C \right] 
        \leq \sum_{t = \lceil\epsilon_1^2T\rceil-1 }^\infty \mathbb{P}\left[ 
        \| \boldsymbol{\mu} - \hat{\boldsymbol{\mu}}(t) \| > \zeta(\epsilon) \right] 
        \leq \sum_{t =  \lceil\epsilon_1^2T\rceil-1 }^\infty et^f\exp{\left(-g\sqrt{t}\right)},
    \end{aligned}
\end{equation}
where, $(e, f, g) = \begin{cases}
    \left(c^{\frac{-Md}{2}}, \frac{Md}{4}, \frac{c\zeta(\epsilon)^2}{4}\right) &\text{for ATBOC-Gauss}\\
    \left(2M, 1, cB\right) &\text{for ATBOC-1pExp}
\end{cases}$, for some constant $c$ (Lemma \ref{Lemma:Concbound} in Section \ref{appsec:otherclaims}).
In \eqref{asym:eq5}, for large $t$, the exponential dominates the polynomial. Hence, $\exists T_4 \geq T_3^*(\delta)$, such that $\forall T > T_4$, $t^{f}\exp\left({-g\sqrt{t}}\right) $ is decreasing in $( \lceil\epsilon_1^2T\rceil -2, \infty)$. Hence, for all $T\geq T_4$, we can upper bound the summation by the integral as 
\begin{equation} \label{asym:eq6}
    \begin{aligned}
        \mathbb{P}\left[ \mathcal{E}_T^C \right] &\leq c_3 \int_{t =  \lceil\epsilon_1^2T\rceil -2}^\infty e t^f\exp\left({-g\sqrt{t}}\right) dt. \\
    \end{aligned}
\end{equation}
By changing the variable $x = g\sqrt{t}$ in the integral and after simplification, we get
\begin{equation} \label{asym:eq3}
    \begin{aligned}
        \mathbb{P}\left[ \mathcal{E}_T^C \right] 
        \leq \frac{2e}{g^{2(f+1)}}\int_{g\sqrt{\lceil\epsilon_1^2T\rceil - 2}}^\infty x^{2f+1}e^{-x} dx 
        = \frac{2e}{g^{2(f+1)}} \Gamma\left[ 2f+2, g\sqrt{\lceil\epsilon_1^2T\rceil - 2} \right].
    \end{aligned}
\end{equation}
From Lemma \ref{Lemma:gammabound} in Section \ref{appsec:otherclaims}, we get
\begin{equation} 
    \begin{aligned}
        \mathbb{P}\left[ \mathcal{E}_T^C \right] &\leq \frac{2e}{g^{2(f+1)}} \left[g\sqrt{\lceil\epsilon_1^2T\rceil - 2} \right]^{2f+1} \exp\left( -g\sqrt{\lceil\epsilon_1^2T\rceil - 2} \right).
    \end{aligned}
\end{equation}
We have $\lceil \epsilon_1^2T\rceil - 2 \leq \epsilon_1^2T$. Let $T_5 = \max\left\{T_4, \frac{2+c_1}{\epsilon_1^2}\right\}$ for some $c_1>1$. For $T\geq T_5$, we have $\lceil \epsilon_1^2T\rceil - 2 \geq \lceil \epsilon_1^2T_5\rceil - 2 \geq c_1 \geq \left( \epsilon_1^2T\right)^\gamma$ for some $\gamma\in (0, 1)$. Hence, for $T\geq T_5$, we have
\begin{equation} \label{asym:eq4}
    \begin{aligned}
        \mathbb{P}\left[ \mathcal{E}_T^C \right] &\leq c_2(\epsilon) T^{f+0.5} \exp\left( -c_3(\epsilon)T^{\frac{\gamma}{2}} \right), \quad \text{where } c_2(\epsilon) = \frac{2e\epsilon_1^{2f+1}}{g}, c_3(\epsilon) = g\epsilon_1^\gamma.
    \end{aligned}
\end{equation}
% where, $c_2(\epsilon) = \frac{2e\epsilon_1^{2f+1}}{g}$, $c_3(\epsilon) = g\epsilon_1^\gamma$. \\ 
Now, we upper bound $\sum_{T=T_3^*(\delta)}^\infty \mathbb{P}\left[ \mathcal{E}_T^C \right]$ as 
\begin{equation}
    \begin{aligned}
        \sum_{T=T_3^*(\delta)}^\infty \mathbb{P}\left[ \mathcal{E}_T^C \right] &= \sum_{T=T_3^*(\delta)}^{T_5-1} \mathbb{P}\left[ \mathcal{E}_T^C \right] + \sum_{T=T_5}^\infty \mathbb{P}\left[ \mathcal{E}_T^C \right]
        \leq \sum_{T=T_3^*(\delta)}^{T_5-1} \mathbb{P}\left[ \mathcal{E}_T^C \right] + c_2(\epsilon) \sum_{T=T_5}^\infty \frac{T^{f+0.5}}{\exp{\left(c_3(\epsilon)T^{\frac{\gamma}{2}}\right)}}. \ \ \text{(from \eqref{asym:eq4})} \\
    \end{aligned}
\end{equation}
Similarly to the explanation from \eqref{asym:eq5} to \eqref{asym:eq6}, there exists $T_6$ such that for all $T\geq T_6$, we have 
\begin{equation} \label{asym:eq7} 
\begin{aligned}
    &\sum_{T=T_3^*(\delta)}^\infty \mathbb{P}\left[ \mathcal{E}_T^C \right] \leq \sum_{T=T_3^*(\delta)}^{T_5-1} \mathbb{P}\left[ \mathcal{E}_T^C \right] + c_2(\epsilon) \sum_{T=T_5}^{T_6} \frac{T^{f+0.5}}{\exp{\left(c_3(\epsilon)T^{\frac{\gamma}{2}}\right)}} + c_2(\epsilon) \int_{T=T_6}^\infty \frac{T^{f+0.5}}{\exp{\left(c_3(\epsilon)T^{\frac{\gamma}{2}}\right)}}dT.
\end{aligned}
\end{equation}
Let us simplify the third term in \eqref{asym:eq7}. By changing the variable $x = c_2(\epsilon)T^{\frac{\gamma}{2}}$, we get 
\begin{equation} \label{asym:eq8}
\begin{aligned}
    \int_{T=T_6}^\infty \frac{T^{f+0.5}}{\exp{\left(-c_3(\epsilon)T^{\frac{\gamma}{2}}\right)}} dT &\leq \frac{1}{c_3(\epsilon)^{\frac{2f+3}{\gamma}}} \frac{2}{\gamma} \int_{c_3(\epsilon)T_6^{\frac{\gamma}{2}}}^\infty x^{\frac{2f+3}{\gamma}-1}\exp(-x) dx \\
    &\leq \frac{1}{c_3(\epsilon)^{\frac{2f+3}{\gamma}}} \frac{2}{\gamma} \int_{0}^\infty x^{\frac{2f+3}{\gamma}-1}\exp(-x) dx 
    \leq \frac{1}{c_3(\epsilon)^{\frac{2f+3}{\gamma}}} \frac{2}{\gamma} \Gamma\left[ \frac{2f+3}{\gamma} \right] .
\end{aligned}
\end{equation}
Substituting \eqref{asym:eq8} into \eqref{asym:eq7} proves the claim.
% Substituting equation \eqref{asym:eq8} in equation \eqref{asym:eq7}, we get, 
% \begin{equation} 
% \begin{split}
%         \sum_{T=T_3^*(\delta)}^{\infty} \mathbb{P}\left[ \mathcal{E}_T^C \right] \leq \sum_{T=T_3^*(\delta)}^{T_5-1} \mathbb{P}\left[ \mathcal{E}_T^C \right] + c_2(\epsilon) \sum_{T=T_5}^{T_6} \frac{T^{f+0.5}}{\exp{\left(c_3(\epsilon)T^{\frac{\gamma}{2}}\right)}} + \frac{c_2(\epsilon)}{c_3(\epsilon)^{\frac{2f+3}{\gamma}}} \frac{2}{\gamma} \Gamma\left[ \frac{2f+3}{\gamma} \right]. 
%         \end{split}
% \end{equation}
% Hence proved.
\end{proof}

Now, we proceed to prove Theorem \ref{Theorem4:AsymptoticOptimal}.
\begin{proof}
We upper bound $\mathbb{E}[\tau_\delta]$ as  
%\begin{equation} 
    \begin{align}\label{asym:eq9}
        \mathbb{E}[\tau_\delta]  &=  \sum_{T=0}^\infty \mathbb{P}\left[ \tau_\delta > T \right] \nonumber 
        = \sum_{T=0}^{T_3^*(\delta)-1} \mathbb{P}\left[ \tau_\delta > T \right] + \sum_{T=T_3^*(\delta)}^\infty \mathbb{P}\left[ \tau_\delta > T \right] \nonumber\\
        &\leq T_3^*(\delta)\times 1 + \sum_{T=T_3^*(\delta)}^\infty \mathbb{P}\left[ \tau_\delta > T \right] \nonumber
        \leq T_3^*(\delta) + \sum_{T=T_3^*(\delta)}^\infty \mathbb{P}\left[ \mathcal{E}_T^C \right] \ \ \text{(Claim \ref{claim:1})} \nonumber\\
        &\leq T_0 + T_2^*(\delta) + \sum_{T=T_3^*(\delta)}^\infty \mathbb{P}\left[ \mathcal{E}_T^C \right]\ \ (\because T_3^*(\delta) = \max\left\{T_0, T_2^*(\delta)\right\})  \\
        &\leq T_0 + T_2^*(\delta) + \sum_{T=T_3^*(\delta)}^{T_5-1} \mathbb{P}\left[ \mathcal{E}_T^C \right] +c_2(\epsilon) \sum_{T=T_5}^{T_6} \frac{T^{f+0.5}}{\exp{\left(c_3(\epsilon)T^{\frac{\gamma}{2}}\right)}} + \frac{c_2(\epsilon)}{c_3(\epsilon)^{\frac{2f+3}{\gamma}}} \frac{2}{\gamma} \Gamma\left[ \frac{2f+3}{\gamma} \right] \text{(Claim \ref{claim:2})}.
    \end{align} 
%\end{equation}
% From Claim \ref{claim:2}, we get, 
% \begin{equation}
% \begin{split}
%     &\mathbb{E}[\tau_\delta] \leq T_0 + T_2^*(\delta) + \sum_{T=T_3^*(\delta)}^{T_5-1} \mathbb{P}\left[ \mathcal{E}_T^C \right] +c_2(\epsilon) \sum_{T=T_5}^{T_6} \frac{T^{f+0.5}}{\exp{\left(c_3(\epsilon)T^{\frac{\gamma}{2}}\right)}} + \frac{c_2(\epsilon)}{c_3(\epsilon)^{\frac{2f+3}{\gamma}}} \frac{2}{\gamma} \Gamma\left[ \frac{2f+3}{\gamma} \right].
%     \end{split}
% \end{equation}
% We know $\Gamma(n)$ is finite for all $n\geq 0$. Hence, all terms except $T_2^*(\delta)$ are finite in the upper bound and do not depend on $\delta$. Hence, by 
Dividing both sides by $\log\left( \frac{1}{\delta} \right)$ and taking $\limsup_{\delta \rightarrow 0}$, we get
\begin{equation}
    \limsup_{\delta \rightarrow 0} \frac{\mathbb{E}[\tau_\delta]}{\log\left( \frac{1}{\delta} \right)} \leq \limsup_{\delta \rightarrow 0} \frac{T_2^*(\delta)}{\log\left( \frac{1}{\delta} \right)} \leq a T^*(\boldsymbol{\mu}) \quad \text{(from \eqref{T2star}, on letting $\epsilon \rightarrow 0$)}.
\end{equation}
%Expression of $T_2^*(\delta)$ is the same as that of the second term in $\max$ of \eqref{eq: finite stopping time} and as in \eqref{eq:2ts} 
% On taking $\limsup_{\delta \rightarrow 0}$ using equation \eqref{T2star} and on letting $\epsilon \rightarrow 0$, we get
% % On letting $\epsilon \rightarrow 0$, from equation \eqref{T2star}, we get 
% \textcolor{blue}{
% \begin{equation}
% \begin{aligned}
%     \limsup_{\delta \rightarrow 0} \frac{\mathbb{E[\tau_\delta]}}{\log\left(\frac{1}{\delta}\right)} &\leq a T^*(\boldsymbol{\mu}).
% \end{aligned}
% \end{equation}
% }
\end{proof}
\noindent {\em Case 2: ATBOC-subGauss} 

 Following a procedure similar to the one described above, but using $\psi_{\text{mod}}(\cdot, \cdot)$, $T_{\text{mod}}^{}(\cdot)$, and $\mathcal{S}_{\text{mod}}(\cdot)$ in place of $\psi(\cdot, \cdot)$, $T^{}(\cdot)$, and $\mathcal{S}(\cdot)$, respectively, we obtain $\limsup_{\delta \rightarrow 0} \frac{\mathbb{E[\tau_\delta]}}{\log\left(\frac{1}{\delta}\right)} \leq a T_{\text{mod}}^*(\boldsymbol{\mu}).$ 

\section{BOC-ELIM - Analysis Proof} \label{appsec: bocelim}
\subsection{Proof of Theorem \ref{BOC-Elim delta PC}}\label{appsubsec:elimdelpc}
\begin{proof}
%\textcolor{red}{Recall that we assume that the means are sorted, i.e., $\boldsymbol{\mu}_1\geq \boldsymbol{\mu}_2 \ldots \geq \boldsymbol{\mu}_M$. Let us denote the arms corresponding to the $k^{th}$ highest gap be $m_k$ and $m_k+1$.}
    If the algorithm returns the wrong clustering then any one of the following cases should have happened.
    % \begin{itemize}
    %     \item If there exists an arm that belongs to the top $K-1$ gaps but is declared as not belonging to the top $K-1$ gaps, i.e., $\exists m \in \left\{ m_1, \ldots, m_{K-1} \right\}$ such that $U\Delta^l_m(t) < L\Delta^{(K-1)}(t)$ or $\exists m \in \left\{ m_1+1, \ldots, m_{K-1}+1 \right\}$ such that $U\Delta^r_m(t) < L\Delta^{(K-1)}(t)$.
    %     \item If there exists an arm that does not belongs to the top $K-1$ gaps
    % \end{itemize}
    
    \begin{itemize}
        \item {\em Case 1:} If there exists an arm that has one of the top $K-1$ gaps on its left side but is declared as not having one of the top $K-1$ gaps on its left side, i.e.,   $\exists m \in \left\{ m_1, \ldots, m_{K-1} \right\}$ such that $U\Delta^l_m(t) < L^{(K-1)}(t)$ or 
        \item {\em Case 2:} $\exists m \in \left\{ m_1+1, \ldots, m_{K-1}+1 \right\}$ such that $U\Delta^r_m(t) < L^{(K-1)}(t)$ or
        \item {\em Case 3:} $\exists m \in [M]\setminus\left\{ m_1, \ldots, m_{K-1} \right\}$ such that $L\Delta^l_m(t) > U^{(K)}(t)$ or
        \item {\em Case 4:} $\exists m \in [M]\setminus\left\{ m_1+1, \ldots, m_{K-1}+1 \right\}$ such that $L\Delta^r_m(t) > U^{(K)}(t)$.
    \end{itemize}
    Define good event $E$ as the event where the true means lies inside their corresponding confidence intervals, i.e.,  $E\coloneqq \bigcap_{t \in \mathbb{N}}\bigcap_{m \in [M]} \left\{ \mu_m \in [l_m(t), r_m(t)] \right\}$.
    Now, we show that under the good event $E$, none of the four cases described above can occur. We prove this by contradiction. We assume that each case holds, one at a time, along with the good event $E$, and show that this leads to a contradiction. \\
    % if any one of the above events holds, then it will lead to a contradiction.\\
    {\em Case 1:} Assume $\exists m \in \left\{ m_1, \ldots, m_{K-1} \right\}$ such that $U\Delta^l_m(t) < L^{(K-1)}(t)$.\\
    We have $\Delta_{(k)}=\mu_m-\mu_{m+1}$, for some $k \in \{1, \ldots, K-1\}$. We write the following.
    \begin{equation}
        \Delta_{(k)} \overset{(a)}{\leq} U\Delta_{m}^l(t) 
        \overset{(b)}{<} L^{(K-1)}(t) \overset{(c)}{\leq} l_{a_1}(t) - r_{a_1+1}(t) \overset{(d)}{\leq} {\mu}_{a_1} - {\mu}_{a_1+1}
    \end{equation}
    Since, for any arm $m$, its actual left gap is less than the maximum left gap of the arm, $(a)$ holds. $(b)$ holds, because of {\em Case 1} assumption. Here, the notation $(k)_t$ denotes the arm with the $k^{th}$ highest empirical mean at time $t$. Suppose that the $K-1^{th}$ highest gap corresponds to the gap $s$. Then, there must exits an $a_1\in [M]$ such that arm $a_1 \in \left\{(1)_t, \dots, (s)_t\right\}$ and arm $a_1+1 \in \left\{ (s+1)_t, \dots, (M)_t \right\}$ and hence $(c)$ holds. Since, $l_{a_1}(t) \leq \mu_{a_1}(t)$ and $r_{a_1+1}(t) \geq \mu_{a_1+1}(t)$, $(d)$ holds. Since $L^{(1)}(t) \geq L^{(2)}(t) \geq \ldots \geq L^{(K-1)}(t)$ and then by using the arguments similar to that of $(c)$ and $(d)$, we have the following $K-2$ inequalities.    
    % Let $m \in \left\{ m_1, \dots, m_{K-1} \right\}$. We assume $U\Delta_m^l(t) < L\Delta^{(K-1)}(t)$.\\
    % Here we have that the left gap of the arm $m$ corresponds to one of the top $K-1$ highest gap, i.e., it corresponds to $k^{th}$ highest gap for some $k \in \left\{ 1, \dots, K-1 \right\}$. Hence, we can say $\Delta_{(k)} = \boldsymbol{\mu}_m - \boldsymbol{\mu}_{m+1}$ and we write the following.
    % \begin{align}
    %     % \Delta_{(k)} \overset{(a)}{\leq} U\Delta_{m}^l(t) 
    %     % &\overset{(b_1)}{<} L\Delta^{(K-1)}(t) \overset{(c_1)}{\leq} l_{a_1}(t) - r_{a_1+1}(t) \overset{(d_1)}{\leq} \boldsymbol{\mu}_{a_1} - \boldsymbol{\mu}_{a_1+1} \\
    %     \Delta_{(k)}&{\leq} L\Delta^{(K-2)}(t) {\leq} \boldsymbol{\mu}_{a_2}-\boldsymbol{\mu}_{a_2+1} \\
    %     \vdots \\
    %     \Delta_{(k)}&{\leq} L\Delta^{(1)}(t) {\leq} \boldsymbol{\mu}_{a_{K-1}}-\boldsymbol{\mu}_{a_{K-1}+1}
    % \end{align}
    \begin{equation}
        \Delta_{(k)}{\leq} L^{(K-2)}(t) {\leq} {\mu}_{a_2}-{\mu}_{a_2+1} \quad
        \hdots \quad \hdots \quad
        \Delta_{(k)}{\leq} L^{(1)}(t) {\leq} {\mu}_{a_{K-1}}-{\mu}_{a_{K-1}+1}.
    \end{equation}
      % The inequalities $(b_2), \dots, (b_{K-1})$ holds as $L\Delta^{(1)}(t) \geq L\Delta^{(2)}(t) \geq \ldots \geq L\Delta^{(K-1)}(t)$. The inequalities $(d_2), \dots, (d_{K-1})$ holds similar to that of $(c_1)$ and $(d_1)$.
     From the above set of $K-1$ inequalities, we can say that for some $k \in \{ 1, \ldots, K-1 \}$, $k^{th}$ highest gap is less than some $K-1$ gaps. This results in a contradiction. \\
     {\em Case 2:} Assume $\exists m \in \left\{ m_1+1, \dots, m_{K-1}+1 \right\}$ such that  $U\Delta_m^r(t) < L^{(K-1)}(t)$.\\
   By following similar steps as in {\em Case 1}, we can show that this assumption results in a contradiction. \\
    {\em Case 3:} Assume $\exists m \in [M]\setminus\left\{ m_1, \dots, m_{K-1} \right\}$ such that $L\Delta_m^l(t) > U^{(K)}(t)$.\\
    We have $\Delta_{(k)}=\mu_m-\mu_{m+1}$, for some $k \in \{K, K+1, \dots, M-1\}$. We write the following.
    % Let $k = K, K+1, \dots, M-1$
    % Here we have that the left gap of arm $m$ doesn't correspond to one of the top $K-1$ highest gaps, i.e., it corresponds to $k^{th}$ highest gap for some $k \in \{ K, K+1, \dots, M-1 \}$. Hence, we can say $\Delta_{(k)} = \boldsymbol{\mu}_m - \boldsymbol{\mu}_{m+1}$ and we write the following equations.
    \begin{equation}
        \Delta_{(k)} \overset{(a)}{\geq} L\Delta_{m}^l(t) 
        \overset{(b)}{>} U^{(K)}(t) \overset{(c)}{\geq}  {\mu}_{a_1} - {\mu}_{a_1+1}
    \end{equation}
    Since for any arm $m$, its actual left gap is greater than the minimum left gap of the arm $m$, $(a)$ holds. $(b)$ holds from our assumption. $U^{(K)}(t)$ corresponds to the UCB gap between two arms. So, there must exist two consecutive arms $a_1$ and $a_1+1$ whose true means lie inside this UCB gap and hence $(c)$ holds. Since $U^{(K)}(t) \geq U^{(K+1)}(t) \geq \ldots \geq U^{(M-1)}(t)$, and then by using argument similar to that of $(c)$, we have 
    % the following $M-K-2$ inequalities.
    \begin{equation}
        \Delta_{(k)}{\geq} U^{(K+1)}(t) {\geq} {\mu}_{a_2}-{\mu}_{a_2+1} \quad
        \hdots \quad \hdots \quad 
        \Delta_{(k)}{\geq} U^{(M-1)}(t) 
        {\geq} {\mu}_{a_{M-K}}-{\mu}_{a_{M-K}+1}
    \end{equation}
    % \begin{align}
    %     \Delta_{(k)} \overset{(a)}{\geq} L\Delta_{m}^l(t) 
    %     &\overset{(b_1)}{>} U\Delta^{(K)}(t) \overset{(c_1)}{\geq}  \boldsymbol{\mu}_{a_1} - \boldsymbol{\mu}_{a_1+1} \\
    %     \Delta_{(k)}&\overset{(b_2)}{\geq} U\Delta^{(K+1)}(t) \overset{(c_2)}{\geq} \boldsymbol{\mu}_{a_2}-\boldsymbol{\mu}_{a_2+1} \\
    %     \vdots \\
    %     \Delta_{(k)}&\overset{(b_{M-K})}{\geq} U\Delta^{(M-1)}(t) 
    %     \overset{(c_{M-K})}{\geq} \boldsymbol{\mu}_{a_{M-K}}-\boldsymbol{\mu}_{a_{M-K}+1}.
    % \end{align}
     % The inequalities $(b_2), \dots, (b_{M-K})$ holds as $U\Delta^{(K)}(t) \geq U\Delta^{(K+1)}(t) \geq \ldots \geq U\Delta^{(M-1)}(t)$. The inequalities $(c_2), \dots, (c_{M-K})$ holds similar to that of $(c_1)$.
    From the above set of inequalities, we can say that for some $k\in \{K, \dots, M-1\}$, $k^{th}$ highest gap is greater than $M-K$ gaps. This results in a contradiction. \\
    {\em Case 4: } Assume $\exists m \in [M]\setminus\left\{ m_1+1, \dots, m_{K-1}+1 \right\}$ such that $L\Delta_m^r(t) > U^{(K)}(t)$.\\
    By following similar steps as in {\em Case 3}, we can show that this assumption results in a contradiction. \\
    Hence, under the good event $E$, algorithm returns the correct clustering. Now, from \eqref{eq:confinterval}, we have that the algorithm returns correct clustering with probability $1-\delta$. 
\end{proof}

\subsection{Proof of Theorem \ref{boc-elim sc}} \label{appsubsec:elimsc}
% \begin{proof}
% Now we will prove the sample complexity result of BOC-Elim algorithm.
% Recall that the confidence interval for arms with each arm being pulled $n$ times is $c_n \coloneqq \sqrt{\frac{2}{n} \log\left( \frac{4M n^2}{\delta} \right)}$.
First, we mathematically characterize the number of samples $n$ required from each arm for its confidence interval $c_n$ to be less than $\rho>0$ in Lemma \ref{lemma: bound a arm}. Recall that $c_n \coloneqq \sqrt{\frac{2}{n} \log\left( \frac{4M n^2}{\delta} \right)}$. 
% Define the good event $E\coloneqq \bigcap_{t \in \mathbb{N}}\bigcap_{m \in [M]} \left\{ \mu_m \in [l_m(t), r_m(t)] \right\}$.
% First, we mathematically characterize the relation between the number of samples $n$ each arm observes and its confidence interval $c_n$ in Lemma \ref{lemma: bound a arm}.
\begin{lemma} \label{lemma: bound a arm}
    If the number of samples $n \geq \frac{23\log\frac{4M}{\delta \rho}}{\rho^2}$, then the confidence interval $c_n\leq \rho$.
\end{lemma}
\begin{proof}
We have $n = \frac{C\log\frac{4M}{\delta \rho}}{\rho^2}$, for $C\geq23$.
We can write $c_n \coloneqq \sqrt{\frac{2}{n} \log\left( \frac{4M n^2}{\delta} \right)}\leq \rho \iff n \geq \frac{2\log\left(\frac{4Mn^2}{\delta}\right)}{\rho^2}.$
% \begin{equation}
    
% \end{equation}
Now, to prove Lemma \ref{lemma: bound a arm}, we equivalently prove $\frac{2\log\left(\frac{4Mn^2}{\delta}\right)}{\rho^2}\leq n$ as follows.
    \begin{align}
        \frac{2}{\rho^2}\log\left(\frac{4Mn^2}{\delta}\right) &\leq \frac{2}{\rho^2}\left[ \log\left(\frac{4M}{\delta}\right) + 2\log(C) + 4\log\left(\frac{1}{\rho}\right) + \log\left( 
        \frac{4M}{\delta \rho} \right) \right] \\
        &\leq \frac{2}{\rho^2} \left[ 5\log\left(\frac{4M}{\delta\rho}\right) + 2\log(C) \right] 
        \leq \frac{2}{\rho^2}\left[5+2\log(C)\right]\log\left( 
        \frac{4M}{\delta \rho} \right) \\
        &\leq \frac{C}{x^2}\log\left( \frac{4M}{\delta\rho} \right) \quad \text{(for $C \geq 23$)} 
       \quad = n.
    \end{align}
 % Hence, proved.
\end{proof}

Define the good event $E\coloneqq \bigcap_{t \in \mathbb{N}}\bigcap_{m \in [M]} \left\{ \mu_m \in [l_m(t), r_m(t)] \right\}$. Now, we show that, under the good event $E$, an arm $m$ will be eliminated if it satisfies the condition $c_{N_m(t)}\leq \rho_m$, in the following Lemmas  \ref{lemma:non top right}, \ref{lemma: non top left}, \ref{lemma: top right}, and \ref{lemma; top left}.

% Now we discuss the proof of Theorem \ref{boc-elim sc}
% \begin{proof}
%     It will be shown in Lemmas \ref{lemma:non top right}, \ref{lemma: non top left}, \ref{lemma: top right}, \ref{lemma; top left}, that under the good event \eqref{eq:confinterval}, arm $m$ will be eliminated if it satisfies the condition $c_{N_m(t)}\leq \rho_m$. The remainder of this Theorem proof follows directly from Lemma \ref{lemma: bound a arm}.
% \end{proof}

\begin{lemma} \label{lemma:non top right}
    Consider an arm $m \notin \left\{ m_1+1, \ldots, m_{K-1}+1 \right\}$ that does not form any one of the top $K-1$ gaps on the right side. If $t$ is such that $c_{N_m(t)} \leq \rho_m^r$, then under the good event $E$, we have $U\Delta_m^r(t) < L^{(K-1)}(t)$. 
\end{lemma}
\begin{proof}
    Since we sample all arms that are not eliminated at each time $t$, we have $t = N_m(t)$. 
    % From the expression for the $\rho_m^r$ in equation \eqref{eq: rho r not top}, it can be verified that $\rho_m^r \leq \frac{\Delta_{(K-1)}}{4}$. Since we have $c_t \leq \rho_m^r $, we get the following inequality.
    % \begin{equation} \label{eq: ct bound}
    %     c_t \leq \frac{\Delta_{(K-1)}}{4}
    % \end{equation}
     Now, we write
    \begin{equation*}
    \begin{split}
        l_{m_k}(t) \geq {\mu}_{m_k} - 2c_t &= {\mu}_{m_k+1}+\Delta_{(k)} - 2c_t \geq r_{m_k+1}(t)+\Delta_{(k)}-4c_t
    \end{split}
    \end{equation*}
    Hence, we have $\forall k \in \{1, \ldots, K-1\}$, $l_{m_k}(t) - r_{m_k+1(t)} \geq \Delta_{(k)} - 4c_t \geq \Delta_{(K-1)} - 4c_t$.
    % \begin{equation*}
    %     l_{m_k}(t) - r_{m_k+1(t)} \geq \Delta_{(k)} - 4c_t \geq \Delta_{(K-1)} - 4c_t
    % \end{equation*}
    Therefore, there exist $K-1$ gaps with their LCB values greater than or equal to $\Delta_{(K-1)}-4c_t$ and so we get
    \begin{equation} \label{eq 1}
        L^{(K-1)}(t) \geq \Delta_{(K-1)} - 4c_t.
    \end{equation}
    We have $c_t \leq \rho_m^r$. Since $\rho_m^r$ in \eqref{eq: rho r not top} involves the minimum of two terms, we have the following two cases.\\
    {\em Case 1:} Consider $\displaystyle c_t < \max_{j: \Delta_{a,j}>0} \left[ \min\left\{ \frac{\Delta_{a,j}}{4}, \frac{\Delta_{(K-1)}-\Delta_{a,j}}{8}\right\} \right].$\\
    % \begin{equation*}
    %     c_t < \max_{j: \Delta_{a,j}>0} \left[ \min\left\{ \frac{\Delta_{a,j}}{4}, \frac{\Delta_{(K-1)}-\Delta_{a,j}}{8}\right\} \right]
    % \end{equation*}
    Let $e$ be the maximizer for the above maximization problem, i.e., $\displaystyle e \coloneqq \argmax_{j: \Delta_{a,j}>0} \left[ \min\left\{ \frac{\Delta_{a,j}}{4}, \frac{\Delta_{(K-1)}-\Delta_{a,j}}{8}\right\} \right].$
    % \begin{equation*}
    %     e \coloneqq \argmax_{j: \Delta_{a,j}>0} \left[ \min\left\{ \frac{\Delta_{a,j}}{4}, \frac{\Delta_{(K-1)}-\Delta_{a,j}}{8}\right\} \right]
    % \end{equation*}
    Hence, we have 
    \begin{equation} \label{eq 3}
        c_t \leq \frac{\Delta_{ae}}{4} \text{ and } c_t \leq \frac{\Delta_{(K-1)}-\Delta_{ae}}{8}.
    \end{equation}
    Since $c_t \leq \frac{\Delta_{ae}}{4}$, we write $l_e(t) \geq r_a(t)$. Since there exists an arm $e$ whose LCB is greater than the UCB of the arm $a$, we can bound the right UCB gap of the arm $a$ as 
    \begin{equation}
    \begin{split}
        U\Delta_a^r(t) &\leq r_e(t) - l_a(t) \leq \Delta_{ae}+4c_t 
        \leq \Delta_{(K-1)}-4c_t  \text{   (using \eqref{eq 3})}
        \leq L^{(K-1)}(t)  \text{   (using \eqref{eq 1})}.
    \end{split}
    \end{equation}
    {\em Case 2:} 
    \begin{equation} \label{eq 4}
        c_t < \frac{\Delta_{(K-1)}-\Delta_{a1}}{8}
    \end{equation}
    Let $e \coloneqq \argmax_{i\neq a} r_i(t)$. Now we bound the right UCB gap of the arm $a$ as 
    \begin{equation}
    \begin{split}
        U\Delta_a^r(t) &\leq r_e(t) - l_a(t) 
        \leq \Delta_{ae}+4c_t 
        \leq \Delta_{a1}+4c_t \text{   (using \eqref{eq 4})}
        \leq \Delta_{(K-1)}-4c_t 
        \leq L^{(K-1)}(t) \text{   (using \eqref{eq 1})}.
    \end{split}
    \end{equation}
 This proves Lemma \ref{lemma:non top right}.
\end{proof}
\begin{lemma} \label{lemma: non top left}
    Consider an arm $m \notin \left\{ m_1, \ldots, m_{K-1} \right\}$ that does not form any one of the top $K-1$ gaps in the left side. If $t$ is such that $c_{N_m(t)} \leq \rho_m^l$, then under the good event $E$, we have $U\Delta_m^l(t) < L^{(K-1)}(t)$.
\end{lemma}
\begin{proof}
    The proof follows steps similar to the proof of Lemma \ref{lemma:non top right}.
\end{proof}
\begin{lemma} \label{lemma: top right}
    Consider an arm $m \in \left\{ m_1+1, \ldots, m_{K-1}+1 \right\}$ that forms any one of the top $K-1$ gaps on the right side. If $t$ is such that $c_{N_m(t)} \leq \rho_m^r$, then under the good event $E$, we have $L\Delta_m^r(t) > U^{(K)}(t)$.
\end{lemma}
\begin{proof}
    % We have $c_t \leq \rho_a^r \leq \frac{\Delta_{(K)}}{4}$. 
    First, we write $r_{m_k}(t) \leq {\mu}_{m_k} + 2c_t = {\mu}_{m_k+1}+\Delta_{(k)} + 2c_t \leq l_{m_k+1}(t)+\Delta_{(k)}+4c_t$.
    % \begin{equation*}
    % \begin{split}
    %     r_{m_k}(t) \leq \boldsymbol{\mu}_{m_k} + 2c_t &= \boldsymbol{\mu}_{m_k+1}+\Delta_{(k)} + 2c_t \\
    %     &\leq l_{m_k+1}(t)+\Delta_{(k)}+4c_t
    % \end{split}
    % \end{equation*}
 Hence, for all $k \in \{K, \ldots, M-1\}$, we have $r_{m_k}(t) - l_{m_k+1(t)} \leq \Delta_{(k)} + 4c_t \leq \Delta_{(K)} + 4c_t$.
    % \begin{equation*}
    %     r_{m_k}(t) - l_{m_k+1(t)} \leq \Delta_{(k)} + 4c_t \leq \Delta_{(K)} + 4c_t
    % \end{equation*}
    Therefore, there exist $M-K$ gaps with a UCB value less than or equal to $\Delta_{(K)}+4c_t$ and so we get
    \begin{equation} \label{eq 2}
        U^{(K)}(t) \leq \Delta_{(K)} + 4c_t.
    \end{equation}
    Without loss of generality, assume that $m\neq K$. From \eqref{eq: rho r top}, we have $c_t \leq \frac{\Delta_{m, m-1}}{4}$ and $c_t\leq \frac{\Delta_{m+1, m}}{4}$. Therefore, the confidence interval of the arm $m$ is disjoint from that of other arms, and hence for some $e<m$, we write
    \begin{equation}
    \begin{split}
        L\Delta_m^r(t) = l_e(t)-r_m(t) &\geq \Delta_{m,e}-4c_t \geq \Delta_{m,m-1}-4c_t 
        \geq \Delta_{(K)}+4c_t \text{   (using \eqref{eq: rho r top})}
        \geq U^{(K)}(t) \text{   (using \eqref{eq 2})}.
    \end{split}
    \end{equation}
    This proves Lemma \ref{lemma: top right}.
\end{proof}

\begin{lemma} \label{lemma; top left}
    Consider an arm $m \in \left\{ m_1, \ldots, m_{K-1} \right\}$ that forms any one of the top $K-1$ gaps on the left side. If $t$ is such that $c_{N_m(t)} \leq \rho_m^l$, then under the good event $E$, we have $L\Delta_m^l(t) > U^{(K)}(t)$.
\end{lemma}
\begin{proof}
    The proof follows steps similar to the proof of Lemma \ref{lemma: top right}.
\end{proof}
% \noi?ndent{Proof of Theorem \ref{boc-elim sc}:}
Now we use the Lemmas \ref{lemma: bound a arm}, \ref{lemma:non top right}, \ref{lemma: non top left}, \ref{lemma: top right}, and \ref{lemma; top left} to prove Theorem \ref{boc-elim sc}.
\begin{proof}
    Through Lemmas \ref{lemma:non top right} and \ref{lemma: top right}, we showed that an arm $m$ is not right sided active if $c_{N_m(t)}<\rho_m^r$. Similarly through Lemmas \ref{lemma: non top left} and \ref{lemma; top left}, we showed that an arm $m$ is not left sided active if $c_{N_m(t)}<\rho_m^l$. Hence, we say that an arm $m$ will be eliminated if $c_{N_m(t)}<\rho_m$.
    % Through Lemmas \ref{lemma:non top right}, \ref{lemma: non top left}, \ref{lemma: top right}, and \ref{lemma; top left}, we show that the arm $m$ will be eliminated if $c_{N_m(t)}<\rho_m$. 
    Now, from Lemma \ref{lemma: bound a arm}, we have $N_m(t) \leq \frac{23\log\frac{4M}{\delta\rho}}{\rho^2}$. Therefore, the total sample complexity is upper bounded by $\sum_{m\in[M]} \frac{23\log\frac{4M}{\delta\rho}}{\rho^2}$. This completes the proof of Theorem \ref{boc-elim sc}.
\end{proof}

\section{Proof of Lemma \ref{Lemma:finiteconvex}}\label{appsec: computation}
\begin{proof}
We have $\text{Alt}(\boldsymbol{\mu}) = \left\{ \boldsymbol{\lambda} \in \mathbb{R}^M \mid \mathcal{C}\left( \boldsymbol{\lambda} \right) \nsim \mathcal{C}\left( \boldsymbol{\mu} \right) \right\}$, which we rewrite using $d_{\text{INTRA}}$ and $d_{\text{INTER}}$ as $\text{Alt}(\boldsymbol{\mu}) =\left\{ \boldsymbol{\lambda}\mid \left(\boldsymbol{\lambda}, \mathcal{C}(\boldsymbol{\mu})\right) \text{ satisfies } d_{\text{INTRA}} > d_{\text{INTER}} \right\}$. We use the expressions of $d_{\text{INTRA}}$ and $d_{\text{INTER}}$ to get, 
        \begin{equation} \label{eq:Alt1}
    \begin{aligned}
        \text{Alt}(\boldsymbol{\mu}) 
        = \bigcup_{k \in [K]} \bigcup_{\substack{P_1 \in 2^{D_k}\setminus\{ \emptyset, D_k \} \\ P_2 = D_k \setminus P_1}} \bigcup_{\substack{m \in D_p, n \in D_q \\ p, q\in [K], p \neq q}} &\bigcap_{\substack{i \in P_1 
        j \in P_2}} \left\{  d_{i, j}(\boldsymbol{\lambda}) > d_{m, n}(\boldsymbol{\lambda}) \right\} 
    \end{aligned}
    \end{equation}
    Here, we use the notation $\left\{  d_{i, j}(\boldsymbol{\lambda}) > d_{m, n}(\boldsymbol{\lambda}) \right\}$ to informally represent the set $\left\{  \boldsymbol{\lambda} \mid d_{i, j}(\boldsymbol{\lambda}) > d_{m, n}(\boldsymbol{\lambda}) \right\}$. We follow the same informal notation for the sets in the remainder of the proof. Hence the inner infimum $\psi\left(\boldsymbol{w}, \boldsymbol{\mu}\right)$ can be written as  
    \begin{equation*} 
     \psi(\boldsymbol{w}, \boldsymbol{\mu}) = \min_{k \in [K]}  \min_{\substack{P_1 \in 2^{D_k}\setminus\{ \emptyset, D_k \} \\ P_2 = D_k \setminus P_1}}  \min_{\substack{m \in D_p, n \in D_q \\ p, q\in [K], p \neq q}}  \inf_{\substack{\boldsymbol{\lambda} \in \bigcap_{\substack{i \in P_1 j \in P_2}}\\ \left\{  d_{i, j}(\boldsymbol{\lambda})\right.  \left. > d_{m, n}(\boldsymbol{\lambda}) \right\}}} q_{\boldsymbol{w}}(\boldsymbol{\lambda}).
    \end{equation*}
    Note that we take the infimum over the open set, which is the same as taking the infimum over the closure of that open set. Hence, we get 
    \begin{equation*} 
    \psi(\boldsymbol{w}, \boldsymbol{\mu}) = \min_{k \in [K]}  \min_{\substack{P_1 \in 2^{D_k}\setminus\{ \emptyset, D_k \} \\ P_2 = D_k \setminus P_1}}  \min_{\substack{m \in D_p, n \in D_q \\ p, q\in [K], p \neq q}}  \inf_{\substack{\boldsymbol{\lambda} \in \bigcap_{\substack{i \in P_1 j \in P_2}}\\ \left\{  d_{i, j}(\boldsymbol{\lambda})\right.  \left. \geq d_{m, n}(\boldsymbol{\lambda}) \right\}}} q_{\boldsymbol{w}}(\boldsymbol{\lambda}).
    \end{equation*}
    Note that the optimal solution $\boldsymbol{\lambda}^*$ cannot be unbounded, otherwise $\psi\left(\boldsymbol{w}, \boldsymbol{\mu}\right)=\infty$, which is obviously not an optimal solution. Also, the search space of $\boldsymbol{\lambda}$ in the above infimum problem is closed. Therefore, the optimal solution $\boldsymbol{\lambda}^*$ lies in the compact space and hence we replace $\inf$ with $\min$. This proves Lemma \ref{Lemma:finiteconvex}.
\end{proof}

% --------------------------------------------------------------
\section{Auxiliary Results} \label{appsec:otherclaims}

We use the following lemmas from the literature. 

\begin{lemma} \label{Lemma:Concbound}
        If there exists some constant $c>0$, and $t_0\geq 0$ such that $\forall t \geq t_0$, $\min_{m \in [M]} N_m(t) \geq c\sqrt{t}$ a.s., then
        \begin{equation}
             \mathbb{P}\left( \| \hat{\boldsymbol{\mu}}(t) - \boldsymbol{\mu} \| \geq \epsilon \right) \leq et^f\exp\left(-g\sqrt{t}\right)
             %c^{-\frac{Md}{2}} t^{\frac{Md}{4}} \exp{\left( -\frac{c\epsilon^2\sqrt{t}}{4} \right)}, \forall t\geq t_0.
        \end{equation}
        where, $(e, f, g) = \begin{cases}
    \left(c^{\frac{-Md}{2}}, \frac{Md}{4}, \frac{c\zeta(\epsilon)^2}{4}\right) &\text{for multivariate Sub-Gaussian distributions} \\
    \left(2M, 1, cB\right) &\text{for single parameter exponential family of distributions}
\end{cases}$ and \\$B = \min\left\{ d_{\text{KL}}(\mu+\epsilon/\sqrt{M}, \theta), d_{\text{KL}}(\mu-\epsilon/\sqrt{M}, \theta) \right\}$.
% , for some constant $c$.
    \end{lemma} 
\begin{proof}
    For multivariate Sub-Gaussian distributions, the result follows directly from Lemma $4$ in \cite{jedra2020optimal}.
    For the single-parameter exponential family of distributions, we provide the proof as follows.
    % \begin{equation}
        \begin{align}
            \mathbb{P}\left( \| \hat{\boldsymbol{\mu}}(t) - \boldsymbol{\mu} \| \geq \epsilon \right) 
            &\leq \mathbb{P}\left(\exists m \in [M] :  |\hat{\mu}_m(t) - \mu_m| \geq \frac{\epsilon}{\sqrt{M}} \right) \\
            &\leq \sum_{m=1}^M \mathbb{P}\left(|\hat{\mu}_m(t) - \mu_m| \geq \frac{\epsilon}{\sqrt{M}} \right) \ \ \text{(Union bound)} \\
            &\leq \sum_{m=1}^M \mathbb{P}\left(|\hat{\mu}_m(t) - \mu_m| \geq \frac{\epsilon}{\sqrt{M}}, c\sqrt{t} \leq N_m(t) \leq t \right) \\
            &\leq \sum_{m=1}^M \sum_{k = c\sqrt{t}}^t \mathbb{P}\left(|\hat{\mu}_m(t) - \mu_m| \geq \frac{\epsilon}{\sqrt{M}}, N_m(t) = k \right)  \\
            &\leq \sum_{m=1}^M \sum_{k = c\sqrt{t}}^t 2\exp\left( -k B) \right), \ \ \text{(Chernof bound)}
        \end{align}
    % \end{equation}
    where $B = \min\left\{ d_{\text{KL}}(\mu+\epsilon/\sqrt{M}, \theta), d_{\text{KL}}(\mu-\epsilon/\sqrt{M}, \theta) \right\}$.
    Now we can upper bound it as $\mathbb{P}\left( \| \hat{\boldsymbol{\mu}}(t) - \boldsymbol{\mu} \| \geq \epsilon \right) \leq 2Mt \exp\left(-cB\sqrt{t}\right).$
    % \begin{equation}
    %     \mathbb{P}\left( \| \hat{\boldsymbol{\mu}}(t) - \boldsymbol{\mu} \| \geq \epsilon \right) \leq 2Mt \exp\left(-cB\sqrt{t}\right).
    % \end{equation}
 Hence, proved.
\end{proof}
    
\begin{lemma} \label{Lemma:gammabound}
         Consider $x\in\mathbb{R}$, $a>1$ and $B>1$, which satisfies the condition $x>\frac{B}{B-1}(a-1)$, then we have
         % for some constants $a>1$ and $B>1$, then we have the following inequality, 
        \begin{equation}
            x^{a-1}e^{-x} < \Gamma(a, x) < B x^{a-1}e^{-x}, \quad \text{where } \Gamma(a, x) = \int_x^\infty e^{-x} x^{a-1} dx.
        \end{equation}
        % where
        % \begin{equation}
        %     \Gamma(a, x) = \int_x^\infty e^{-x} x^{a-1} dx.
        % \end{equation}
    \end{lemma}
\begin{proof}
    Equation $1.5$ in \cite{borwein2009uniform}.
\end{proof}
\begin{lemma} \label{Lemma:infbound}
         For any constants $c_1, c_2 > 0$ and $\frac{c_2}{c_1}>1$, we have
        \begin{equation}
        \begin{aligned}
            \inf\left\{ t \in \mathbb{N}: c_1t \geq \log\left( c_2 t \right) \right\} &\leq \frac{1}{c_1}\left( \log\left( \frac{c_2e}{c_1} \right) + \log\log\left( \frac{c_2}{c_1} \right) \right).
        \end{aligned}
        \end{equation}
    \end{lemma}
\begin{proof}
    Lemma $8$ in \cite{jedra2020optimal}.
\end{proof}
\begin{lemma} \label{delprop}
     Let $\left( \mathcal{F}_t \right)_{t \geq 0}$ be a filtration. Let $\{ \eta_t \}_{t\geq 1}$ be a real valued stochastic process such that for all $t\geq 1$, $\eta_t$ is $\mathcal{F}_{t}-$ measurable and satisfies the conditional $\sigma-$ sub-Gaussian condition for some positive $\sigma$,i.e., $\mathbb{E}\left[ 
    \exp\left( x\eta_t\right) \mid \mathcal{F}_{t-1} \right] \leq \exp\left( 
    -\frac{-x^2\sigma^2}{2}\right)$, for all $x \in \mathbb{R}$. Let $V$ be a positive definite matrix and $(A_t)_{t \geq 1}$ be an $\mathbb{R}^d-$ valued stochastic process adapted to $\{\mathcal{F}_t\}_{t\geq 0}$. Let $\tau$ be any stopping time with respect to the filtration $(\mathcal{F}_t)_{t \geq 1}$. Then, for any $\delta > 0$, 
    \begin{equation} 
    \begin{aligned}
        &\mathbb{P}\left( \|A_\tau^T E_\tau\|^2_{(A_\tau^TA_\tau+V)^{-1}} \leq 2\sigma^2 \log\left( \frac{\det\left( (A_\tau^TA_\tau+V)V^{-1} \right)^{-1}}{\delta} \right) \right) \geq 1 - \delta.
    \end{aligned}
    \end{equation}
\end{lemma}
\begin{proof}
    Theorem 1 of \cite{abbasi2011improved}.
\end{proof}
{
\begin{lemma} \label{lemma: mixmart}
        Fix $\zeta \in \left(0, \frac{1}{2}\right)$ and $x>0$. Define $X_m(t) = N_m(t) d_{\text{KL}}\left( \hat{\mu}_m(t), \mu_m \right) - 3\log\left( 1+\log{N_m(t)} \right)$  Then there exists a martingale $Z_m(t)$ with $Z_m(0)=1$ satisfies, that for all $t \in \mathbb{N}$,
        \begin{equation}
        \begin{aligned}
            &\left\{ X_m(t) - (1+\zeta)\log\left( \frac{\pi^2/3}{\left(\log(1+\zeta)\right)^2} \right) \geq x \right\} \subseteq \left\{ Z_m(t) \geq e^{\frac{x}{1+\zeta}} \right\}.
            \end{aligned}
        \end{equation}
        Moreover, the product of martingale $\prod_{m \in [M]}Z_m(t)$ is a martingale.
    \end{lemma}
    \begin{proof}
        Lemma 13 of \cite{kaufmann2021mixture}.
    \end{proof}}

\end{document}